\documentclass[a4paper,11pt]{llncs}
\AtBeginDocument{%
	\providecommand\BibTeX{{%
			\normalfont B\kern-0.5em{\scshape i\kern-0.25em b}\kern-0.8em\TeX}}}

\usepackage{microtype}
\usepackage{xcolor}
\usepackage{wrapfig}
\usepackage{microtype}
\usepackage{epsfig}
\usepackage{epstopdf}
\usepackage{amsmath}
\usepackage{latexsym}
\usepackage{multirow}
\usepackage{amsfonts}
\usepackage{cite}
\usepackage{amssymb}
\usepackage{caption}
\usepackage{algorithm}
\usepackage{algorithmic}
\usepackage{enumerate}
\usepackage{diagbox}
\usepackage{multirow}  
\usepackage{makecell}
\usepackage{subfigure}
\usepackage{bbm}

\usepackage{times}  
\usepackage{helvet}  
\usepackage{courier}  
\usepackage{url}  
\usepackage{graphicx}  
\usepackage{multirow}
\usepackage{hyperref}[colorlinks,linkcolor=black]

\makeatletter

\newcommand{\Rmnum}[1]{\expandafter\@slowromancap\romannumeral #1@}
\makeatother

\renewcommand{\algorithmicrequire}{\textbf{Input:}}
\renewcommand{\algorithmicensure}{\textbf{Output:}}

\usepackage[a4paper,margin=2.7cm,footskip=0.25in]{geometry}

\spnewtheorem{claim}{Claim}{\bfseries}{\rmfamily}

\spnewtheorem{remark}{Remark}{\bfseries}{\rmfamily}

\spnewtheorem{property}{Property}{\bfseries}{\rmfamily}

\begin{document}

\title{A Data-dependent Approach for High Dimensional (Robust) Wasserstein Alignment
}
\author{Hu Ding\inst{1} \and Wenjie Liu\inst{1} \and Mingquan Ye\inst{2}}
\institute{
University of Science and Technology of China \\
He Fei, Anhui, China\\
\email{huding@ustc.edu.cn, lwj1217@mail.ustc.edu.cn}\\
\and University of Illinois Chicago\\
Chicago, Illinois, USA\\
\email{mye9@uic.edu}\\
}

\maketitle

\thispagestyle{empty}

\begin{abstract}
Many real-world problems can be formulated as the alignment between two geometric patterns. Previously, a great amount of research focus on the alignment of 2D or 3D patterns  in the field of computer vision. Recently, the alignment problem in high dimensions finds several novel applications in practice. However, the research is still rather limited in the algorithmic aspect. To the best of our knowledge, most existing approaches are just simple extensions of their counterparts for 2D and 3D cases, and often suffer from the issues such as high computational complexities. In this paper, we propose an effective framework to compress the high dimensional geometric patterns. Any existing alignment method can be applied to the compressed geometric patterns and the time complexity can be significantly reduced. Our idea is inspired by the observation that high dimensional data often has a low intrinsic dimension. Our framework is a ``data-dependent'' approach that has the complexity depending on the intrinsic dimension of the input data. Our experimental results reveal that running the alignment algorithm on compressed patterns can achieve similar qualities, comparing with the results on the original patterns, but the runtimes (including the times cost for compression) are substantially lower.

\end{abstract}


\pagestyle{plain}
\pagenumbering{arabic}
\setcounter{page}{1}

\section{Introduction}

\label{sec-intro}

Given two geometric patterns, the problem of alignment is to find their appropriate spatial positions so as to minimize the difference between them. In general, a geometric pattern is represented by a set of (weighted) points in the space, and their difference is often measured by some objective function. Geometric alignment finds many applications in the field of computer vision, such as image retrieval, pattern recognition, fingerprint and facial shape alignment, {\em etc}~\cite{cohen1999earth,maltoni2009handbook,cao2014face}. For different applications, we may have different constraints for the alignment, {\em e.g.}, we can allow rigid transformations for fingerprint alignment. In addition, the {\em Wasserstein distance} (which is also called {\em ``earth mover's distance''} in some applications)~\cite{rubner2earth} has been widely adopted  for measuring the difference of patterns in computer vision, where its major advantage over other measures is the robustness  in practice~\cite{Villani08references}. 
%
%
%
Besides the computer vision applications in 2D/3D, recent research shows that a number of high dimensional problems can be solved through geometric alignments. We briefly introduce several interesting examples below. 

\begin{itemize}

	\item The research on natural language processing has revealed that different languages often share similar structure at the word level~\cite{youn2016universal}; in particular, the recent study on word semantics embedding has also shown the existence of structural isomorphism across languages~\cite{mikolov2013exploiting,DBLP:journals/kais/DevHP21,DBLP:conf/iclr/AboagyeZYWZWYP22}, and finds that the  Wasserstein distance can serve as a good distance for languages and  documents~\cite{DBLP:conf/emnlp/ZhangLLS17,kusner2015word}.   Zhang {\em et al.}~\cite{DBLP:conf/emnlp/ZhangLLS17} proposed to learn the transformation between different languages without any cross-lingual supervision, and the problem is reduced to minimizing the Wasserstein distance via finding the optimal  geometric alignment in high dimensions.  A number of improved algorithms that use the Wasserstein distance to compute the alignment between languages have also been proposed in recent years~\cite{grave2019unsupervised,DBLP:conf/emnlp/Alvarez-MelisJ18}.

	\item A Protein-Protein Interaction (PPI) network is a graph representing the interactions among proteins. Given two PPI networks, finding their alignment is a fundamental bioinformatics problem for understanding the correspondences between different species~\cite{malod2017unified}. However, most existing approaches require to solve the NP-hard subgraph isomorphism problem and often suffer from high computational complexities. To resolve this issue, Liu {\em et al.}~\cite{DBLP:conf/aaai/LiuDC017} recently applied the geometric embedding techniques to develop a new framework based on geometric alignment in Euclidean space.

	\item
	In supervised learning, our task usually is to learn the knowledge from a given labeled training dataset. However,  labeled data could be very limited in practice. We can generate the labels for an unlabeled dataset by exploiting an existing annotated dataset, that is, transfer the knowledge from a source domain to a target domain. The problem is called ``domain adaptation'' in the field of transfer learning~\cite{DBLP:journals/tkde/PanY10}. The problem has received a great amount of attention in the past years~\cite{DBLP:conf/nips/BlitzerCKPW07,DBLP:journals/ml/Ben-DavidBCKPV10}. Recently,  Courty {\em et al.} \cite{DBLP:journals/pami/CourtyFTR17} modeled the domain adaptation problem as a transportation problem of minimizing the Wasserstein distance between the source and target domains in high dimensions.
\end{itemize}
Despite of the above   studies in the practical areas, the research on the algorithmic aspect of high dimensional alignment is still rather limited. We need to take into account of the high dimensionality and large number of points of the geometric patterns simultaneously. 
In particular, as the developing of data acquisition techniques, the data sizes also increase very fast. For example, due to the rapid progress of high-throughput sequencing technologies, biological data are growing exponentially in recent years~\cite{yin2017computing}. So it is quite important to develop efficient algorithmic techniques for dealing with the high-dimensional geometric alignment problem. Unfortunately, it has been shown that \textbf{any constant factor approximation for geometric alignment under rigid transformation with Wasserstein distance takes at least $n^{\Omega(d)}$ time in the worst case}, where $n$ is the maximum size of the two input patterns and $d$ is the dimensionality, unless $\mathtt{SNP}\subset \mathtt{DTIME}(2^{o(n)})$~\cite{DBLP:journals/algorithmica/DingX17}. \textbf{The reason  why the alignment problem is so hard is that it needs to compute the ``transformation'' and ``matching'' jointly for optimizing the objective}. Namely, the transformation and matching interact and influence each other during the optimization procedure. Moreover, it is complicated to determine a rigid transformation in high dimensional space. 
For instance, we can imagine that determining an appropriate rigid transformation in $\mathbb{R}^d$ needs to fix at least $d$ points in the space (so intuitively the complexity can be as large as ${n\choose d}$). 
Therefore, it is natural to ask the question: 

\vspace{0.05in}
{\em Is it possible to design low-complexity algorithm for high-dimensional geometric alignment under some reasonable assumption? } 
\vspace{0.05in}




To the best of our knowledge, we are the first to consider this problem from theoretical perspective. Our idea is inspired by the observation that many real-world datasets often manifest low intrinsic dimensions~\cite{belkin2003problems}. 
For example, human handwriting images can be well embedded to some low dimensional manifold though the Euclidean dimension can be very high. Following this observation, we consider to exploit the widely used notion, ``doubling dimension''~\cite{krauthgamer2004navigating,talwar2004bypassing,karger2002finding,har2006fast,dasgupta2013randomized}, to deal with the geometric alignment problem. 
The doubling dimension is particularly suitable to depict the data having low intrinsic dimension. Our framework is a ``data-dependent'' approach where the complexity depends on the doubling dimension of the input data. We prove that the given geometric patterns with low doubling dimensions can be substantially compressed so as to save a large amount of running time when computing the alignment. More importantly, our compression approach is an independent step, and hence can serve as the preprocessing for different alignment methods. A preliminary version of this work has appeared in~\cite{DBLP:conf/aaai/DingY19}. 

\textbf{Note:} Independent of our work~\cite{DBLP:conf/aaai/DingY19} published in 2019, Beugnot {\em et al.}~\cite{beugnot2021improving} also considered the speedup for computing Wasserstein distance and proposed a $k$-means++ based compression method (which is somewhat similar to our $k$-center based method in Section~\ref{sec-aa}).   Nevertheless, there are several significant differences between our work and~\cite{beugnot2021improving}. First, we consider the alignment problem under rigid transformation, where the work of~\cite{beugnot2021improving} only considers the static version. Also, we extend our method to fractional Wasserstein distance but it is unclear whether the theoretical analysis of~\cite{beugnot2021improving} is also available for the case with outliers. 


The rest of the paper is organized as follows. We introduce several important definitions in Section~\ref{sec-pre}, and discuss some related works and our main idea in Section~\ref{sec-prior} and Section~\ref{sec-oa} respectively. Then, we present our algorithm, analysis, and the time complexity  in Section~\ref{sec-aa} and Section \ref{sec-time}. Finally, we study the practical performance of our proposed algorithm in Section~\ref{sec-exp}.

\subsection{Preliminaries}
\label{sec-pre}

Before introducing the formal definition of geometric alignment, we need to define ``Wasserstein distance'' and ``rigid transformation''  first. 
In general, the Wasserstein distance is used to measure the difference between two distributions. In this paper, we consider the case that the distributions are discrete point sets. Given two points $p$ and $q\in \mathbb{R}^d$, we use $||p-q||$ to denote their Euclidean distance.


\begin{definition}[\textbf{Wasserstein distance $\mathcal{W}^2_2$}~\cite{Villani08references}]
	\label{def-emd}
	Let $A=\{a_1, a_2, \cdots, a_{n_1}\}$ and $B=\{b_1, b_2, \cdots, $ $b_{n_2}\}$ be two sets of weighted points in $\mathbb{R}^d$ with nonnegative weights $\alpha_i$ and $\beta_j$ for each $a_i\in A$ and $b_j\in B$ respectively, and $W_A$ and $W_B$ be their respective total weights. The Wasserstein distance between $A$ and $B$ is  
	\begin{eqnarray}
		\mathcal{W}^2_2(A, B)=\frac{1}{\min\{W_A, W_B\}}\min_{F}\sum^{n_1}_{i=1}\sum^{n_2}_{j=1}f_{ij}||a_i-b_j||^2, \label{for-emd}
	\end{eqnarray} 
	where $F=\{f_{ij}\}$ is a feasible flow from $A$ to $B$, i.e., each $f_{ij}\geq 0$, $\sum^{n_1}_{i=1}f_{ij}\leq\beta_j$, $\sum^{n_2}_{j=1}f_{ij}\leq\alpha_i$, and $\sum^{n_1}_{i=1}\sum^{n_2}_{j=1}f_{ij}=\min\{W_A, W_B\}$ (see Figure~\ref{fig:emd}(a)). 
\end{definition}
\begin{remark}
	\textbf{(1)} Usually we assume $W_A=W_B=1$ for convenience, but in this paper we also consider the ``partial matching'' ($W_A$ can be not equal to $W_B$).  
	\textbf{(2)} Intuitively, the Wasserstein distance can be viewed as the minimum transportation cost between $A$ and $B$, where the weights of $A$ and $B$ are the ``supplies'' and ``demands'' respectively, and the cost of each edge connecting a pair of points from $A$ to $B$ is their ``ground distance''. In general, the ``ground distance'' can be defined in various forms, and here we use the squared Euclidean distance due to its simplicity in practice.  
\end{remark}
\begin{figure}[htbp]
	\subfigure[Wasserstein distance]{
		\begin{minipage}[t]{0.25\linewidth}
			\centering
			\includegraphics[scale=0.35]{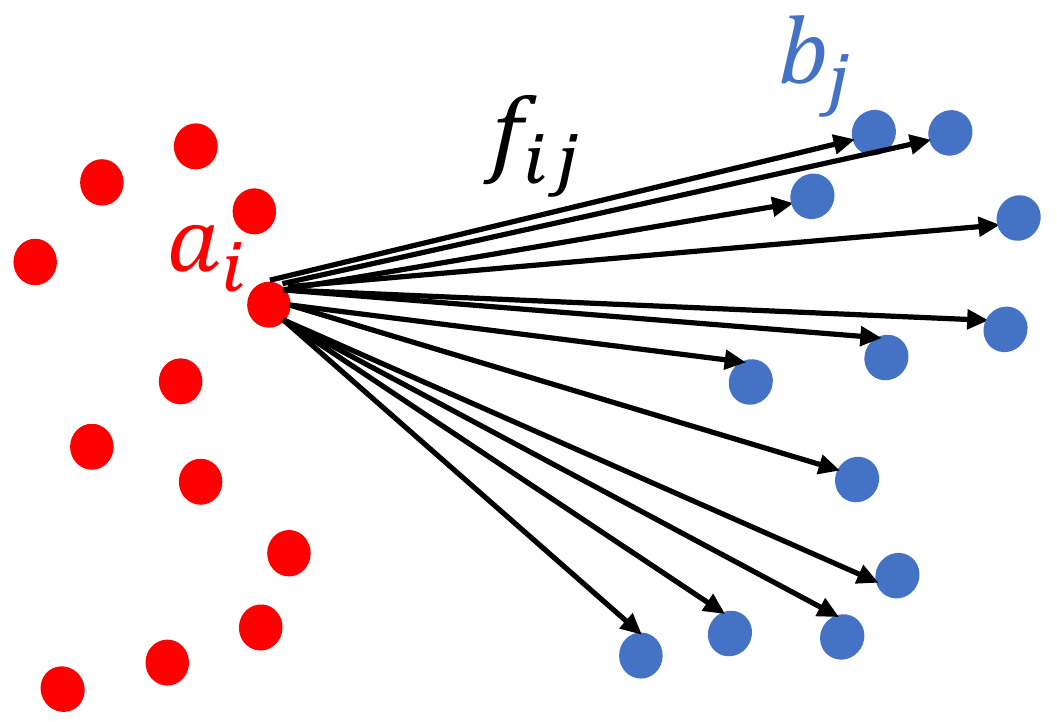}
	\end{minipage}}
	\subfigure[Rigid transformation]{
		\begin{minipage}[t]{0.25\linewidth}
			\centering
			\includegraphics[scale=0.35]{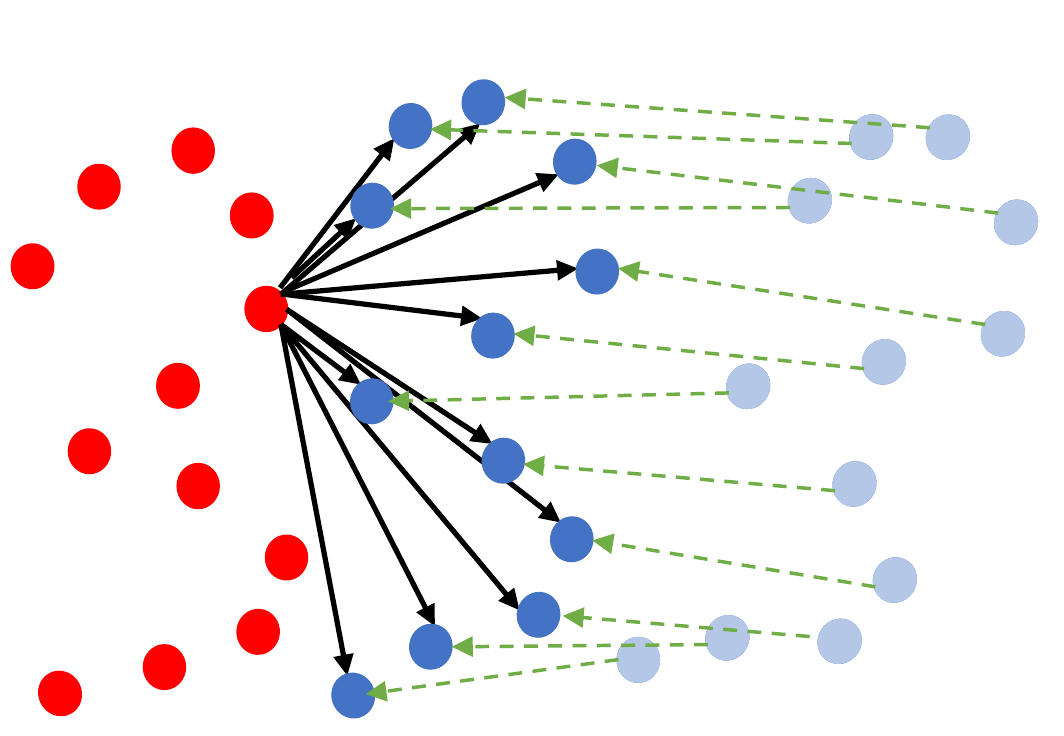}
	\end{minipage}}%
	\subfigure[Fractional Wasserstein distance]{
		\begin{minipage}[t]{0.35\linewidth}
			\centering
			\includegraphics[scale=0.25]{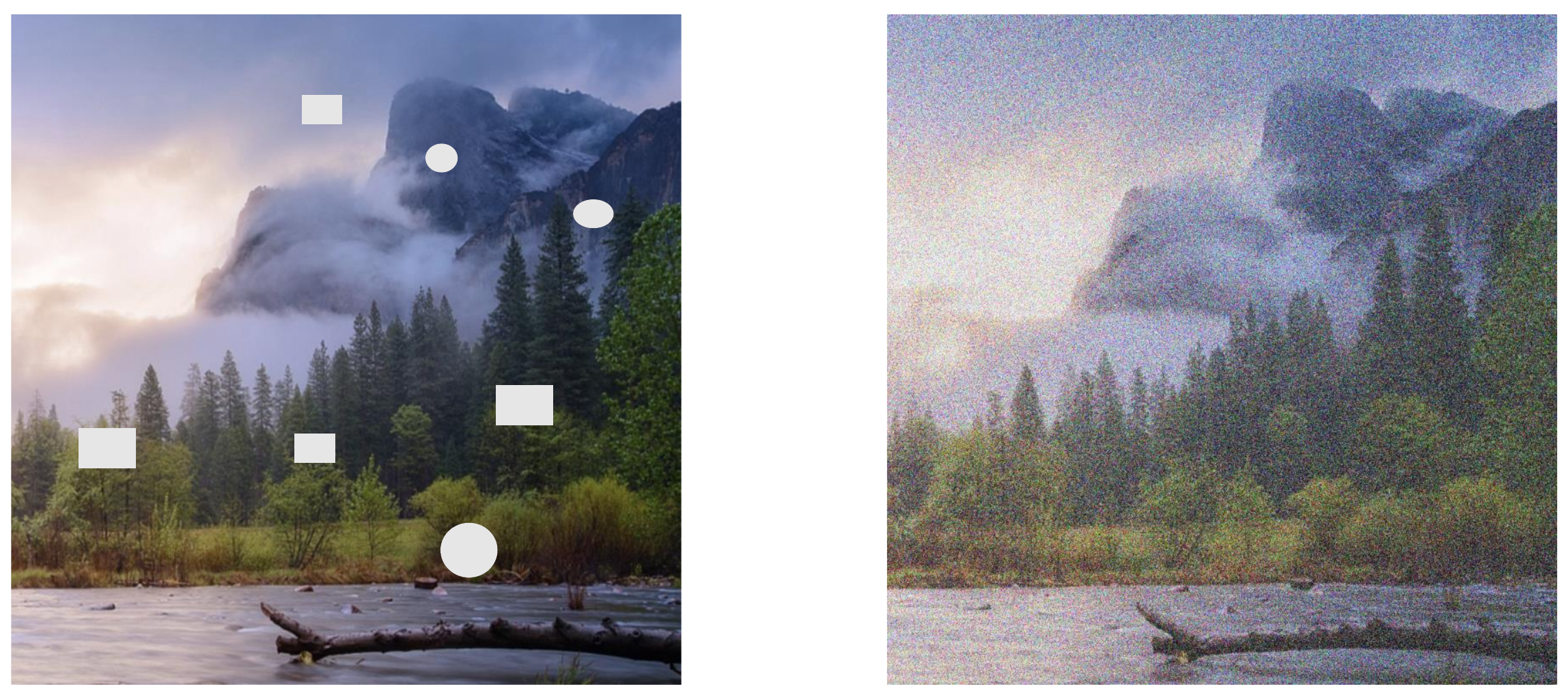}
	\end{minipage}}%
	\caption{(a) The flows of the Wasserstein distance from  $A=\{a_i\mid 1\leq i\leq n_1\}$ to $B=\{b_j\mid 1\leq j\leq n_2\}$, where each $f_{ij}$ is the flow from $a_i$ to $b_j$.  
		(b) To reduce the Wasserstein distance between (the fixed) $A$ and $B$, we find a rigid transformation for the pattern $B$ that transforms it from the light blue point set to the dark blue point set. 
		(c) The left image has some missing parts (masked by the white patches) and the right image has some noise. So it is more appropriate to compute the fractional Wasserstein distance to match them partially. 
	}
\label{fig:emd}  
\end{figure}

In some scenarios, we may only want a ``partial matching'' between $A$ and $B$ ({\em e.g.,} $A$ or $B$ may contain some outliers, or a small pattern $A$ only corresponds to a part of a large pattern $B$)~\cite{mukherjee2021outlier}.  See Figure~\ref{fig:emd}(c) for an example. For the applications in high dimensions, it is also natural to consider the partial matching. For example, one language may contain some words that cannot find their perfectly matched counterparts in the other language in the problem of cross-lingual alignment~\cite{DBLP:conf/emnlp/ZhangLLS17} (this phenomenon is very common between Chinese and English).  Therefore we introduce the fractional Wasserstein distance below.

\begin{definition}[\textbf{fractional Wasserstein distance}]
\label{def-fracemd}
Let $\lambda\in [0,1]$. In Definition~\ref{def-emd}, if we replace the constraint ``$\sum^{n_1}_{i=1}\sum^{n_2}_{j=1}f_{ij}=\min\{W_A, W_B\}$'' by ``$\sum^{n_1}_{i=1}\sum^{n_2}_{j=1}f_{ij}=\lambda\cdot\min\{W_A, W_B\}$'', we call the distance as the fractional Wasserstein distance and denote it as $\mathcal{W}^2_2(A, B, \lambda)$. 
\end{definition}

We consider the rigid transformation for alignment, because it is very natural to interpret in real world and has already been widely used in the aforementioned applications.

\begin{definition}[\textbf{Rigid Transformation}]
\label{def-rt}
Let $P$ be a set of points in $\mathbb{R}^{d}$. A rigid transformation $\mathcal{T}$ on $P$ is a 
transformation ({\em i.e.,} rotation, translation, reflection, or their combination) which preserves the pairwise distances of the points in $P$. 
\end{definition}

\begin{definition}[\textbf{Wasserstein Alignment}]
\label{def-align}
Given two weighted point sets $A$ and $B$ as described in Definition~\ref{def-emd}, the problem of Wasserstein alignment between $A$ and $B$ under rigid transformation is to determine a rigid transformation $\mathcal{T}$ for $B$ so as to minimize the  Wasserstein distance  $\mathcal{W}^2_2(A, \mathcal{T}(B))$, or $\mathcal{W}^2_2(A, \mathcal{T}(B),\lambda)$ if we consider the fractional Wasserstein distance (see Figure~\ref{fig:emd}(b)).
\end{definition}


As previously mentioned, we  use the doubling dimension to describe high dimensional data with low intrinsic dimension. We denote a metric space by $(X, d_X)$ where $d_X$ is the distance function of the set $X$. For instance, we can imagine that $X$ is a set of points in a low dimensional manifold and $d_X$ is simply the Euclidean distance. For any $x\in X$ and $r\geq 0$, we use $\mathtt{Ball}(x, r)=\{p\in X\mid d_X(x,p)\leq r\}$ to denote the ball of radius $r$ around $x$.


\begin{definition}[\textbf{Doubling Dimension}\cite{krauthgamer2004navigating,talwar2004bypassing}]
\label{def-dd}
The doubling dimension of a metric space $(X, d_X)$ is the smallest number $\rho$, such that for any $x\in X$ and $r\geq 0$, $\mathtt{Ball}(x, 2r)$ is always covered by the union of at most $2^\rho$ balls with radius $r$.
\end{definition}

\begin{remark}
The doubling dimension describes the expansion rate of $(X, d_X)$; intuitively, we can imagine a special case: a set of points are uniformly distributed inside a $\rho$-dimensional hypercube, where its doubling dimension is $O(\rho)$ but the Euclidean dimension can be very high. For a more general case, a manifold in high dimensional Euclidean space may have a very low doubling dimension, as many examples studied in machine learning~\cite{belkin2003problems}. 
Unfortunately, as shown in~\cite{laakso2002plane}, such low doubling dimensional metrics cannot always be embedded to low dimensional Euclidean spaces with low distortion of Euclidean distance. 

\end{remark}

\subsection{Related Works}
\label{sec-prior}


\textbf{Wasserstein distance.} If building a bipartite graph, where the two columns of vertices correspond to the points of $A$ and $B$ respectively and each edge connecting $(a_i, b_j)$ has the weight $||a_i-b_j||^2$, we can see that computing the Wasserstein distance actually is a min-cost flow problem from $A$ to $B$. 
A number of minimum cost flow algorithms have been developed in the past decades~\cite{ahuja1993network,orlin1993polynomial,orlin1997polynomial,DBLP:journals/combinatorica/Tardos85,DBLP:journals/jacm/GoldbergT89}. Suppose $n$ and $m$ are the numbers of vertices and edges in the bipartite graph respectively, and $U$ is the maximum weight. Orlin~\cite{DBLP:conf/stoc/Orlin88} developed a strongly polynomial algorithm with the time complexity $O(n\log n (m+n\log n))$. 
Lee and  Sidford~\cite{DBLP:conf/focs/LeeS14} proposed an algorithm that can solve the minimum cost flow problem in $O(n^{2.5} \mathtt{poly}(\log U))$ time. Very recently, Chen~{\em et al.}~\cite{chen2022maximum} improved the time complexity to be $O(m^{1+o(1)})$. 
Sherman~\cite{DBLP:conf/soda/Sherman17} provided a $(1+\epsilon)$-approximation algorithm based on preconditioning, and  the running time is  $O(n^{2+o(1)}\epsilon^{-2})$. Andoni {\em et al.}~\cite{DBLP:conf/stoc/AndoniSZ20} further proposed a parallel $(1+\epsilon)$-approximation algorithm that runs in $\frac{1}{\epsilon^2}\mathtt{poly (log n)}$ time using $\tilde{O}(\frac{1}{\epsilon^2}m)$ expected work. 

In the community of machine learning, Cuturi~\cite{DBLP:conf/nips/Cuturi13} proposed  a much faster ``Sinkhorn Distance'' algorithm which yields an approximation for the Wasserstein distance. Following Cuturi's work, Altschuler {\em et al.}~\cite{DBLP:conf/nips/AltschulerWR17} provided an additive approximation~\footnote{The ``additive approximation'' means that the computed Wasserstein distance is at most $OPT+\epsilon$, if $OPT$ is the optimal Wasserstein distance.} algorithm for computing Wasserstein distance with the running time $O(n^2L^3\log n \epsilon^{-3})$, where $L$ is the maximum edge weight in the bipartite graph. 
Recently, the problems of robust Wasserstein distance and unbalanced Wasserstein distance ({\em e.g.,} allow a small number of outliers) have attracted a lot of attention in the machine learning applications~\cite{mukherjee2021outlier,DBLP:conf/nips/ChapelFWFG21}.

When we focus on the Wasserstein distance problem in Euclidean space $\mathbb{R}^d$, a number of faster algorithms have been proposed in  computational geometry~\cite{DBLP:conf/compgeom/AgarwalV04,DBLP:conf/soda/VaradarajanA99,indyk2007near,cabello2008matching,DBLP:conf/stoc/AndoniNOY14,DBLP:conf/stoc/SharathkumarA12,DBLP:conf/soda/SharathkumarA12,DBLP:journals/siamcomp/Vaidya89a}. 
Recently, Agarwal {\em et al.}~\cite{DBLP:conf/compgeom/AgarwalFPVX17} improved the running time to be $O\big(n^{3/2}\epsilon^{-d}\mathtt{poly}(\log U)\mathtt{poly}(\log n)\big)$.  Using the idea of preconditioning~\cite{DBLP:conf/soda/Sherman17}, Khesin {\em et al.}~\cite{DBLP:conf/compgeom/KhesinNP19} developed two randomized $(1+\epsilon)$-approximation algorithms with the running times $O(n\epsilon^{-O(d)}\log (M)^{O(d)}\log n)$ and $O(n\epsilon^{-O(d)}$ $\log (U)^{O(d)}\log n)$, respectively ($M$ is the aspect ratio of the given point sets). Fox and
Lu~\cite{DBLP:journals/jocg/FoxL22} also presented a near-linear  time algorithm but  their time bound is independent of $M$ and $U$. 
Most of these algorithms rely on the low-dimensional geometric techniques, and therefore their complexities are exponential in the dimensionality $d$.  Li~\cite{DBLP:journals/corr/abs-1002-4034} studied the problem of estimating the Wasserstein distance; his idea is a generalization of the method of~\cite{indyk2007near} that yields an $O(\rho)$-approximate estimate of the Wasserstein distance, where $\rho$ is the doubling dimension of the given data; but  the approximation factor is relatively large and the algorithm needs at least a quadratic preprocessing time. Recently Chen {\em et al.}~\cite{chen2022new} applied the Quadtree to design a streaming Wasserstein distance algorithm in high dimensions, where the approximation factor is $\tilde{O}(\log n)$. Ding {\em et al.}~\cite{DBLP:conf/sdm/DingCYW21} considered the Wasserstein distance query problem for quickly answering the question that whether the distance is lower or higher than a query number; their algorithm is based on a hierarchical $k$-center clustering method in  doubling metric.

\textbf{Geometric alignment.} Computing the geometric alignment of $A$ and $B$ is more challenging, since we need to determine the rigid transformation and the Wasserstein flow simultaneously. Moreover, due to the flexibility of rigid transformations, we cannot apply the Wasserstein distance embedding techniques~\cite{IT03,andoni2009efficient} to resolve the challenges. For example,  the embedding can only preserve the Wasserstein distance between $A$ and $B$; however, since there are infinite number of possible rigid transformations $\mathcal{T}$ for $B$ (note that we do not know $\mathcal{T}$ in advance), it is difficult to also preserve the Wasserstein distance between $A$ and $\mathcal{T}(B)$. In theory, 
Cabello {\em et al.}~\cite{cabello2008matching} presented a $(2+\epsilon)$-approximation solution for the 2D case, and later Klein and  Veltkamp~\cite{klein2005approximation} achieved an $O(2^{d-1})$-approximation in $\mathbb{R}^d$; 
Ding and Xu~\cite{DBLP:journals/algorithmica/DingX17} proposed a PTAS for constant dimensional space. However, these theoretical algorithms cannot be efficiently implemented when the dimensionality is not constant. 
It was also mentioned in~\cite{DBLP:journals/algorithmica/DingX17}  that any constant factor approximation needs a time complexity at least $n^{\Omega(d)}$ based on some reasonable assumption in the theory of computational complexity, where $n=\max\{|A|, |B|\}$. That is, it is unlikely to obtain a $(1+\epsilon)$-approximation within a practical running time, especially when $n$ is very large.  
The recent works \cite{DBLP:journals/corr/NasserJF15,DBLP:journals/corr/abs-2101-03588} proposed a core-set based compression approach to speed up the computation of alignment. However,  their compression methods only work for the case that  $d$ is small. We also refer the reader to the recent surveys~\cite{DBLP:journals/ki/MunteanuS18,DBLP:journals/widm/Feldman20} for the detailed introductions on coresets.

In practice, Cohen and  Guibas~\cite{cohen1999earth} proposed an alternating minimization approach for computing the geometric alignment of $A$ and $B$. 
Several other approaches~\cite{cornea20053d,todorovic2008region} based on graph matching are inappropriate to be extended for high dimensional alignment. 
In machine learning, a related topic is called ``manifold alignment''~\cite{ham2005semisupervised,wang2011manifold}; however, it usually has different settings and applications, and thus is out of the scope of this paper. 

\subsection{Overview of The Alignment Procedure}
\label{sec-oa}

Because the approach of~\cite{cohen1999earth}  is closely related to our proposed algorithm, we introduce it with more details for the sake of completeness. 
Roughly speaking, their approach is similar to the {\em Iterative Closest Point method} (ICP) method~\cite{besl1992method}, where in each iteration it  alternatively updates the Wasserstein distance flow and the rigid transformation. Thus it converges to some local optimum eventually. 
To update the rigid transformation, we can apply the {\em Orthogonal Procrustes (OP) analysis}~\cite{schonemann1966generalized}. The original OP analysis is only for unweighted point sets, and the weighted OP analysis was studied in the Wahba's problem~\cite{wahba1965least}. But our problem with the Wasserstein distance flow is a special case of the weighted OP analysis, and thus the complexity can be further reduced via a more careful analysis below.


Let $A$ and $B$ be the two sets of weighted points for alignment (as Definition~\ref{def-align}). Suppose that the Wasserstein distance flow $F=\{f_{ij}\}$ is fixed and the rigid transformation is waiting to update in the current stage.
We imagine that there exist two new sets of weighted points  
\begin{eqnarray}
	\hat{A}&=&\cup^{n_1}_{i=1}\{a^1_i, a^2_i,\cdots, a^{n_2}_i\};\label{for-cg1}\\
	\hat{B}&=&\cup^{n_2}_{j=1}\{b^1_j, b^2_j,\cdots, b^{n_1}_j\},\label{for-cg2}
\end{eqnarray}
where each $a^j_i$ ({\em resp.,} $b^i_j$) has the weight $f_{ij}$ and the same spatial position of  $a_i$ ({\em resp.,} $b_j$). That is, each $a_i$ ({\em resp.,} $b_j$) is decomposed into $n_2$ ({\em resp.,} $n_1$) copies.  
First, we take a translation vector $\overrightarrow{v}$ such that the weighted mean points of $\hat{A}$ and $\hat{B}+\overrightarrow{v}$ coincide with each other (this can be easily derived, due to the fact that the objective function uses squared distance~\cite{cohen1999earth}). Second, by using the weighted OP analysis, we compute an orthogonal matrix $\mathcal{R}$ for $\hat{B}+\overrightarrow{v}$ to minimize its weighted $L^2_2$ difference to $\hat{A}$. For this purpose, we generate two $d\times (n_1 n_2)$ matrices $\mathbb{M}_A$ and $\mathbb{M}_B$ as follows.

We use $\mathbbm{a}_i$ ({\em resp.,} $\mathbbm{b}_j$) to denote the corresponding $d$-dimensional column vector of $a_i$ ({\em resp.,} $b_j$).
Each point of $\hat{A}$ ({\em resp.,} $\hat{B}+\overrightarrow{v}$) corresponds to an individual column of $\mathbb{M}_A$ ({\em resp.,} $\mathbb{M}_B$); for example, a point $a^j_i\in \hat{A}$ ({\em resp.,} $b^i_j+\overrightarrow{v}\in \hat{B}+\overrightarrow{v}$) corresponds to a column $\sqrt{f_{ij}}\mathbbm{a}_i$ ({\em resp.,} $\sqrt{f_{ij}}(\mathbbm{b}_j+\overrightarrow{v})$) in $\mathbb{M}_A$ ({\em resp.,} $\mathbb{M}_B$). Formally, 
\begin{eqnarray}
	\mathbb{M}_A&=&\big[\sqrt{f_{11}}\mathbbm{a}_1, \sqrt{f_{12}}\mathbbm{a}_1,\cdots, \sqrt{f_{1n_2}}\mathbbm{a}_1,\sqrt{f_{21}}\mathbbm{a}_2, \sqrt{f_{22}}\mathbbm{a}_2,\cdots, \sqrt{f_{2n_2}}\mathbbm{a}_2,\nonumber\\
	&&\cdots,\sqrt{f_{n_1 1}}\mathbbm{a}_{n_1}, \sqrt{f_{n_1 2}}\mathbbm{a}_{n_1},\cdots, \sqrt{f_{n_1 n_2}}\mathbbm{a}_{n_1}\big];\label{for-cg1-matrix}\\
	\mathbb{M}_B&=&\big[\sqrt{f_{11}}(\mathbbm{b}_1+\overrightarrow{v}), \sqrt{f_{12}}(\mathbbm{b}_2+\overrightarrow{v}),\cdots, \sqrt{f_{1n_2}}(\mathbbm{b}_{n_2}+\overrightarrow{v}), \nonumber \\
	&&\sqrt{f_{21}}(\mathbbm{b}_1+\overrightarrow{v}), \sqrt{f_{22}}(\mathbbm{b}_2+\overrightarrow{v}),\cdots, \sqrt{f_{2n_2}}(\mathbbm{b}_{n_2}+\overrightarrow{v}),\nonumber\\
	&&\cdots,\sqrt{f_{n_1 1}}(\mathbbm{b}_1+\overrightarrow{v}), \sqrt{f_{n_1 2}}(\mathbbm{b}_2+\overrightarrow{v}),\cdots, \sqrt{f_{n_1 n_2}}(\mathbbm{b}_{n_2}+\overrightarrow{v})\big].\label{for-cg2-matrix}
\end{eqnarray}
Let  the SVD of $\mathbb{M}_A\times \mathbb{M}^T_B$ be $U\Sigma V^T$, and the optimal orthogonal matrix $\mathcal{R}$ should be $UV^T$ through the OP analysis.

The above idea has a drawback that the matrices $\mathbb{M}_A$ and $\mathbb{M}_B$ are too large and thus the complexity for computing the multiplication $\mathbb{M}_A\times \mathbb{M}^T_B$ can be very high (which is $O(n_1n_2 d^2)$).  Actually we do not need to really construct the large matrices $\mathbb{M}_A$ and $\mathbb{M}_B$, since many of the columns are identical (except for the scalar $\sqrt{f_{ij}}$).  Instead, we can compute the multiplication $\mathbb{M}_A\times \mathbb{M}^T_B$ in $O(n_1n_2d+\min\{n_1, n_2\}\cdot d^2)$ time. \textbf{Note:} Both the complexities ``$O(n_1n_2 d^2)$'' and ``$O(n_1n_2d+\min\{n_1, n_2\}\cdot d^2)$'' can be slightly improved by using the faster rectangle matrix multiplication algorithms~\cite{le2012faster}, but it needs a somewhat complicated discussion on the relation of the values $n_1$, $n_2$, and $d$, and so it is out of the scope of this paper.

\begin{claim}
	\label{lem-app}
	The multiplication $\mathbb{M}_{A}\times \mathbb{M}_{B}^{T}$ can be computed in $O(n_{1}n_{2}d+\min\{n_{1},n_{2}\}\cdot d^{2})$ time.
\end{claim}
\begin{proof}
	With a slight abuse of notations, we also use $F$ to denote the $n_1\times n_2$ matrix of the Wasserstein distance flow where each entry is $f_{ij}$; also, ``$F_{i,:}$'' represents the $i$-th row of the matrix $F$. Given a vector $t$, we use $\sqrt{t}$ to denote the new vector with each entry being the square root of the corresponding one in $t$. Also, we use $\mathtt{diag}(t)$ to denote the diagonal matrix where the $i$-th diagonal entry is the $i$-th entry of $t$. For example, if $t=[t_1, t_2, \cdots, t_n]$, then $\sqrt{t}=[\sqrt{t_1}, \sqrt{t_2}, \cdots, \sqrt{t_n}]$ and 
	$$\mathtt{diag}(\sqrt{t})=
	\begin{bmatrix}
		\sqrt{t_1},~~~0~,~~~0~,\cdots,~~~0~\\
		~~~0~,\sqrt{t_2},~~~0~,\cdots,~~~0~ \\
		~~~0~,~~~0~, \sqrt{t_3},\cdots,~~~0~ \\
		\cdots,\cdots,\cdots ~ \\
		~~~~~0~,~~~0~,~~~0~,\cdots,\sqrt{t_n} \\
	\end{bmatrix}.
	$$

	Following the constructions of $\mathbb{M}_A$ and $\mathbb{M}_B$ with some simple calculation, we have 
	\begin{eqnarray}
		\mathbb{M}_A&=&\big[\sqrt{f_{11}}\mathbbm{a}_1, \sqrt{f_{12}}\mathbbm{a}_1,\cdots, \sqrt{f_{1n_2}}\mathbbm{a}_1,\sqrt{f_{21}}\mathbbm{a}_2, \sqrt{f_{22}}\mathbbm{a}_2,\cdots, \sqrt{f_{2n_2}}\mathbbm{a}_2,\nonumber\\
		&&\cdots,\sqrt{f_{n_1 1}}\mathbbm{a}_{n_1}, \sqrt{f_{n_1 2}}\mathbbm{a}_{n_1},\cdots, \sqrt{f_{n_1 n_2}}\mathbbm{a}_{n_1}\big]\nonumber\\
		&=&\big[\mathbbm{a}_{1}\sqrt{F_{1,:}},\mathbbm{a}_{2}\sqrt{F_{2,:}},\cdots,\mathbbm{a}_{n_{1}}\sqrt{F_{n_{1},:}}\big];\nonumber\\
		\mathbb{M}_B&=&\big[\sqrt{f_{11}}(\mathbbm{b}_1+\overrightarrow{v}), \sqrt{f_{12}}(\mathbbm{b}_2+\overrightarrow{v}),\cdots, \sqrt{f_{1n_2}}(\mathbbm{b}_{n_2}+\overrightarrow{v}),\nonumber\\
		&&\sqrt{f_{21}}(\mathbbm{b}_1+\overrightarrow{v}), \sqrt{f_{22}}(\mathbbm{b}_2+\overrightarrow{v}),\cdots, \sqrt{f_{2n_2}}(\mathbbm{b}_{n_2}+\overrightarrow{v}),\nonumber\\
		&&\cdots,\sqrt{f_{n_1 1}}(\mathbbm{b}_1+\overrightarrow{v}), \sqrt{f_{n_1 2}}(\mathbbm{b}_2+\overrightarrow{v}),\cdots, \sqrt{f_{n_1 n_2}}(\mathbbm{b}_{n_2}+\overrightarrow{v})\big]\nonumber\\
		&=&\big[\mathbbm{b}_1+\overrightarrow{v},\cdots,\mathbbm{b}_{n_2}+\overrightarrow{v}\big]\times\big[\mathtt{diag}(\sqrt{F_{1,:}}),\mathtt{diag}(\sqrt{F_{2,:}}),\cdots,\mathtt{diag}(\sqrt{F_{n_{1},:}})\big].\nonumber
	\end{eqnarray}
	For simplicity, we let $\mathbb{A}=\big[\mathbbm{a}_1,\cdots, \mathbbm{a}_{n_1}\big]$ and $\mathbb{B}=\big[\mathbbm{b}_1+\overrightarrow{v},\cdots,\mathbbm{b}_{n_2}+\overrightarrow{v}\big]$. 
	Then,
	\begin{align*}
		\mathbb{M}_{A}\times \mathbb{M}_{B}^{T}&=\sum_{i=1}^{n_{1}}\big(\mathbbm{a}_{i}\sqrt{F_{i,:}}\big)\times\big(\mathtt{diag}(\sqrt{F_{i,:}})\mathbb{B}^{T}\big)\\
		&=\sum_{i=1}^{n_{1}}\mathbbm{a}_{i}F_{i,:}\mathbb{B}^{T}=\mathbb{A}F\mathbb{B}^{T}.
	\end{align*}
	The sizes of $\mathbb{A}$, $F$, and $\mathbb{B}$ are $d\times n_1$, $n_1\times n_2$, and $d\times n_2$, respectively. So 
	it is easy to see that computing ``$\mathbb{A}F\mathbb{B}^{T}$'' takes $O(n_{1}n_{2}d+\min\{n_{1},n_{2}\}\cdot d^{2})$ time ({\em e.g.,} if $n_1\geq n_2$, we compute $\mathbb{A}F$ first and then compute $(\mathbb{A}F)\times \mathbb{B}^{T}$).
\end{proof}

Therefore, the time complexity for obtaining the optimal $\mathcal{R}$ is $O(n_1n_2d+\min\{n_1, n_2\}\cdot d^2+d^3)$, where the term ``$d^3$'' comes from the computation of the SVD.

\begin{proposition}
	\label{pro-cg}
	Each iteration of the approach of~\cite{cohen1999earth} takes $\Gamma(n_1, n_2, d)+O(n_1n_2d+\min\{n_1, n_2\}\cdot d^2+d^3)$ time, where $\Gamma(n_1, n_2, d)$ denotes the time complexity of the used Wasserstein distance algorithm. In practice, we usually assume $n_1, n_2=O(n)$ with some $n\geq d$, and then the complexity can be simply written as $\Gamma(n, d)+O(n^2d)$.
\end{proposition}

The bottleneck is that the algorithm needs to repeatedly compute the Wasserstein distance and transformation, especially when $n$ and $d$ are large (usually $\Gamma(n,d)=\Omega(n^2 d)$). Based on the property of low doubling dimension, we construct a pair of compressed point sets to replace the original $A$ and $B$, and run the same algorithm on the compressed data instead. As a consequence, the running time is reduced significantly. Note that our compression step is \textbf{independent of} the approach~\cite{cohen1999earth}; actually, any alignment method with the same objective function in Definition~\ref{def-align} can benefit from our compression idea.

\section{The Algorithm and Analysis}
\label{sec-aa}
Our idea starts from the widely studied 
$k$-center clustering problem.
Given an integer $k\geq 1$ and a point set $P$ in some metric space, the $k$-center clustering is to partition $P$ into $k$ clusters and cover each cluster by an individual ball, such that the maximum radius of the balls is minimized. Gonzalez~\cite{gonzalez1985clustering} presented an elegant $2$-approximation algorithm, where the radius of each resulting ball ({\em i.e.,} cluster) is at most two times the optimum. Initially, it  sets $S=\{c_1\}$ where $c_1$ is an arbitrary point selected from $P$; then in each of the following $k-1$ iterations,  a new point who has the largest distance to $S$ among the points of $P$ is added to $S$ (we define the distance between a point $q$ and $S$ to be $\min\{||q-p||\mid p\in S\}$). Let $S=\{c_1, \cdots, c_k\}$, and then $P$ is covered by the $k$ balls $\mathtt{Ball}(c_1, r), \cdots, \mathtt{Ball}(c_k, r)$ with  
\begin{eqnarray}
	r\leq\min\{||c_i-c_j||\mid 1\leq i\neq j\leq k\}.\label{for-2pairwise} 
\end{eqnarray}
It is easy to prove that $r$ is at most two times the optimal radius of the given instance.

Let $P$ be a point set in $ \mathbb{R}^d$ with the doubling dimension $\rho$. The diameter of $P$ is denoted by $\Delta$, {\em i.e.}, $\Delta=\max\{||p-q||\mid p, q\in P\}$.
Then we have the following lemma.

\begin{lemma}
	\label{lem-number}
	Given an integer $k\geq 1$, if one runs the Gonzalez's $k$-center clustering algorithm,  the obtained radii of  the clusters are at most $\frac{2}{k^{1/\rho}}\Delta$. 
\end{lemma}
\begin{proof}
	Let $S$ be the set of $k$ points obtained by the Gonzalez's algorithm, and the obtained radius be $r$. We also define the aspect ratio of $S$ as the ratio of the maximum to the  minimum pairwise distance in $S$. Then, it is easy to see that the aspect ratio of $S$ is at most $\Delta/r$ from (\ref{for-2pairwise}). Now, we need the following Claim~\ref{claim-compress} from~\cite{krauthgamer2004navigating,talwar2004bypassing}. Actually, the claim can be obtained by recursively applying the definition of doubling dimension.
	
		
		\begin{claim}
			\label{claim-compress}
			Let $(X, d_X)$ be a metric space with the doubling dimension $\rho$, and $Y\subset X$. If the aspect ratio of $Y$ is upper bounded by some positive value $\alpha$, then $|Y|\leq 2^{\rho\lceil\log_2 \alpha\rceil}$.
		\end{claim}
	
	Replacing $X$ and $Y$ by $P$ and $S$ respectively in the above claim, we have
	\begin{eqnarray}
		|S|\leq 2^{\rho\lceil\log_2\Delta/r\rceil}\leq 2^{\rho (1+\log_2\Delta/r)}. \label{for-number1} 
	\end{eqnarray}
	Since $|S|=k$, (\ref{for-number1}) implies $r\leq \frac{2}{k^{1/\rho}}\Delta$.
\end{proof}

\begin{corollary}
	\label{rem-k}
	If we let $k=(\frac{2}{\epsilon})^\rho$ with some small $\epsilon>0$, the radius in Lemma~\ref{lem-number} will be $\epsilon\Delta$.
	%
	%
	%
\end{corollary}

\textbf{Our compression algorithm.} Let $A$ and $B$ be the two given point sets in Definition~\ref{def-align}, and we assume their diameters are $\Delta_A$ and $\Delta_B$ respectively. Let $\Delta=\max\{\Delta_A, \Delta_B\}$. We also assume that they both have the doubling dimension at most $\rho$. Our idea for compressing $A$ and $B$ is as follows. As described in Lemma~\ref{lem-number}, we  run the Gonzalez's algorithm on $A$ and $B$ respectively. We denote by $S_A=\{c^A_1,\cdots, c^A_k\}$ and $S_B=\{c^B_1,\cdots, c^B_k\}$ the obtained sets of $k$-cluster centers with $k=(\frac{2}{\epsilon})^\rho$. For each cluster center $c^A_j$ ({\em resp.,} $c^B_j$), we assign a weight that is equal to the total weight of the points in the corresponding cluster. As a consequence, we obtain a new instance $(S_A, S_B)$ for geometric alignment. It is easy to know that the total weight of $S_A$ ({\em resp.,} $S_B$) is equal to $W_A$ ({\em resp.,} $W_B$). For the sake of convenience, we name this compression method as \textsc{KCenter}. The following theorem shows that one can achieve an approximate solution for the instance $(A, B)$ by solving the alignment of $(S_A, S_B)$.


\begin{theorem}
	\label{the-quality}
	Suppose $\epsilon>0$ is a small parameter in Corollary~\ref{rem-k}. Given any $c\geq 1$, let $\tilde{\mathcal{T}}$ be a rigid transformation yielding $c$-approximation for minimizing $\mathcal{W}^2_2\big(S_A, \mathcal{T}(S_B)\big)$ in Definition~\ref{def-align}. Then,  
	\begin{eqnarray}
		\mathcal{W}^2_2\big(A, \tilde{\mathcal{T}}(B)\big)&\leq& c(1+2\epsilon)^2\cdot\min_\mathcal{T}\mathcal{W}^2_2\big(A, \mathcal{T}(B)\big)+2\epsilon(c+1+2c\epsilon)(1+2\epsilon)\Delta^2\nonumber\\
		&=& c\big(1+O(\epsilon)\big)\cdot\min_\mathcal{T}\mathcal{W}^2_2\big(A, \mathcal{T}(B)\big)+2\epsilon \big(1+O(\epsilon)\big)(c+1)\Delta^2. \label{for-quality1}
	\end{eqnarray}
\end{theorem}
\begin{proof}
	First, we denote by $\mathcal{T}_{\mathtt{opt}}$ the optimal rigid transformation achieving $\min_\mathcal{T}\mathcal{W}^2_2\big(A, \mathcal{T}(B)\big)$. Since $\tilde{\mathcal{T}}$ yields $c$-approximation for minimizing $\mathcal{W}^2_2\big(S_A, \mathcal{T}(S_B)\big)$, we have 
	\begin{eqnarray}
		\mathcal{W}^2_2\big(S_A, \tilde{\mathcal{T}}(S_B)\big)&\leq& c\cdot \min_\mathcal{T}\mathcal{W}^2_2\big(S_A, \mathcal{T}(S_B)\big) \nonumber\\
		&\leq& c\cdot\mathcal{W}^2_2\big(S_A, \mathcal{T}_{\mathtt{opt}}(S_B)\big).\label{for-quality3}
	\end{eqnarray}
	
	Recall that each point $c^A_j$ ({\em resp.,} $c^B_j$) has the weight equal to the total weights of the points in the corresponding cluster. For instance, if the cluster contains $\{a_{j(1)}, a_{j(2)}, \cdots, a_{j(h)}\}$, the weight of $c^A_j$ should be $\sum^h_{l=1}\alpha_{j(l)}$; actually, we can view $c^A_j$ as $h$ overlapping points $\{a'_{j(1)}, a'_{j(2)}, \cdots, a'_{j(h)}\}$ with each $a'_{j(l)}$ having the weight $\alpha_{j(l)}$. 
	Therefore, for the sake of convenience, we use another representation for $S_A$ and $S_B$ in our proof below:
	\begin{eqnarray}
		S_A=\{a'_1, \cdots, a'_{n_1}\}\text{ and } S_B=\{b'_1, \cdots, b'_{n_2}\}, \label{for-newrep}
	\end{eqnarray}
	where each $a'_j$ ({\em resp.,} $b'_j$) has the weight $\alpha_j$ ({\em resp.,} $\beta_j$). Note that $S_A$ and $S_B$ only have $k$ distinct positions respectively in the space. Moreover, due to Corollary~\ref{rem-k}, we know that $||a'_i-a_i||$, $||b'_j-b_j||\leq \epsilon \Delta $ for any $1\leq i\leq n_1$ and $1\leq j\leq n_2$, and these bounds are invariant under any rigid transformation in the space. Consequently, for any pair $(i, j)$ and any  rigid transformation $\mathcal{T}$, we have  
	\begin{eqnarray}
		&&||a_i-\mathcal{T}(b_j)||^2 \nonumber\\
		&\leq& \big(||a_i-a'_i||+||a'_i-\mathcal{T}(b'_j)||+||\mathcal{T}(b'_j)-\mathcal{T}(b_j)||\big)^2\nonumber\\
		&\leq&\big(||a'_i-\mathcal{T}(b'_j)||+2\epsilon\Delta\big)^2\nonumber\\
		&=&||a'_i-\mathcal{T}(b'_j)||^2+4\epsilon\Delta||a'_i-\mathcal{T}(b'_j)||+4\epsilon^2\Delta^2\nonumber\\
		&\leq&||a'_i-\mathcal{T}(b'_j)||^2+2\epsilon\big(\Delta^2+||a'_i-\mathcal{T}(b'_j)||^2\big) +4\epsilon^2\Delta^2 \nonumber\\
		&=&(1+2\epsilon)||a'_i-\mathcal{T}(b'_j)||^2+(2\epsilon+4\epsilon^2)\Delta^2\label{for-quality4}
	\end{eqnarray}
	through the triangle inequality. Using exactly the same idea, we also have  
	\begin{eqnarray}
		||a'_i-\mathcal{T}(b'_j)||^2\leq (1+2\epsilon)||a_i-\mathcal{T}(b_j)||^2+(2\epsilon+4\epsilon^2)\Delta^2.\label{for-quality5}
	\end{eqnarray}
	Based on Definition~\ref{def-emd}, we denote by $\tilde{F}=\{\tilde{f}_{ij}\}$ the induced flow of $\mathcal{W}^2_2\big(S_A, \tilde{\mathcal{T}}(S_B)\big)$ (using the representations (\ref{for-newrep}) for $S_A$ and $S_B$). Then (\ref{for-quality4}) directly implies that  
	\begin{eqnarray}
		&&\mathcal{W}^2_2\big(A, \tilde{\mathcal{T}}(B)\big) \nonumber\\
		&\leq&\frac{1}{\min\{W_A, W_B\}}\sum^{n_1}_{i=1}\sum^{n_2}_{j=1}\tilde{f}_{ij}||a_{i}-\tilde{\mathcal{T}}(b_j)||^2\nonumber\\
		&\leq&\frac{1+2\epsilon}{\min\{W_A, W_B\}}\sum^{n_1}_{i=1}\sum^{n_2}_{j=1}\tilde{f}_{ij}||a'_{i}-\tilde{\mathcal{T}}(b'_j)||^2 +(2\epsilon+4\epsilon^2)\Delta^2\nonumber\\
		&=&(1+2\epsilon)\mathcal{W}^2_2(S_A, \tilde{\mathcal{T}}(S_B)) +(2\epsilon+4\epsilon^2)\Delta^2.\label{for-quality6}
	\end{eqnarray}
	By using the similar idea (replacing $\tilde{\mathcal{T}}$ by $\mathcal{T}_{\mathtt{opt}}$, and exchanging the roles of $(A, B)$ and $(S_A, S_B)$),    (\ref{for-quality5}) directly implies that 
	\begin{eqnarray}
		\mathcal{W}^2_2\big(S_A, \mathcal{T}_{\mathtt{opt}}(S_B)\big)&\leq&(1+2\epsilon)\mathcal{W}^2_2\big(A, \mathcal{T}_{\mathtt{opt}}(B)\big) +(2\epsilon+4\epsilon^2)\Delta^2.\label{for-quality7}
	\end{eqnarray}
	Combining (\ref{for-quality3}), (\ref{for-quality6}), and (\ref{for-quality7}), we have  
	\begin{eqnarray}
		\mathcal{W}^2_2\big(A, \tilde{\mathcal{T}}(B)\big)&\leq&(1+2\epsilon)\mathcal{W}^2_2\big(S_A, \tilde{\mathcal{T}}(S_B)\big)+(2\epsilon+4\epsilon^2)\Delta^2\nonumber\\
		&\leq&(1+2\epsilon)\cdot c\cdot\mathcal{W}^2_2\big(S_A, \mathcal{T}_{\mathtt{opt}}(S_B)\big) +(2\epsilon+4\epsilon^2)\Delta^2 \nonumber\\
		&\leq& c(1+2\epsilon)^2\cdot\mathcal{W}^2_2\big(A, \mathcal{T}_{\mathtt{opt}}(B)\big)+2\epsilon(c+1+2c\epsilon)(1+2\epsilon)\Delta^2,
	\end{eqnarray}
	and the proof is completed.
\end{proof}

When $\epsilon$ is small enough, Theorem~\ref{the-quality} shows that $\mathcal{W}^2_2\big(A, \tilde{\mathcal{T}}(B)\big)\approx c\cdot\mathcal{W}^2_2\big(A, \mathcal{T}_{\mathtt{opt}}(B)\big)$. That is, $\tilde{\mathcal{T}}$, the solution of $(S_A, S_B)$, achieves roughly the same performance on $(A, B)$. 
Consequently, we propose the approximation algorithm for geometric alignment (see Algorithm~\ref{alg-1}). We would like to emphasize that though we use the algorithm from~\cite{cohen1999earth} in Step 3, Theorem~\ref{the-quality} is an independent result; that is, any alignment method with the same objective function in Definition~\ref{def-align} can benefit from Theorem~\ref{the-quality}.

\renewcommand{\algorithmicrequire}{\textbf{Input:}}
\renewcommand{\algorithmicensure}{\textbf{Output:}}
\begin{algorithm}
	\caption{Geometric Alignment with \textsc{KCenter}}
	\label{alg-1}
	\begin{algorithmic}[1]
		\REQUIRE An instance $(A, B)$ of the geometric alignment problem in Definition~\ref{def-align} with bounded doubling dimension $\rho$ in $\mathbb{R}^d$; $k\in \mathbb{Z}^+$.
		\ENSURE A rigid transformation $\mathcal{T}$ of $B$ and the Wasserstein distance flow between $A$ and $\mathcal{T}(B)$.
		\STATE Run the Gonzalez's $k$-center clustering algorithm on $A$ and $B$ as described in Theorem~\ref{the-quality}, and obtain the sets of cluster centers $S_A$ and $S_B$ respectively.
		\STATE Apply the existing alignment algorithm, e.g., the algorithm of~\cite{cohen1999earth}, on $(S_A, S_B)$.
		\STATE Obtain the rigid transformation $\mathcal{T}$ from Step 2, and compute the corresponding Wasserstein distance flow between $A$ and $\mathcal{T}(B)$.
	\end{algorithmic}
\end{algorithm}

\subsection{Extension \Rmnum{1}: When $k$ Is Not Given}
\label{sec-ext1}

As discussed in Corollary~\ref{rem-k},  the value  $k=(2/\epsilon)^\rho$  depends on the doubling dimension $\rho$, if we require the compression error ({\em i.e.,} the radius) to be no larger than $\epsilon\Delta$. However, the exact doubling dimension $\rho$ usually is not easy to obtain~\cite{har2006fast}. So we consider another scenario. Let ``$\epsilon\Delta$'' be the pre-specified compression error bound, and we run the Gonzalez's algorithm iteratively until the radius is reduced to be no larger than the bound. The question is that 

{\em can our algorithm be aware when the stopping condition is reached? }

We answer this question in the affirmative, where the only change is that the value of $k$ is required to be larger than $(2/\epsilon)^\rho$. In the worst case, $k$ can be as large as $(4/\epsilon)^\rho$ (so the value $k$ is enlarged by a factor $2^\rho$). 

First, we need to estimate the diameter $\Delta$. In practice we often avoid to compute the exact value of $\Delta$ since it takes at least quadratic complexity. Instead, we can simply take a point $p$ and its farthest point $q$ from $P$; let $\tilde{\Delta}=||p-q||$ and then we know 
\begin{eqnarray}
	\frac{1}{2}\Delta\leq \tilde{\Delta}\leq \Delta.\label{for-estdia}
\end{eqnarray} 
Then we have the following lemma (which is a counterpart of Lemma~\ref{lem-number}).

\begin{lemma}
	\label{lem-number-2}
	Given a small parameter $\epsilon>0$, if one runs the Gonzalez's $k$-center clustering algorithm iteratively until the obtained radius is no larger than $\epsilon\tilde{\Delta}$,  the number of obtained clusters is  at most $(\frac{4}{\epsilon})^\rho$. 
\end{lemma}
\begin{proof}
	Let $S$ be the set of $k$ points by Gonzalez's algorithm, and the obtained  radius be $r\leq \epsilon \tilde{\Delta}$. It is easy to see that the aspect ratio of $S$ is at most $\Delta/(\epsilon \tilde{\Delta})\leq \frac{2}{\epsilon}$ (by (\ref{for-estdia})). 
	%
	%
	%
	From Claim~\ref{claim-compress}, we have
	\begin{eqnarray}
		|S|\leq 2^{\rho\lceil\log_2 2/\epsilon\rceil}\leq 2^{\rho (1+\log_2 2/\epsilon)}=(4/\epsilon)^\rho. \label{for-number1-2} 
	\end{eqnarray}
\end{proof}

From Lemma~\ref{lem-number-2}, we just need to modify  Step 1 of Algorithm~\ref{alg-1} for the case that the doubling dimension $\rho$ is not given.


\subsection{Extension \Rmnum{2}:  $k$-means $+$ $k$-center}
\label{sec-ext2}
In Theorem~\ref{the-quality}, we show that the Gonzalez's $k$-center clustering algorithm can be used to compress input data. So a natural question is that whether other clustering method, such as the widely used $k$-means clustering, can be also applied to achieve this purpose. To answer this question, we need to revisit the proof of Theorem~\ref{the-quality}. 
%
%
In $k$-center clustering, each obtained cluster has a sufficiently small  radius, and so we can obtain the inequalities (\ref{for-quality4}) and (\ref{for-quality5}). But if we run $k$-means instead, the obtained clusters may have large radii and thus (\ref{for-quality4}) and (\ref{for-quality5}) can be violated. This observation inspires the following compression method that combines $k$-center and $k$-means.  

\textbf{\textsc{KCenter+}.} In Step 1 of Algorithm~\ref{alg-1}, the obtained sets $S_A$ and $S_B$ are the corresponding ball centers (each cluster is covered by a ball with radius $\leq\epsilon\Delta$ if $k\geq(\frac{2}{\epsilon})^\rho$). Actually we can further improve the result locally. For each ball center, we can replace it by \textbf{the mean point} of the cluster. Since our Wasserstein distance $\mathcal{W}^2_2(\cdot, \cdot)$ is the sum of a set of weighted squared Euclidean distances, the additive error bound ``$2\epsilon(c+1+2c\epsilon)(1+2\epsilon)\Delta^2$'' in (\ref{for-quality1}) can be reduced. 
For example, in the second inequality of (\ref{for-quality4}), the items ``$||a_i-a'_i||$'' and ``$||\mathcal{T}(b'_j)-\mathcal{T}(b_j)||$'' are just simply bounded by $\epsilon \Delta$; but if we take a more careful analysis, the formula (\ref{for-quality4}) can be rewritten as follows:
\begin{eqnarray}
	&&||a_i-\mathcal{T}(b_j)||^2 \nonumber\\
	&\leq& \big(||a_i-a'_i||+||a'_i-\mathcal{T}(b'_j)||+||\mathcal{T}(b'_j)-\mathcal{T}(b_j)||\big)^2\nonumber\\
	&=&||a'_i-\mathcal{T}(b'_j)||^2+2(||a_i-a'_i||+||\mathcal{T}(b'_j)-\mathcal{T}(b_j)||)\times||a'_i-\mathcal{T}(b'_j)||\nonumber\\
	&&+(||a_i-a'_i||+||\mathcal{T}(b'_j)-\mathcal{T}(b_j)||)^2\nonumber\\
	&\leq&||a'_i-\mathcal{T}(b'_j)||^2+4\epsilon\Delta\times||a'_i-\mathcal{T}(b'_j)||+2||a_i-a'_i||^2+2||\mathcal{T}(b'_j)-\mathcal{T}(b_j)||^2\nonumber\\
	&\leq&||a'_i-\mathcal{T}(b'_j)||^2+2\epsilon\big(\Delta^2+||a'_i-\mathcal{T}(b'_j)||^2\big) +2||a_i-a'_i||^2+2||\mathcal{T}(b'_j)-\mathcal{T}(b_j)||^2 \nonumber\\
	&=&(1+2\epsilon)||a'_i-\mathcal{T}(b'_j)||^2+2\epsilon\Delta^2+\boxed{2||a_i-a'_i||^2+2||\mathcal{T}(b'_j)-\mathcal{T}(b_j)||^2}.\label{for-quality4-mean}
\end{eqnarray}
Comparing (\ref{for-quality4-mean}) with (\ref{for-quality4}), we can see that the item ``$4\epsilon^2\Delta^2$'' in the right-hand side bound is replaced by ``$2||a_i-a'_i||^2+2||\mathcal{T}(b'_j)-\mathcal{T}(b_j)||^2$''. Note that $a'_i$ and $\mathcal{T}(b'_j)$ are actually the centers of the balls that respectively cover $a_i$ and $\mathcal{T}(b_j)$. So if we replace the ball centers by the means, the accumulated additive errors in (\ref{for-quality6}) and (\ref{for-quality7}) can be further reduced  (though this bound remains the same in the worst case, {\em e.g.,} all the points of the cluster locate uniformly on the sphere of the ball, and then $||a_i-a'_i||$ and $||\mathcal{T}(b'_j)-\mathcal{T}(b_j)||$ are always equal to $\epsilon \Delta$).

Intuitively, we use $k$-center clustering first to bound the radius of each cluster, and then compute the mean of each cluster to refine the result. We name this method  as \textsc{KCenter+}. In Section~\ref{sec-exp}, our experiments also verified the fact that \textsc{KCenter+} usually can achieve better performance than \textsc{KCenter}.

\subsection{Extension \Rmnum{3}: Robust Alignment with Fractional Wasserstein Distance}

We further consider the alignment problem with the fractional Wasserstein distance as Definition~\ref{def-fracemd}. The first question is how to compute this new distance. Recall that the vanilla Wasserstein distance is equivalent to computing the minimum cost maximum flow on the  bipartite graph of $A$ and $B$. When we consider the fractional Wasserstein distance that allows $1-\lambda$ outliers, we can modify the bipartite graph by adding two ``dummy'' points (see Figure~\ref{fig:frac_emd} for an illustration). Then, we can directly apply any off-the-shelf  Wasserstein distance algorithms ({\em e.g.,} the network simplex algorithm~\cite{ahuja1993network} or the Sinkhorn distance algorithm~\cite{DBLP:conf/nips/Cuturi13}) to compute the fractional Wasserstein distance $\mathcal{W}^2_2(A, B,\lambda)$.

\begin{figure}[htbp]
	\begin{center}	
		\includegraphics[width=0.7\columnwidth]{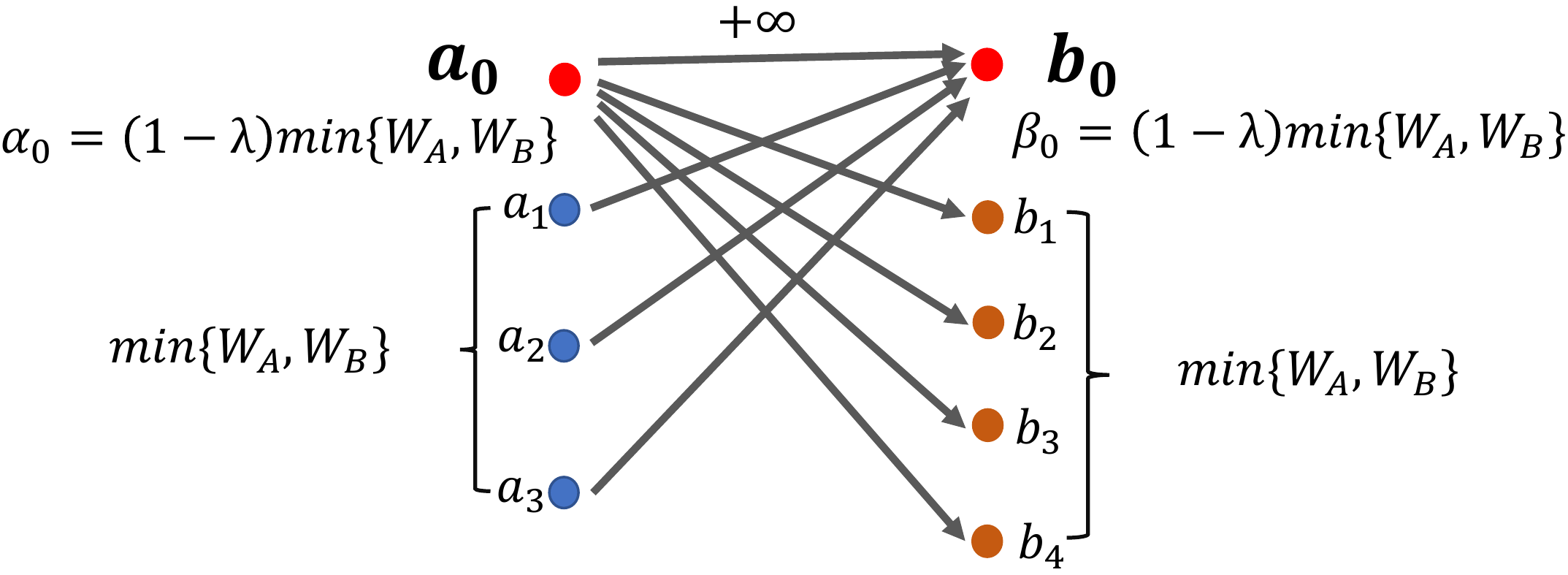} 
	\end{center}
	\caption{We add two dummy points $a_0$ and $b_0$, where their weights are both $(1-\lambda)\cdot\min\{W_A, W_B\}$. The point $a_0$ connects all the $b_j$s with the edge cost $0$, and so does $b_0$; the cost of the edge connecting $a_0$ and $b_0$ is $+\infty$. The total flow from the  left to the right is $\min\{W_A, W_B\}+(1-\lambda)\cdot\min\{W_A, W_B\}$. Intuitively, the dummy point $a_0$ absorbs the $(1-\lambda)\cdot\min\{W_A, W_B\}$ outlier flows from $B$, and the dummy point $b_0$ absorbs the $(1-\lambda)\cdot\min\{W_A, W_B\}$ outlier flows from $A$.} 
	\label{fig:frac_emd}  
\end{figure}

%
%
%
%

\begin{algorithm}
	\caption{Fractional Wasserstein distance}
	\label{alg-2}
	\begin{algorithmic}[1]
		\REQUIRE An instance $(A, B)$ of the Wasserstein distance problem in Definition~\ref{def-fracemd} with the fraction $\lambda ~(0 < \lambda \leq 1)$.
		\ENSURE The fractional Wasserstein distance $\mathcal{W}^2_2(A, B,\lambda)$ and the corresponding flows from $A$ to $B$.
		\STATE Set $w_0=(1-\lambda) \min {\{W_A,W_B \}}$
		\STATE Add the dummy points $a_0$ and $b_0$ to $A$ and $B$, respectively; their weights are both $w_0$. Denote  the new point sets by $\tilde{A}$ and $\tilde{B}$, respectively. 
		\STATE Build the bipartite graph of $\tilde{A}$ and $\tilde{B}$, and set the ground distance (edge cost) as Figure~\ref{fig:frac_emd}. 
		
		\STATE Apply the existing Wasserstein distance algorithm e.g., \cite{ahuja1993network} or~\cite{DBLP:conf/nips/Cuturi13}, to compute the optimal transport flow matrix and the Wasserstein distance.
	\end{algorithmic}
\end{algorithm}
%
%

\begin{lemma}
	\label{theorem_alg2_frac}
	The Algorithm \ref{alg-2} returns the optimal fractional Wasserstein distance flow of $(A,B)$ with the total flow being equal to  $\lambda \min {\{W_A,W_B \}}$. The time complexity is as same as the vanilla Wasserstein distance algorithm~\cite{ahuja1993network} or~\cite{DBLP:conf/nips/Cuturi13}
\end{lemma}

\begin{proof}
	We denote the obtained flows as $F=\{f_{ij}\mid 0\leq i\leq n_1, 0\leq j\leq n_2\}$. First, we need to show that $f_{00}=0$. Otherwise, we can arbitrarily pick a non-zero flow, say $f_{i_0 j_0}$, from $F\setminus\{f_{00}\}$, and perform the following modification: let $\delta=\min\{f_{00}, f_{i_0 j_0}\}$, and update
	\begin{eqnarray}
		f_{00}&\longrightarrow& f_{00}-\delta;\nonumber\\
		f_{i_0 j_0}&\longrightarrow& f_{i_0 j_0}-\delta;\nonumber\\
		f_{0 j_0}&\longrightarrow& f_{0 j_0}+\delta;\nonumber\\
		f_{i_0 0}&\longrightarrow& f_{i_0 0}+\delta. 
	\end{eqnarray}
	From the edge costs in Figure~\ref{fig:frac_emd} (the edge cost connecting $a_0$ and $b_0$ is $\infty$), we know that the total transportation cost is reduced after the above modification. This is in contradiction with the optimality of $F$. Also, since $\alpha_0=\beta_0=(1-\lambda)\cdot\min\{W_A, W_B\}$, together with $f_{00}=0$, we have the total flows from $A$ to $B$  that is  
	\begin{eqnarray}
		\sum^{n_1}_{i=1}\sum^{n_2}_{j=1}f_{ij}= \lambda \cdot\min\{W_A, W_B\}.
	\end{eqnarray}
	
	Now, we prove the optimality of the flows with respect to Definition~\ref{def-fracemd}. Suppose the flows $F'=\{f'_{ij}\mid 1\leq i\leq n_1, 1\leq j\leq n_2\}$ yield the optimal fractional Wasserstein distance, {\em i.e.,}
	\begin{eqnarray}
		\sum^{n_1}_{i=1}\sum^{n_2}_{j=1}f'_{ij} ||a_i-b_j||^2<\sum^{n_1}_{i=1}\sum^{n_2}_{j=1}f_{ij} ||a_i-b_j||^2. \label{for-fwd-1}
	\end{eqnarray}
	We can easily augment $F'$ to be a solution for $(\tilde{A}, \tilde{B})$ with the dummy points $a_0$ and $b_0$: $F'\longrightarrow F'\cup \{f'_{00}, f'_{0j}, f'_{i0}\mid 1\leq i\leq n_1, 1\leq j\leq n_2\}$, where 
	\begin{eqnarray}
		f'_{00}&=&0;\nonumber\\
		f'_{0j}&=&\frac{(1-\lambda)\cdot\min\{W_A, W_B\}}{W_B-\lambda\cdot\min\{W_A, W_B\}} (\beta_j-\sum^{n_1}_{i=1}f'_{ij});\nonumber\\
		f'_{i0}&=&\frac{(1-\lambda)\cdot\min\{W_A, W_B\}}{W_A-\lambda\cdot\min\{W_A, W_B\}} (\alpha_i-\sum^{n_2}_{j=1}f'_{ij}).
	\end{eqnarray}
	It is easy to verify that the augmented $F'$ is a feasible flow set from $\tilde{A}$ to $\tilde{B}$ with the total flows being equal to $\min\{W_A, W_B\}+(1-\lambda)\cdot\min\{W_A, W_B\}$. 
	From the edge costs in Figure~\ref{fig:frac_emd}, together with (\ref{for-fwd-1}), we know that the augmented $F'$ yields a lower transportation cost between $\tilde{A}$ and $\tilde{B}$, which is in contradiction with the optimality of $F$. Therefore,  Algorithm \ref{alg-2} returns the optimal fractional Wasserstein distance flow of $(A,B)$. The time complexity is as same as the vanilla Wasserstein distance algorithm~\cite{ahuja1993network} or~\cite{DBLP:conf/nips/Cuturi13}, since we only add two more points to the input. 
\end{proof}

Using Algorithm~\ref{alg-2}, we can compute the partial alignment between $A$ and $B$. We still use our proposed Algorithm~\ref{alg-1}, where the only difference is that we need to apply Algorithm~\ref{alg-2} to compute the fractional Wasserstein distance between $S_A$ and $\mathcal{T}(S_B)$ ({\em resp.,} $A$ and $\mathcal{T}(B)$). Similar with Theorem~\ref{the-quality}, we also have the following Theorem~\ref{the-quality-2} for the partial alignment with any given parameter $\lambda\in [0,1]$. The proof (which is very similar with that of Theorem~\ref{the-quality}) is placed in the appendix.

\begin{theorem}
	\label{the-quality-2}
	Suppose $\epsilon>0$ is a small parameter in Corollary~\ref{rem-k}. Let $\lambda\in [0,1]$. Given any $c\geq 1$, let $\tilde{\mathcal{T}}$ be a rigid transformation yielding $c$-approximation for minimizing $\mathcal{W}^2_2\big(S_A, \mathcal{T}(S_B), \lambda\big)$ in Definition~\ref{def-align}. Then,  
	\begin{eqnarray}
		&&\mathcal{W}^2_2\big(A, \tilde{\mathcal{T}}(B), \lambda\big)\nonumber\\
		&\leq& c(1+2\epsilon)^2\cdot\min_\mathcal{T}\mathcal{W}^2_2\big(A, \mathcal{T}(B),\lambda\big)+2\epsilon(c+1+2c\epsilon)(1+2\epsilon)\Delta^2\nonumber\\
		&=& c\big(1+O(\epsilon)\big)\cdot\min_\mathcal{T}\mathcal{W}^2_2\big(A, \mathcal{T}(B),\lambda\big)+2\epsilon \big(1+O(\epsilon)\big)(c+1)\Delta^2. \label{for-quality-2-1}
	\end{eqnarray}
\end{theorem}

\section{The Time Complexity}
\label{sec-time}
We analyze the time complexity of Algorithm~\ref{alg-1} and consider Step~1-3 separately. 
To simplify our description, we use $n$ to denote $\max\{n_1, n_2\}$. In Step 2, we suppose that the iterative approach~\cite{cohen1999earth} takes $h\geq1$ rounds. 


\textbf{Step 1.} A straightforward implementation of the Gonzalez's  algorithm is selecting the $k$ cluster centers iteratively with the running time  $O(knd)$. Several faster implementations for the high dimensional case with low doubling dimension have been studied before; their idea is to maintain some data structures to reduce the amortized complexity of each iteration. We refer the reader to~\cite{har2006fast} for more details. Also note that if we use \textsc{KCenter+}, it only yields an extra $O(nd)$ time since we just need  to scan the whole data in one-pass for computing the means.


\textbf{Step 2.} Since we run the algorithm~\cite{cohen1999earth} on the smaller instance $(S_A, S_B)$ instead of $(A, B)$, we know that the complexity of Step~2 is $O\Big(h \big(\Gamma(k,d)+k^2 d+kd^2+d^3\big)\Big)$ by Proposition~\ref{pro-cg}.

\textbf{Step 3.} We need to compute the transformed $\mathcal{T}(B)$ first and then the Wasserstein distance $\mathcal{W}^2_2(A, \mathcal{T}(B))$. Note that the transformation $\mathcal{T}$ is not off-the-shelf, because it is the combination of a sequence of rigid transformations from the iterative approach~\cite{cohen1999earth} in Step 2. Since it takes $h$ rounds, $\mathcal{T}$ should be the multiplication of $h$ rigid transformations. We use $(\mathcal{R}_l, \overrightarrow{v}_l)$ to denote the orthogonal matrix and translation vector obtained in the $l$-th round for $1\leq l\leq h$. We can update $B$ round by round: starting from $l=1$, update $B$ to be $\mathcal{R}_l B+\overrightarrow{v}_l$ in each round; the whole time complexity will be $O(h n d^2)$. In fact, we have a more efficient way by computing $\mathcal{T}$ first before transforming $B$.

\begin{lemma}
	\label{lem-tran}
	Let $(\mathcal{R}, \overrightarrow{v})$ be the orthogonal matrix and translation vector of $\mathcal{T}$. Then
	\begin{eqnarray}
		\mathcal{R}&=&\Pi^\lambda_{l=1}\mathcal{R}_l, \nonumber\\
		\overrightarrow{v}&=&(\Pi^\lambda_{l=2}\mathcal{R}_l)\overrightarrow{v}_1+(\Pi^\lambda_{l=3}\mathcal{R}_l)\overrightarrow{v}_2+\cdots +\mathcal{R}_\lambda\overrightarrow{v}_{\lambda-1}+\overrightarrow{v}_\lambda,\label{for-tran}
	\end{eqnarray}
	and $\mathcal{T}(B)$ can be obtained in $O(h d^3+n d^2)$ time. 
\end{lemma} 
\begin{proof}
	The equations (\ref{for-tran}) can be easily verified by simple calculations and we only need to focus on the time complexity.  We can recursively compute the multiplications $\Pi^\lambda_{l=i}\mathcal{R}_l$ for $i=h, h-1, \cdots, 1$. Consequently, the orthogonal matrix $\mathcal{R}$ and translation vector $\overrightarrow{v}$ can be obtained in $O(h d^3)$ time. In addition, the complexity for computing $\mathcal{T}(B)=\mathcal{R}B+\overrightarrow{v}$ is $O(n d^2)$.
\end{proof}

Lemma~\ref{lem-tran} provides a complexity significantly lower than the previous $O(h n d^2)$ (usually $n$ is much larger than $d$ in practice). 
After obtaining $\mathcal{T}(B)$, we can compute $\mathcal{W}^2_2(A, \mathcal{T}(B))$ in $\Gamma(n,d)$ time.
\textbf{Note} that the complexity $\Gamma(n,d)$ usually is $\Omega(n^2 d)$, which dominates the complexity of Step 1 and the second term $nd^2$ in the complexity of Lemma~\ref{lem-tran}. Overall, we have the following theorem for the total runtime. 

\begin{theorem}
	\label{the-time}
	Suppose $n=\max\{n_1, n_2\}\geq d$ and the algorithm of ~\cite{cohen1999earth} takes $h\geq1$ rounds. The running time of Algorithm~\ref{alg-1} is 
	$O\Big(h\big(\Gamma(k, d)+k^2d+kd^2+d^3\big)\Big)+\Gamma(n, d)$, where $k= (\frac{2}{\epsilon})^\rho$.
\end{theorem}

If we run the same number of rounds on the original instance $(A, B)$ by the approach~\cite{cohen1999earth}, the total running time will be $O\Big(h\big(\Gamma(n, d)+n^2d\big)\Big)$ by Proposition~\ref{pro-cg}. When $k\ll n$, Algorithm~\ref{alg-1} achieves a significant reduction on the running time.

\section{Experiments}
\label{sec-exp}
We consider three applications in our experiments: PPI network alignment, unsupervised bilingual lexicon induction, and domain adaption. All the experimental results were obtained on  a server equipped with 2.4GHz Intel CPUs and 256GB main memory; the algorithms were implemented in Matlab. For each instance, we run $20$ trials and report the average results. Our code is
publicly available at \href{https://github.com/lwjie595/RobustGeometricAlignment}{https://github.com/lwjie595/RobustGeometricAlignment}.

\subsection{Datasets}

\textbf{(1)} For  PPI network alignment, we use the popular benchmark dataset NAPAbench~\cite{NAPAbench} of PPI networks. It consists of several different families of PPI networks generated from the real proteins. The {\em duplication mutation complementation (DMC)}~\cite{vazquez2003modeling} and {\em duplication with random mutation (DMR) } model~\cite{sole2002model} are two different node-duplication network growth models, while {\em crystal growth (CG)} model~\cite{kim2008age} generates the  networks by simulating the physics of growing protein crystals. 
We use $3$ pairs of PPI networks from these  $3$ models, where each network is a graph containing $3000$ to $10000$ nodes. In the   preprocessing step, we apply the popular {\em node2vec} technique~\cite{grover2016node2vec} to represent each network by a group of vectors in $\mathbb{R}^{50}$; following the approach of~\cite{DBLP:conf/aaai/LiuDC017}, we assign a unit weight to each vector. 

\textbf{(2)} For unsupervised bilingual lexicon induction, we have 5 pairs of languages: {\em Chinese-English (zh-en), Spanish-English (es-en), Italian-English (it-en), Japanese-Chinese (ja-zh), and Turkish-English (tr-en)}. Given the datasets from~\cite{DBLP:conf/emnlp/ZhangLLS17}, each language has a vocabulary list containing $3000$ to $13000$ words; we also follow their preprocessing method that represents all the words by the vectors in $\mathbb{R}^{50}$ through the embedding technique~\cite{mikolov2013exploiting}. Each  vocabulary list is represented by a distribution in the space where each vector has the weight equal to the corresponding frequency in the language.

\textbf{(3)} For domain adaption, we consider  the Caltech-Office dataset from~\cite{gong2012geodesic} that was widely studied before.  The Caltech-Office dataset contains $10$ categories of images across  $4$ different domains: \textit{Amazon} (A), \textit{Caltech10} (C), \textit{DSLR} (D), and \textit{Webcam} (W). Each dataset contains $3000$ to $7000$ data items in $\mathbb{R}^{800}$.  We use ``$\rightarrow$'' to denote the transform between two domains, {\em e.g.,} ``\textit{C$\rightarrow$A}'' represents the transform from \textit{Caltech10} to  \textit{Amazon}. Similar with the previous research on domain adaptation~\cite{courty2016optimal,perrot2016mapping,gong2012geodesic,jin2021two}, we use the labeled data in the source domain as the training data and use the k-Nearest Neighbor (k-NN) method as the classifier to predict the labels in the target domain.


\subsection{Algorithms for Testing}
\label{sec-testalg}

We use the   alignment algorithm~\cite{cohen1999earth} (due to its simplicity and practicality) and consider the following data compression methods for comparison. 
\begin{itemize}
	\item \textbf{\textsc{Original}:} directly run the   alignment algorithm on the original datasets without compression. 
	\item \textbf{\textsc{KCenter}:} our proposed method Algorithm~\ref{alg-1}, {\em i.e.,} run the Gonzalez's $k$-center algorithm to compress the input data.  
	\item \textbf{\textsc{KCenter+}:} replace each ball center by the mean point of the cluster (see Section~\ref{sec-ext2}).
	
	\item \textbf{\textsc{KMeans}:} replace the  Gonzalez's $k$-center algorithm by the $k$-means clustering algorithm~\cite{lloyd1982least} for compression (the initialization is implemented by the $k$-means++ seeding~\cite{arthur2007k}). 
	
	\item \textbf{\textsc{Random}:} sample $k$ points uniformly at random from $A$ and $B$ separately, and each sampled point of $A$ ({\em resp.,} $B$) has the weight $\frac{W_A}{k}$ ({\em resp.,} $\frac{W_B}{k}$).
	\item \textbf{\textsc{Random+}:} sample $k$ points uniformly at random from $A$ and $B$ separately; we assign each point of $A$ ({\em resp.,} $B$) to its nearest sampled point; each sampled point of $A$ ({\em resp.,} $B$) has the weight  equal to the total weights of the points that assigned to it. 
	
	\item \textbf{\textsc{StochasticOpt}:} the stochastic gradient descent algorithm that uses the batches of subsamples from the two input patterns  to compute the optimal transportation and  the orthogonal matrices for rigid transformations~\cite{grave2019unsupervised}.

\end{itemize}

%

\subsection{Results} 

For each instance, we vary the compression rate $\gamma=\frac{k}{(n_1+n_2)/2}$ (recall that $n_1$ and $n_2$ are respectively the numbers of points of $A$ and $B$). We also consider the robust alignment with fractional Wasserstein distance. We set the value $\lambda\in [0.9,1]$. For PPI network alignment and unsupervised bilingual lexicon induction, we evaluate the performance based on  the obtained  Wasserstein distance of two patterns and the normalized running time over the time of \textsc{Original} ({e.g.,} if the normalize running time is $0.1$, it means the algorithm saves $90\%$ runtime compared with \textsc{Original}). The runtime of each method includes  the time for data compression,  the time for alignment on the compressed data, and the time for computing the final Wasserstein flow. 

	\textbf{PPI network alignment.}
	To see the trends of the algorithms, we first show the  results  of PPI network alignment on the CG dataset  in Figure~\ref{fig:CG1}.
	The compression rate $\gamma $ ranges from $0.02$ to $0.1$.
	Our \textsc{KCenter+} compression method can achieve the  performance close to \textsc{Original}, but it takes significant lower runtimes comparing with other baselines (\textsc{Random} is always the fastest one, since it is just simple uniform sampling; but it always obtained the largest Wasserstein distance). 
	The similar results  of the other two PPI datasets are shown in Figure~\ref{fig:DMC1} and Figure~\ref{fig:DMR1}, respectively. 
	The algorithms also achieve the similar performances for the experiments on fractional Wasserstein distance, where the results on the three PPI datasets  are shown in Figure~\ref{fig:CG2}, Figure~\ref{fig:DMC2}, and Figure~\ref{fig:DMR2}, respectively. 
	We vary the fraction parameter $\lambda$ from $0.9$ to $1$, and fix the compression rate $\gamma$   to be $0.1$. The experimental results suggest that our proposed methods \textsc{KCenter} and \textsc{KCenter+} also work well for fractional Wasserstein distance. 
	To see the influence from the dimension, we also conduct the following experiment on the CG dataset. Previously, we set the dimension to be $50$ through  {\em node2vec}~\cite{grover2016node2vec}; now we vary the dimension from $50$ to $300$ by tuning the {\em node2vec} algorithm. We show the results of Wasserstein alignment and fraction Wasserstein alignment in Figure~\ref{fig:CG_dim_size} and Figure~\ref{fig:CG_dim_noise}, respectively. The compression rate $\gamma $ is fixed to be $0.1$, and the fraction value  $\lambda$ is fixed to be $0.9$ for the experiment on fractional Wasserstein distance. We can see that for different dimensions, \textsc{KCenter+} can achieve the performance close to \textsc{Original} but  with much lower running time.

\textbf{Unsupervised bilingual lexicon induction.}
The parameters $\gamma$ and $\lambda$ for unsupervised bilingual lexicon induction are set to be as same as the experiments for PPI network alignment. We show the results of Wasserstein alignment and fractional Wasserstein alignment on \textsc{es-en} in Figure~\ref{fig:esen1}  and Figure~\ref{fig:esen2} respectively. Similar with the experiments on PPI networks, 
our \textsc{KCenter+} compression method   can achieve the  performance close to \textsc{Original}.
For the other three unsupervised bilingual lexicon induction pairs,  the similar experimental results are shown in Figure~\ref{fig:iten1}, Figure~\ref{fig:jazh1}, Figure~\ref{fig:tren1}, and Figure~\ref{fig:zhen1}, respectively; the experimental results on fractional Wasserstein distance are shown in Figure~\ref{fig:iten2}, Figure~\ref{fig:jazh2}, Figure~\ref{fig:tren2}, and Figure~\ref{fig:zhen2}, respectively.

\textbf{Domain adaptation.}
The results of Wasserstein distance, normalized time and accuracy for different compression rates are shown in Tables \ref{tab:DA_nonoiseemd}, \ref{tab:DA_nonoisetime}, \ref{tab:DA_nonoiseacc}, respectively. We illustrate the best results with bold fonts in the tables. We can see that
\textsc{KCenter+} performs better than the other five baselines for most cases. 
We also illustrate the results on fractional Wasserstein distance for domain adaption in Tables \ref{tab:DA_noiseemd}, \ref{tab:DA_noisetime}, and \ref{tab:DA_noiseacc}, where their performances are similar to the results in Tables \ref{tab:DA_nonoiseemd}-\ref{tab:DA_nonoiseacc}.  
We vary the value of $\lambda$ from $0.9$ to $1.0$, and the compression rate $\gamma $ is fixed to be $0.1$.

\textbf{Experimental conclusion.} 
Overall, our proposed \textsc{KCenter+} usually outperforms the other baselines and 
achieves the results close to  \textsc{Original}. Although \textsc{KMeans} can also  archive similar good performance as \textsc{KCenter+}, it takes significant higher runtimes than \textsc{KCenter+}.  Also, \textsc{KCenter+} often has  a higher classification accuracy for domain adaption than the other baselines (except for \textsc{Original}). 

\section{Conclusion}
In this paper, we propose a novel framework for compressing point sets in high dimensions, so as to approximately preserve the quality for the Wasserstein alignment. This work is motivated by several emerging applications in the fields of machine learning and bioinformatics. Our method utilizes the property of low doubling dimension, and yields a significant speedup for the computation of alignment. In the experiments, we show that the proposed compression approach can efficiently reduce the running time to a great extent. In the future, it is interesting to consider the alignment problems for other distances rather than the Wasserstein distance. It is also important to study several other issues of the alignment problems, such as the robustness and parallel implementation.

\newcounter{sd2}

\begin{figure}[htbp]
	\begin{center}
		\centerline
		{\includegraphics[width=0.44\columnwidth]{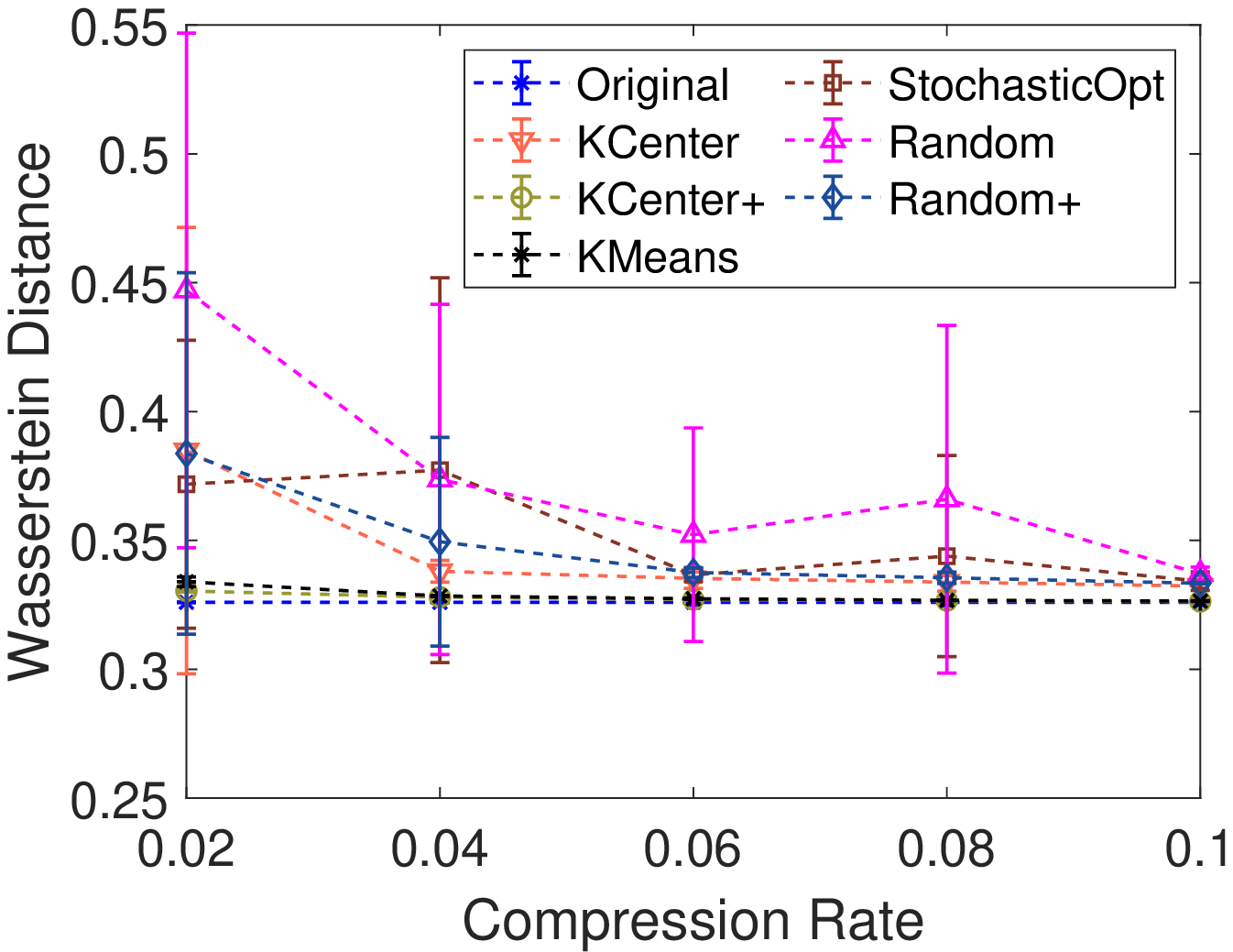} 
			\hspace{0.1in}
			\includegraphics[width=0.44\columnwidth]{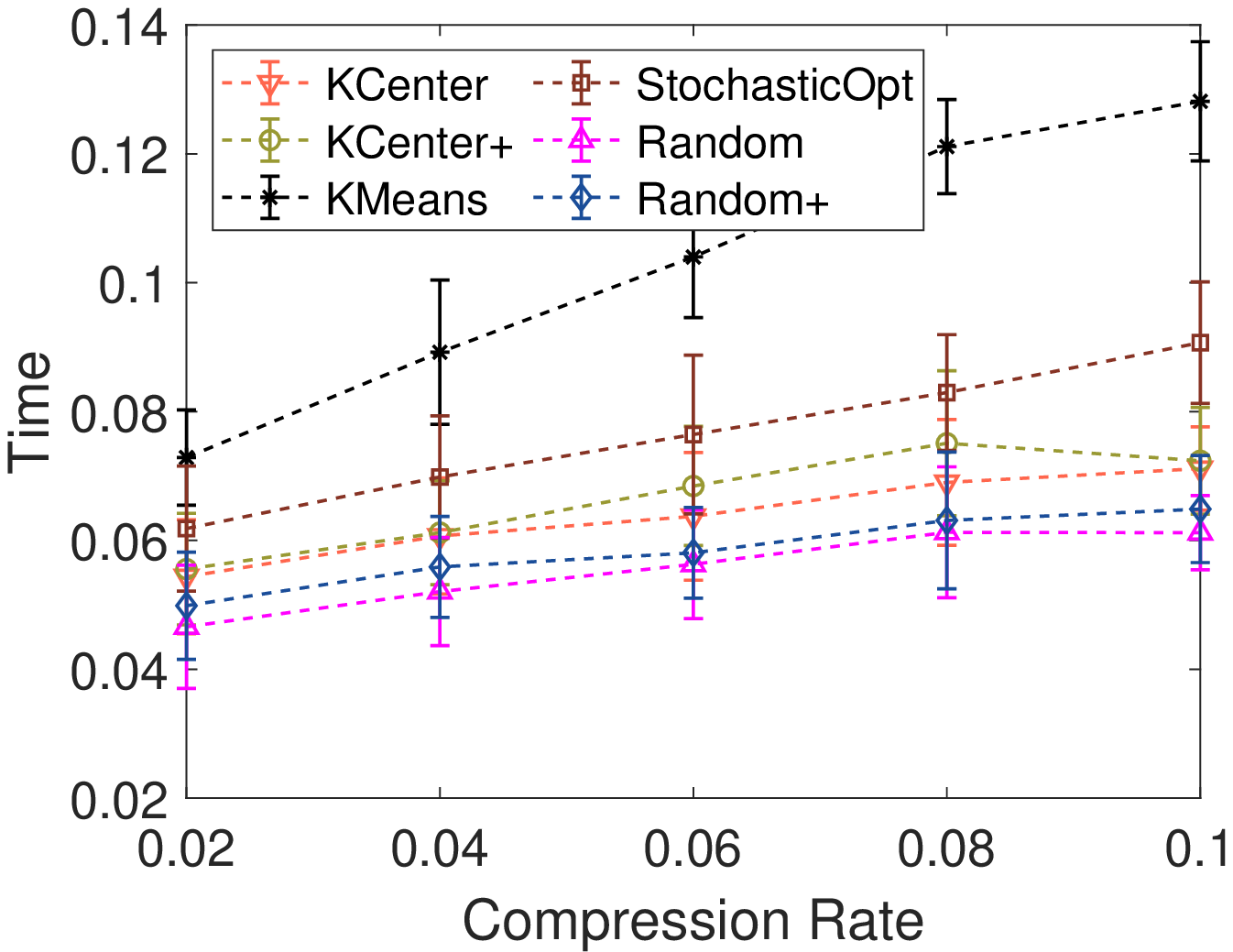}}
	\end{center}
	\vspace{-0.25in}
	\caption{The Wasserstein  distance and normalized running time on \textsc{CG} for PPI network alignment. } 
	\label{fig:CG1}  
	
\end{figure}
\begin{figure}[htbp]
	\begin{center}
		\centerline
		{\includegraphics[width=0.44\columnwidth]{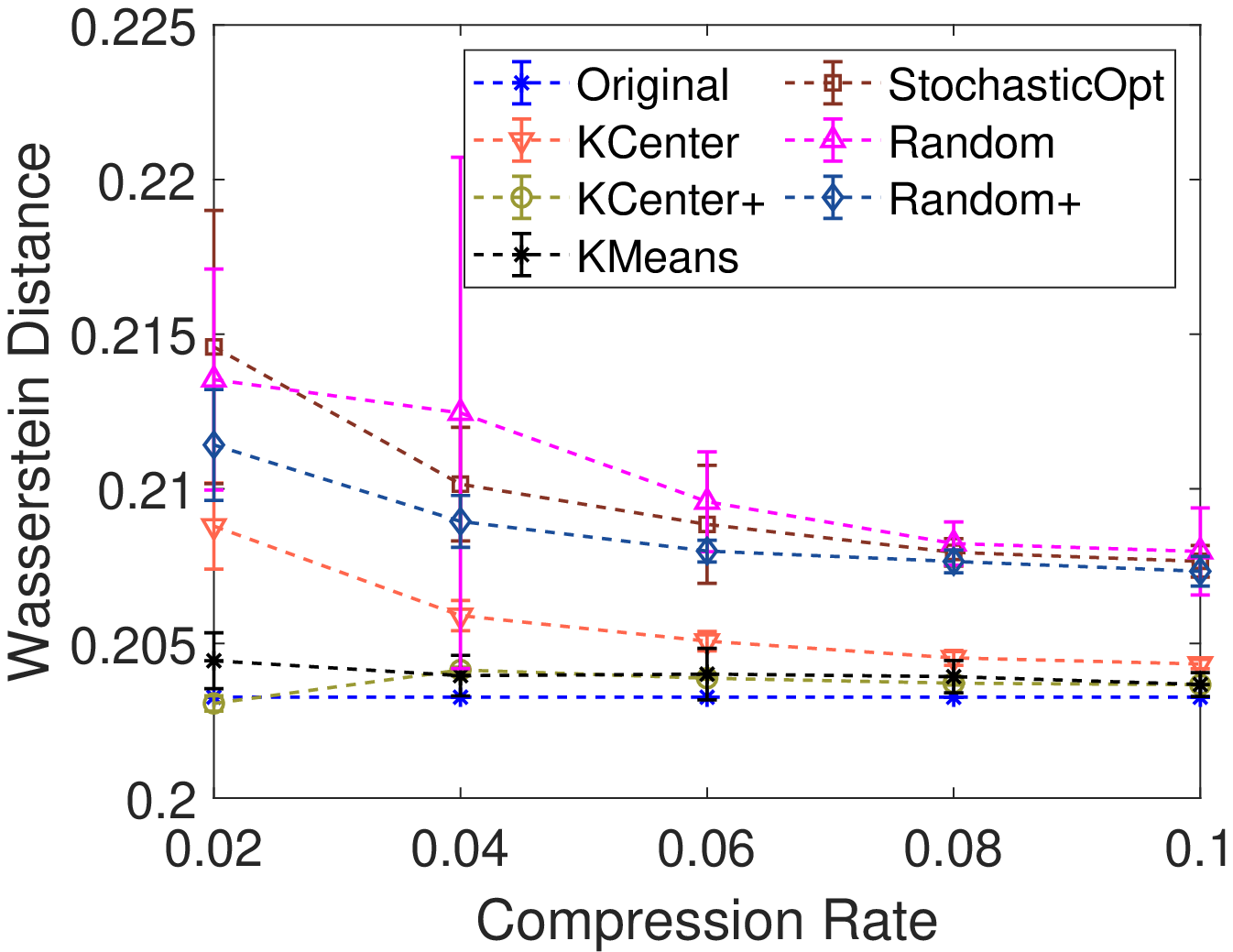} 
			\hspace{0.1in}
			\includegraphics[width=0.44\columnwidth]{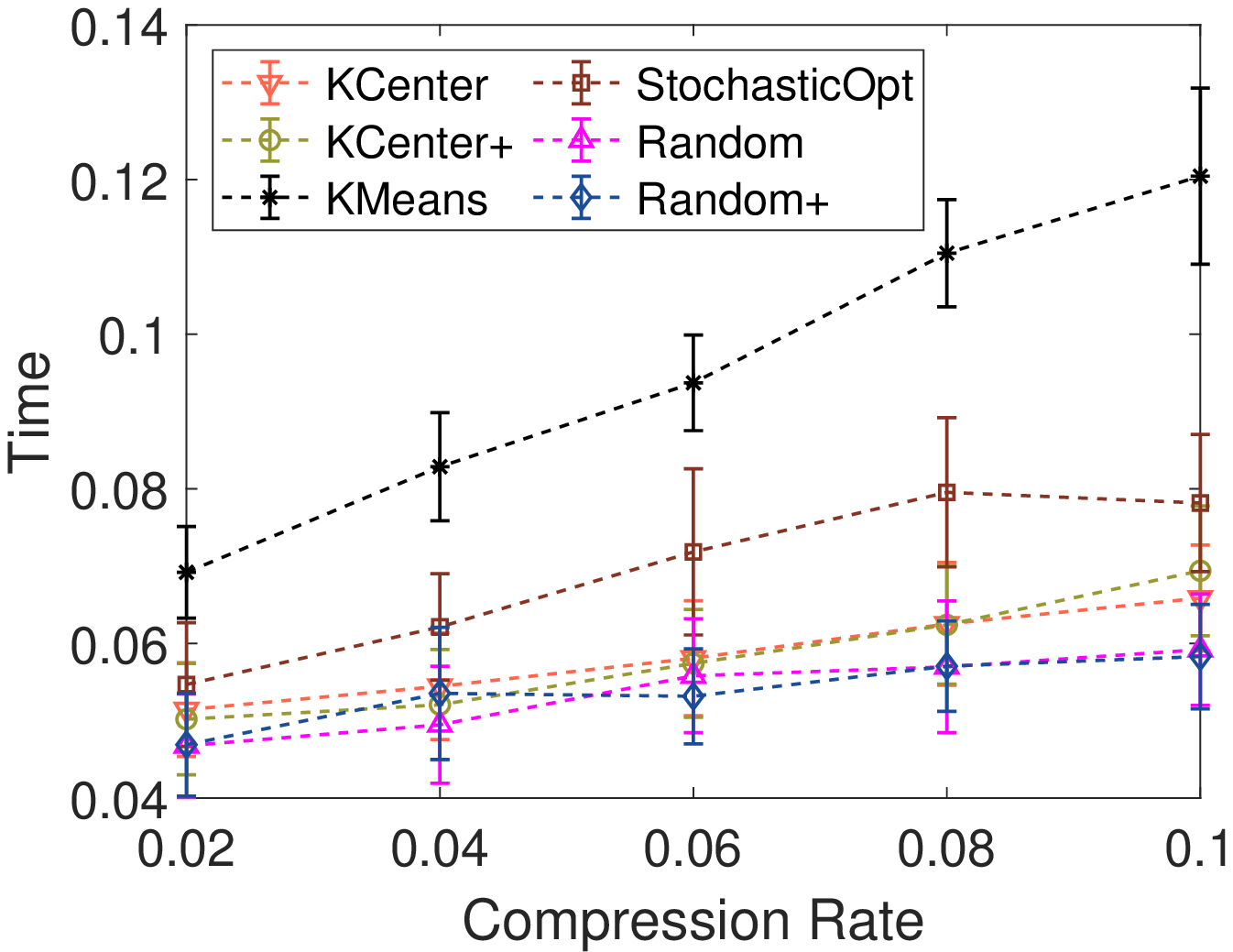}}
	\end{center}
	\vspace{-0.25in}
	\caption{The Wasserstein  distance and normalized running time on \textsc{DMC} for PPI network alignment. } 
	\label{fig:DMC1}  
\end{figure}

\begin{figure}[htbp]
	\begin{center}
		\centerline
		{\includegraphics[width=0.44\columnwidth]{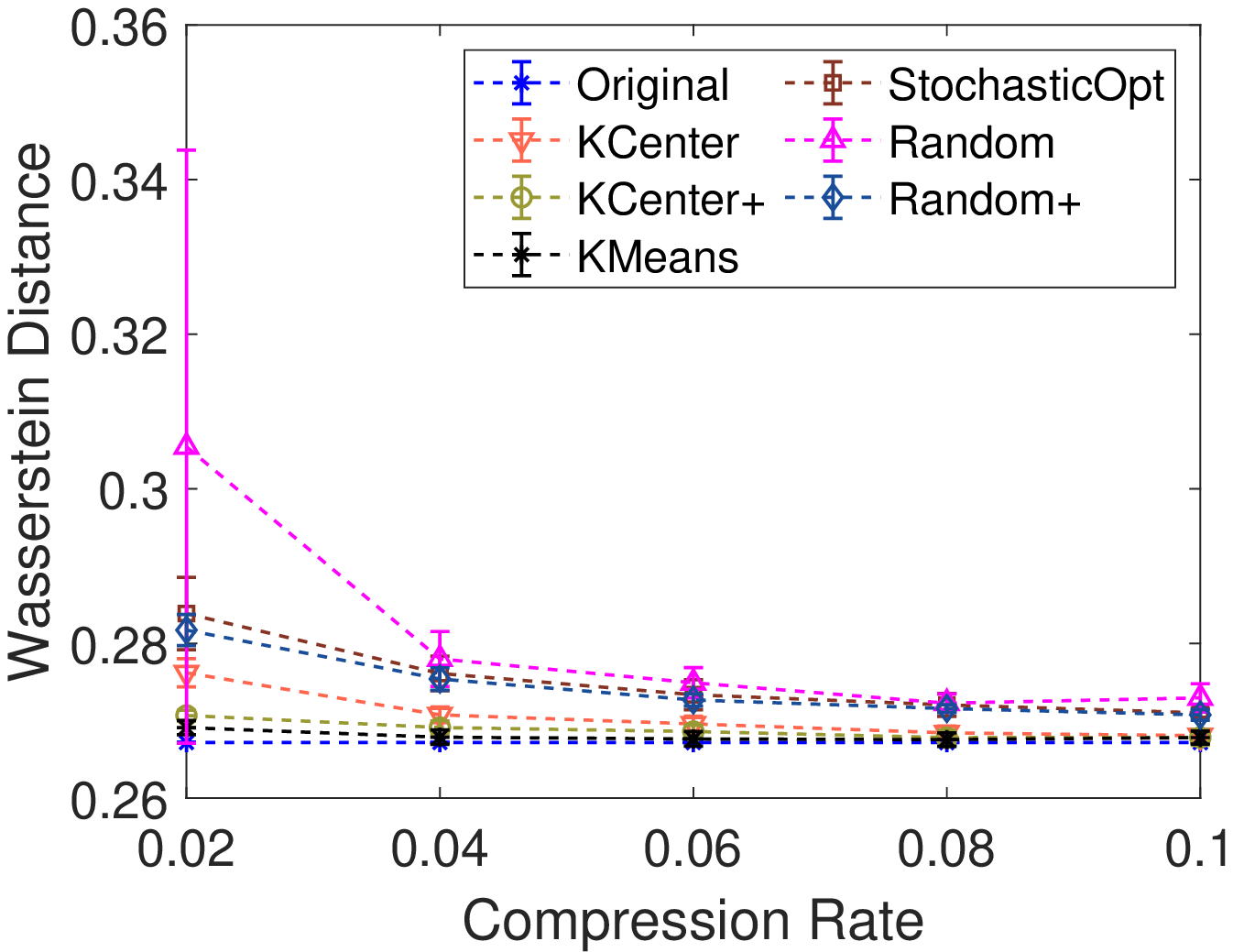} 
			\hspace{0.1in}
			\includegraphics[width=0.44\columnwidth]{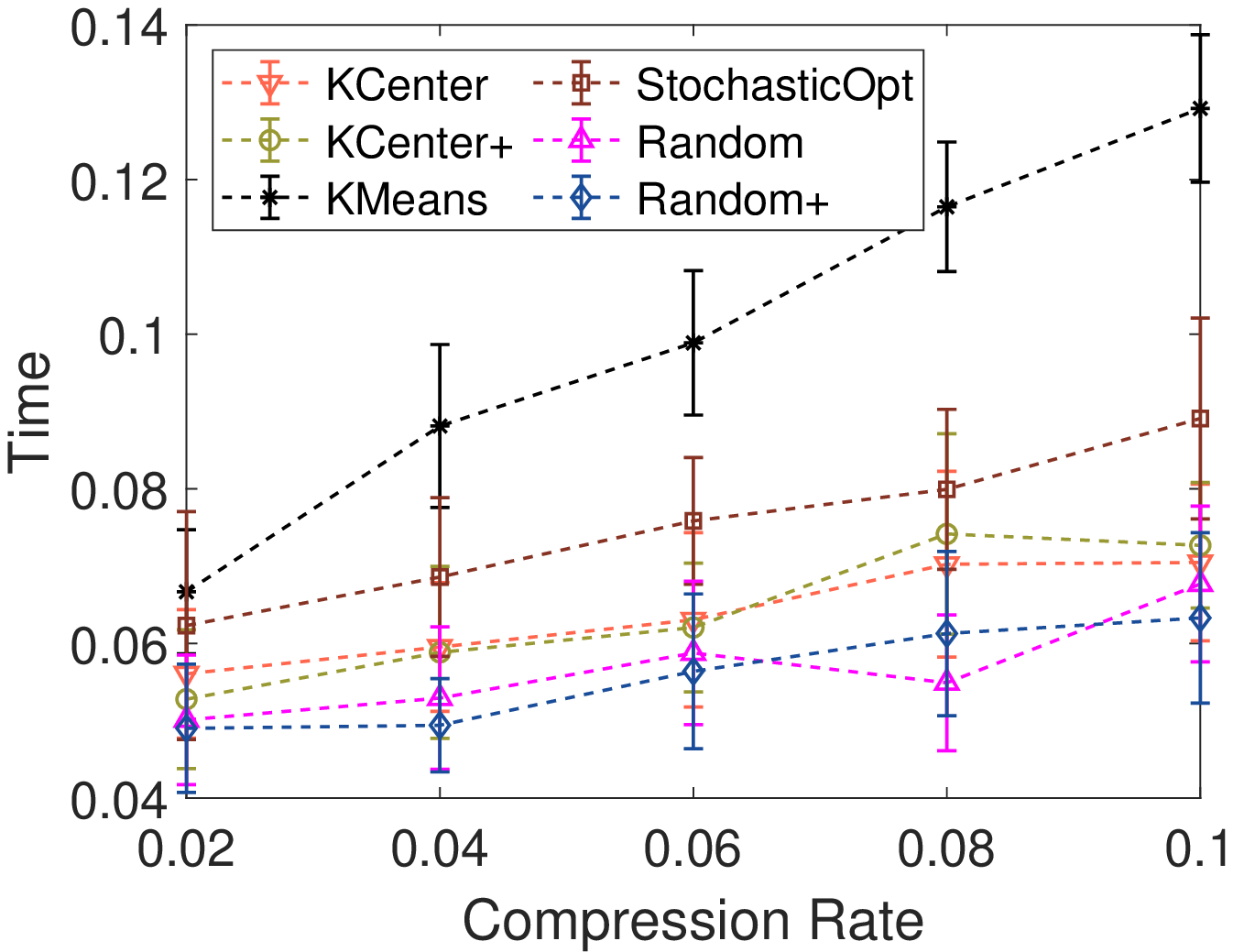}}
	\end{center}
	\vspace{-0.25in}
	\caption{The Wasserstein  distance and normalized running time on \textsc{DMR} for PPI network alignment. } 
	\label{fig:DMR1}  
\end{figure}

\begin{figure}[htbp]
	\begin{center}
		\centerline
		{\includegraphics[width=0.44\columnwidth]{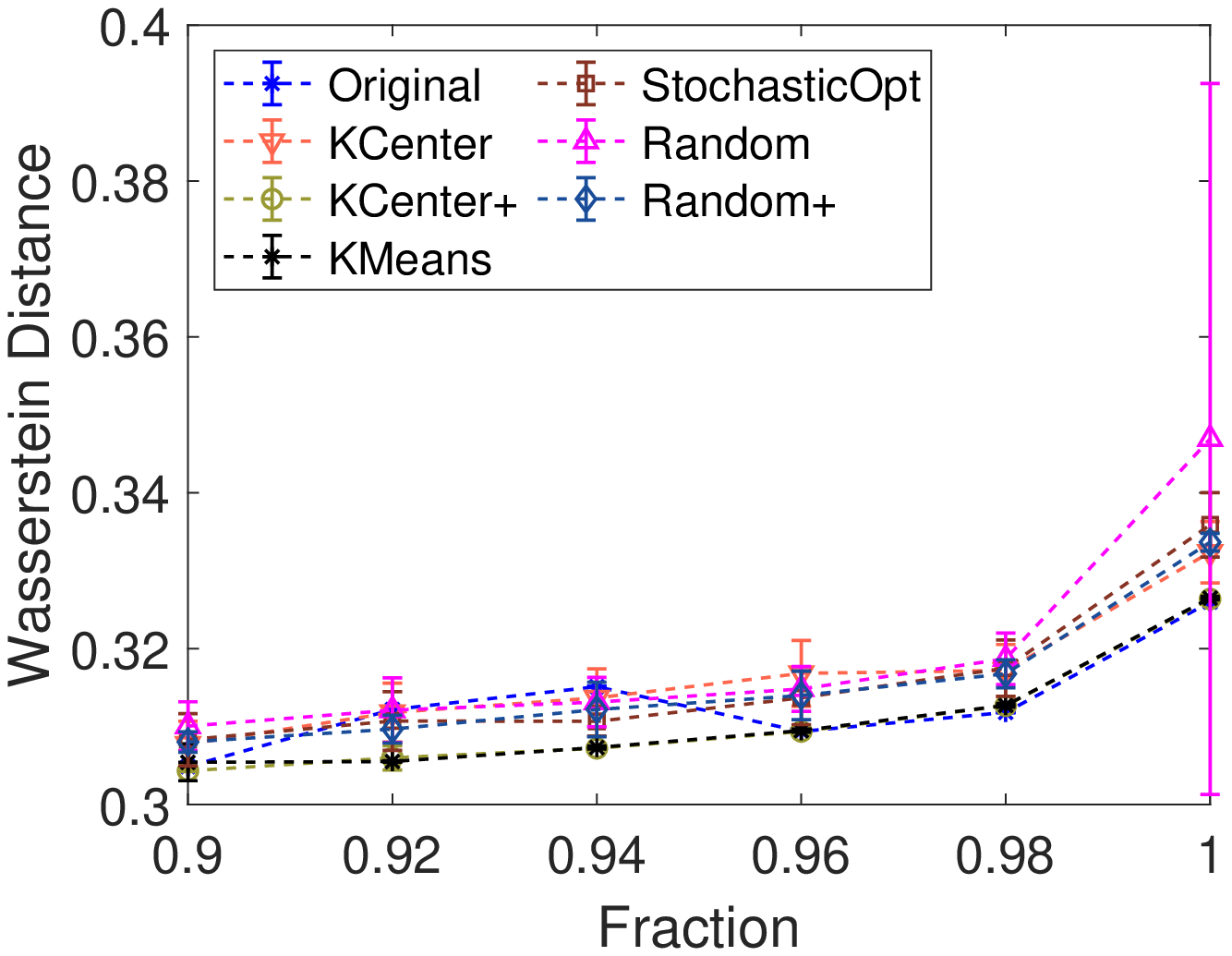} 
			\hspace{0.1in}
			\includegraphics[width=0.44\columnwidth]{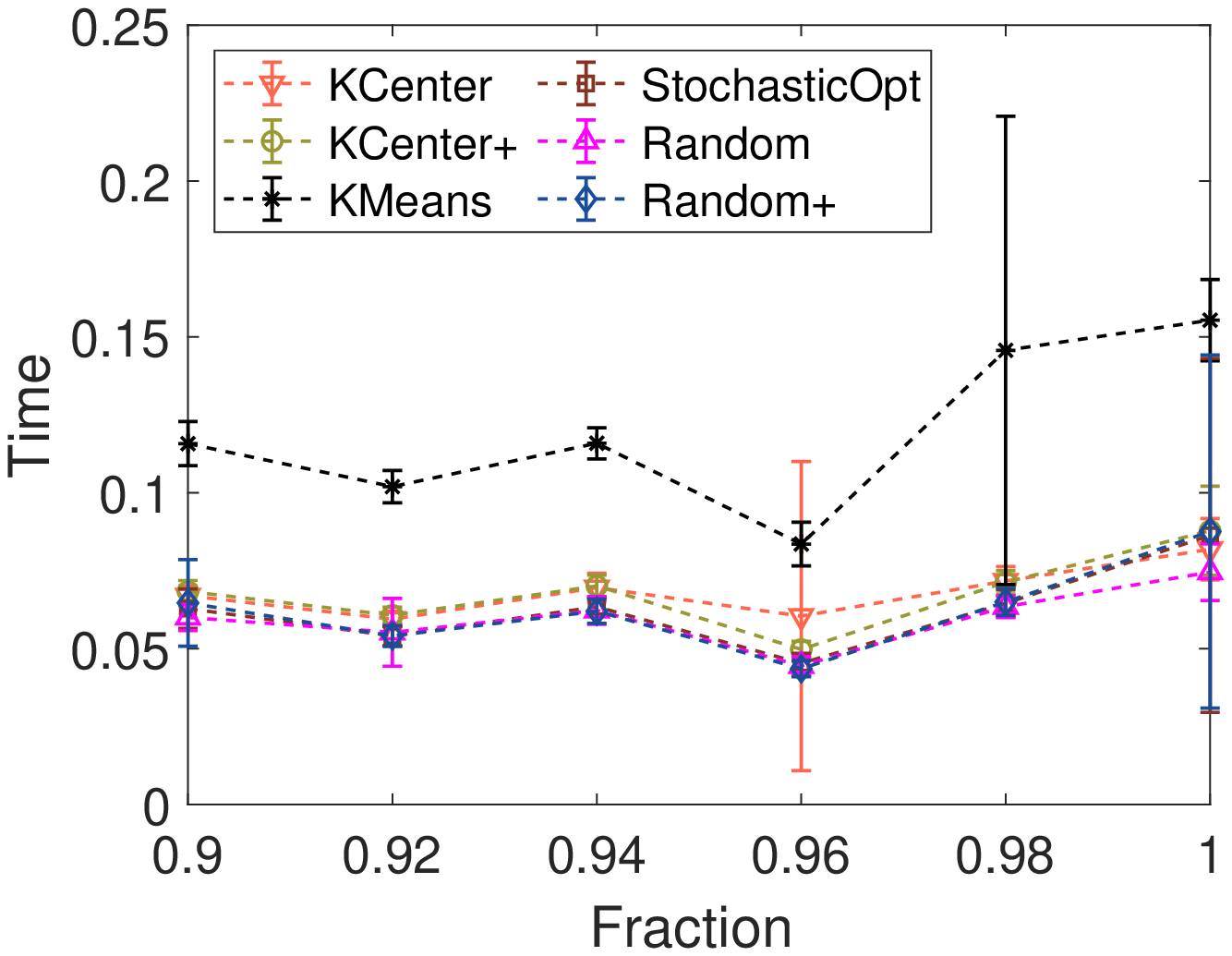}}
	\end{center}
	\vspace{-0.25in}
	\caption{The Wasserstein  distance and normalized running time on \textsc{CG} for PPI network alignment with different fraction $\lambda$. } 
	\label{fig:CG2}  
\end{figure}

\begin{figure}[htbp]
	\begin{center}
		\centerline
		{\includegraphics[width=0.44\columnwidth]{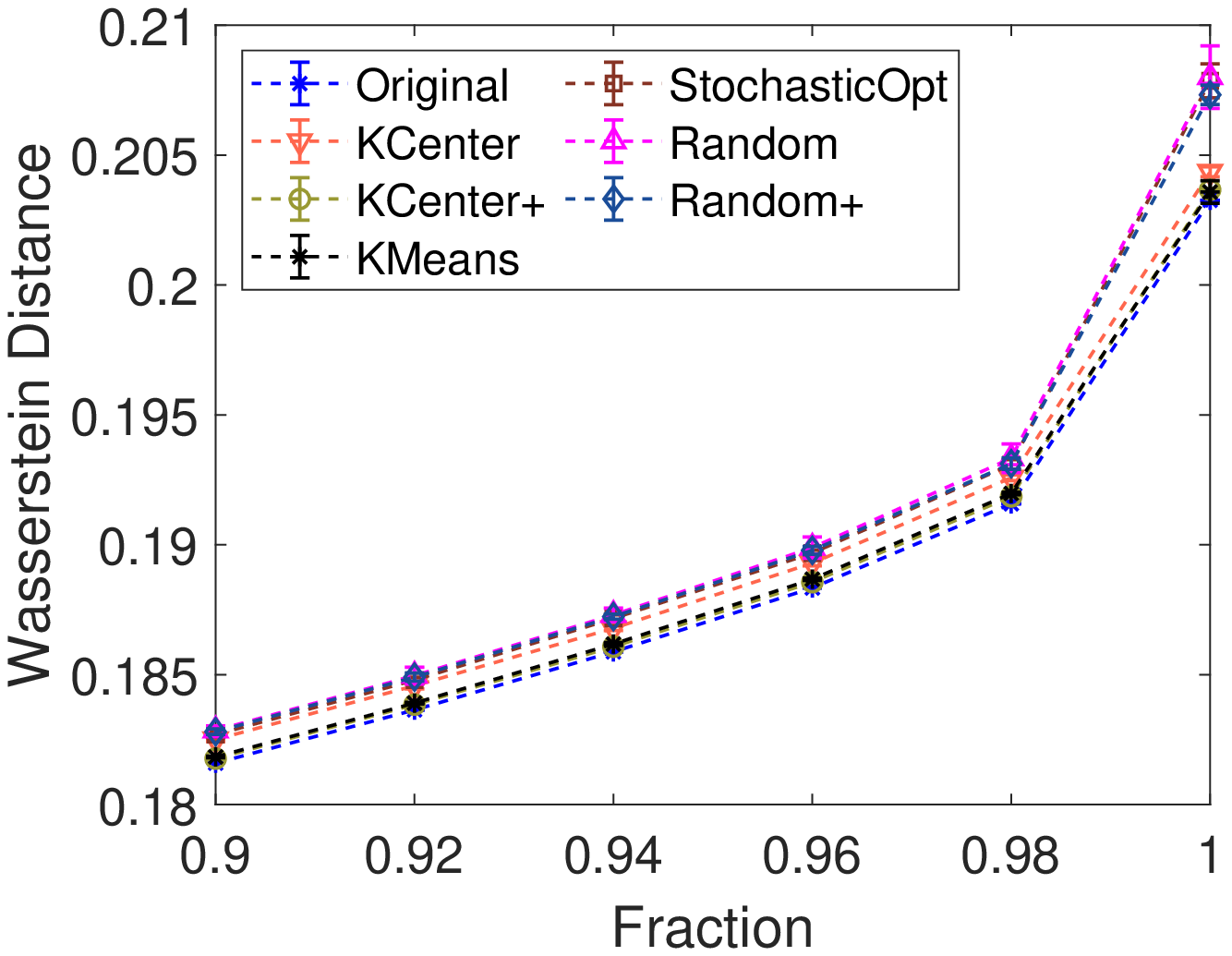} 
			\hspace{0.1in}
			\includegraphics[width=0.44\columnwidth]{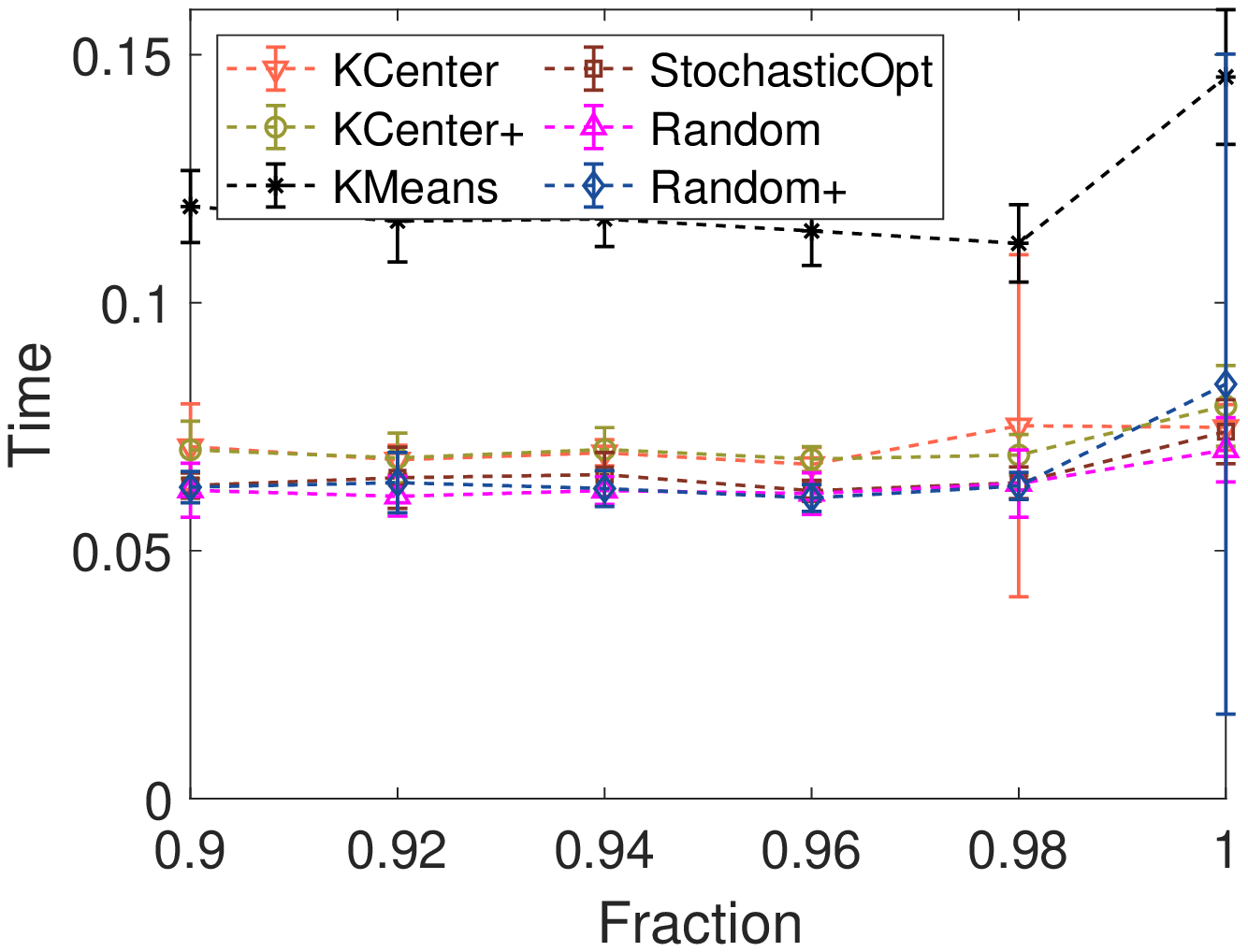}}
	\end{center}
	\vspace{-0.25in}
	\caption{The Wasserstein  distance and normalized running time on \textsc{DMC} for PPI network alignment with different fraction $\lambda$. } 
	\label{fig:DMC2}  
\end{figure}

\begin{figure}[htbp]
	\begin{center}
		\centerline
		{\includegraphics[width=0.44\columnwidth]{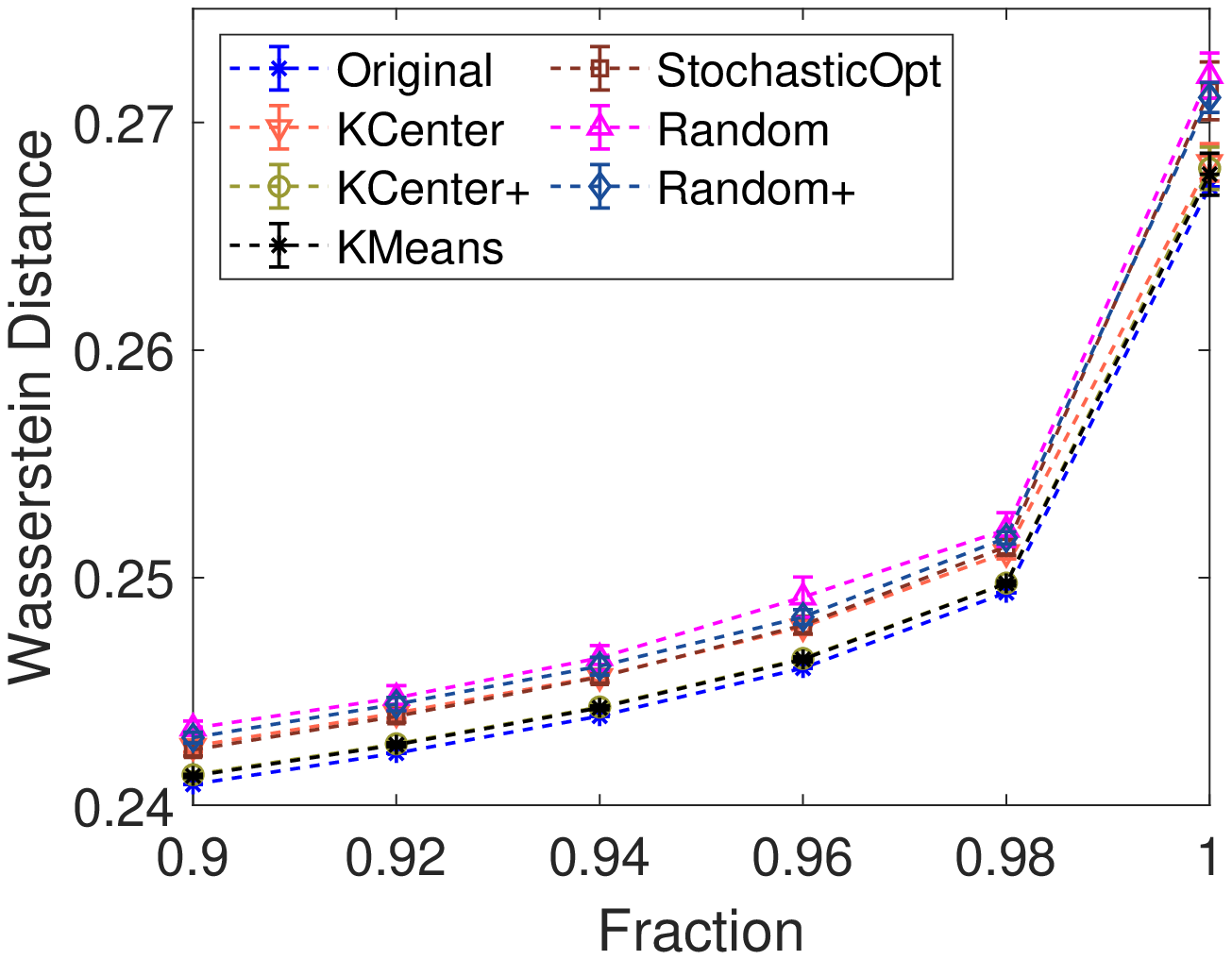} 
			\hspace{0.1in}
			\includegraphics[width=0.44\columnwidth]{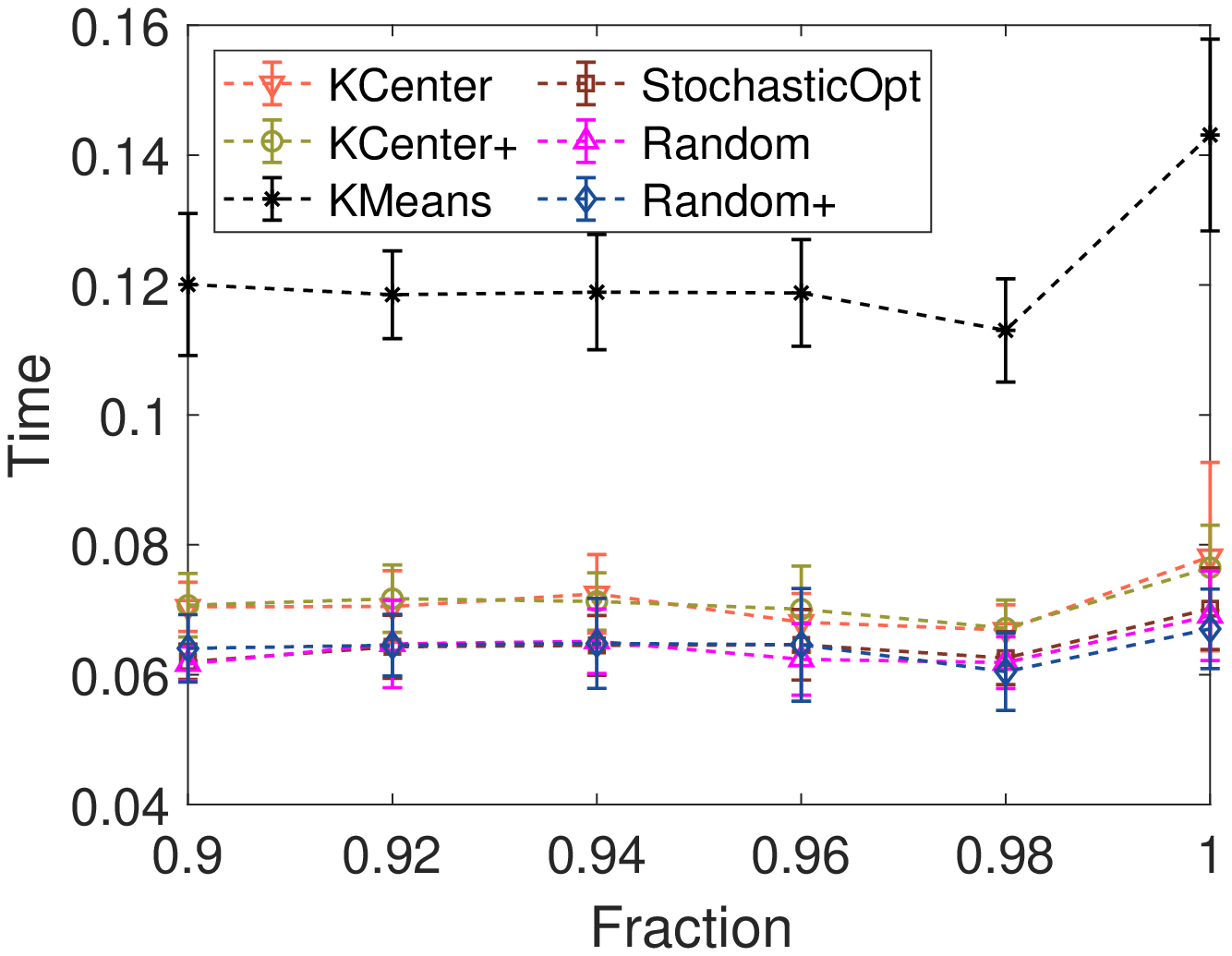}}
	\end{center}
	\vspace{-0.25in}
	\caption{The Wasserstein  distance and normalized running time on \textsc{DMR} for PPI network alignment with different fraction $\lambda$. } 
	\label{fig:DMR2}  
\end{figure}

\begin{figure}[htbp]
	\begin{center}
		\centerline
		{\includegraphics[width=0.44\columnwidth]{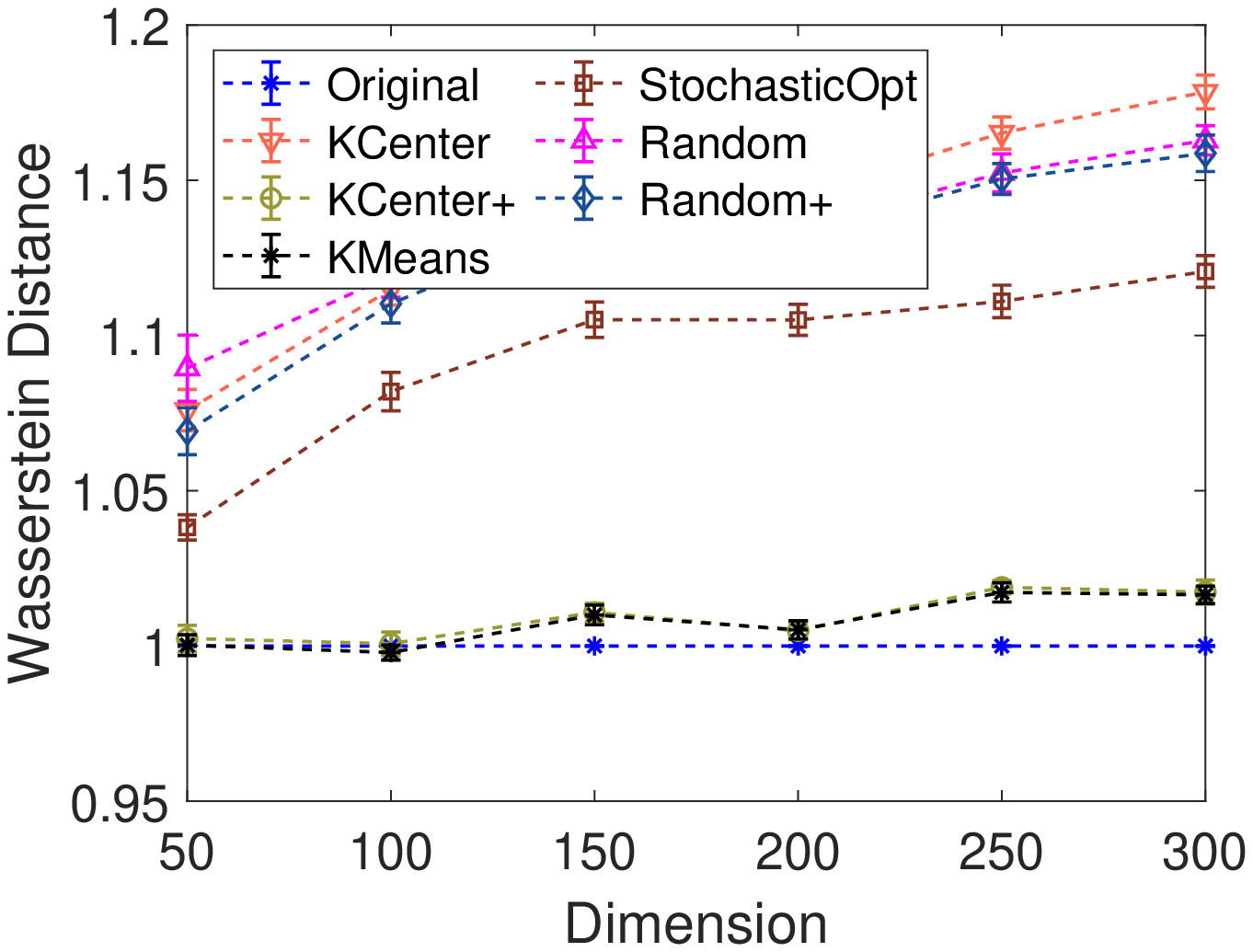} 
			\hspace{0.1in}
			\includegraphics[width=0.44\columnwidth]{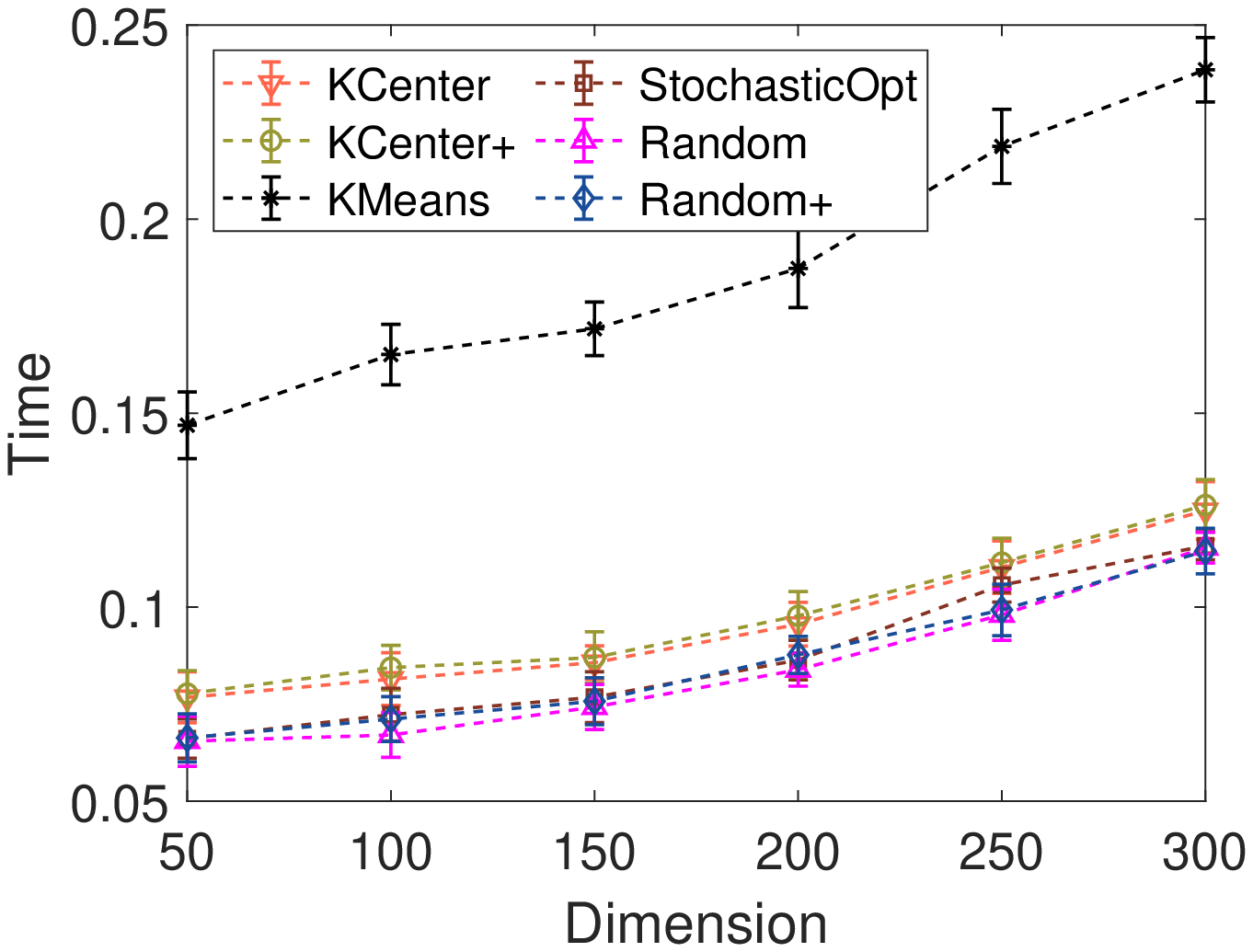}}
	\end{center}
	\vspace{-0.25in}
	\caption{The normalized Wasserstein  distance and normalized running time (over \textsc{Orignial}) on CG for PPI network alignment  with different dimensions. } 
	\label{fig:CG_dim_size} 
\end{figure}

\begin{figure}[htbp]
	\begin{center}
		\centerline
		{\includegraphics[width=0.44\columnwidth]{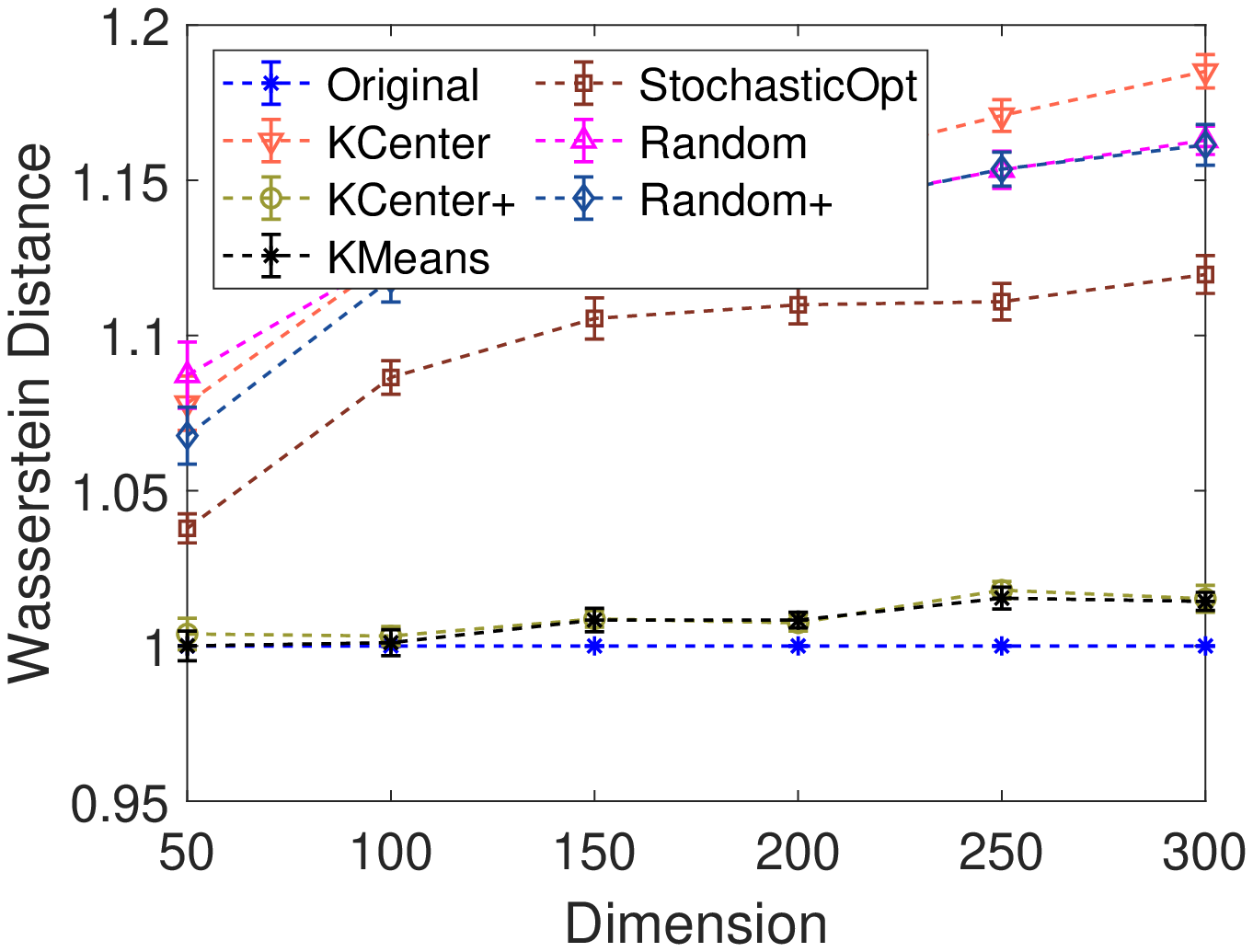} 
			\hspace{0.1in}
			\includegraphics[width=0.44\columnwidth]{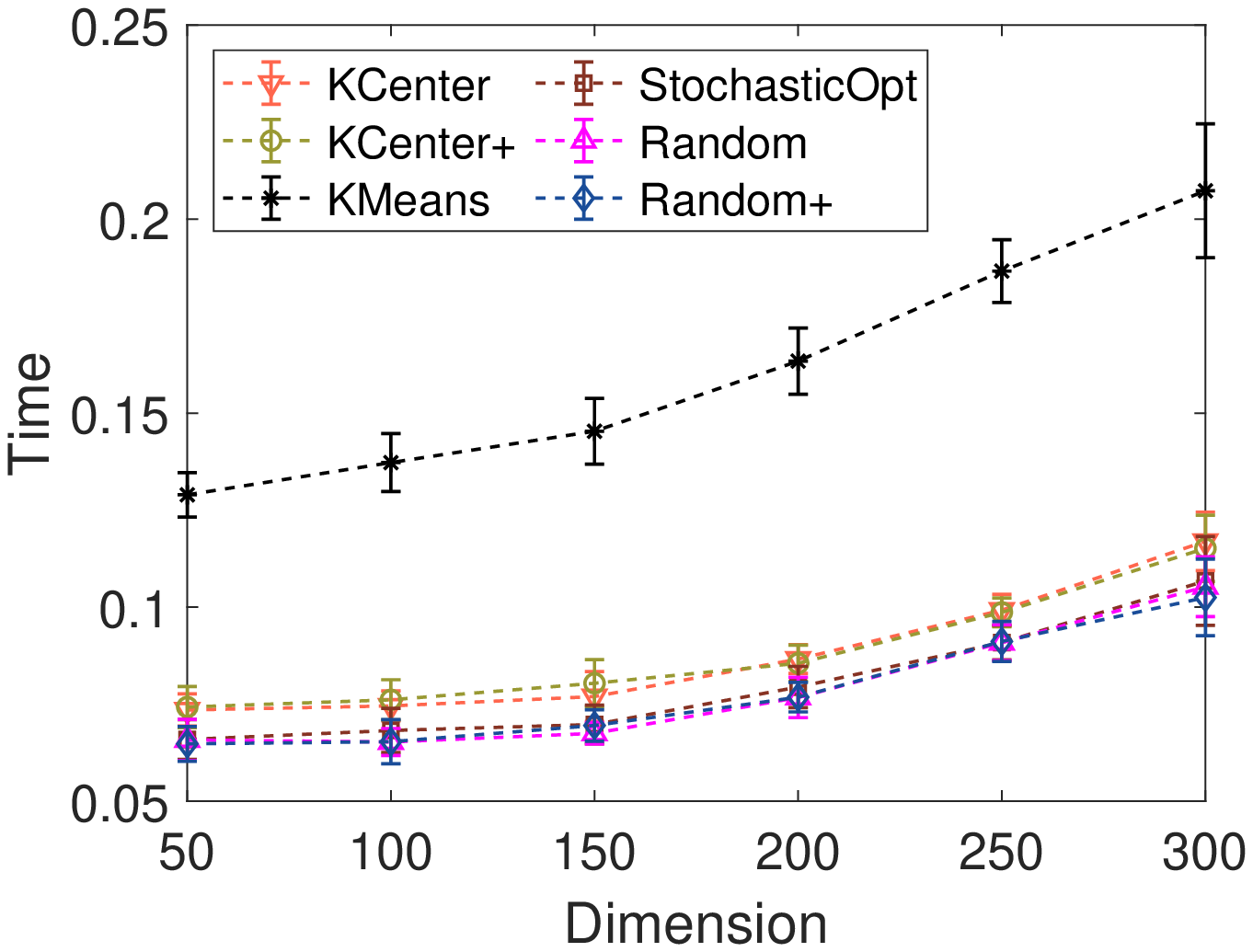}}
	\end{center}
	\vspace{-0.25in}
	\caption{The normalized fractional Wasserstein  distance and normalized running time (over \textsc{Orignial}) on CG for PPI network alignment with different dimensions. The fraction $\lambda$ is    $0.9$. } 
	\label{fig:CG_dim_noise} 
\end{figure}

\begin{figure}[H]
	\begin{center}
		\centerline
		{\includegraphics[width=0.44\columnwidth]{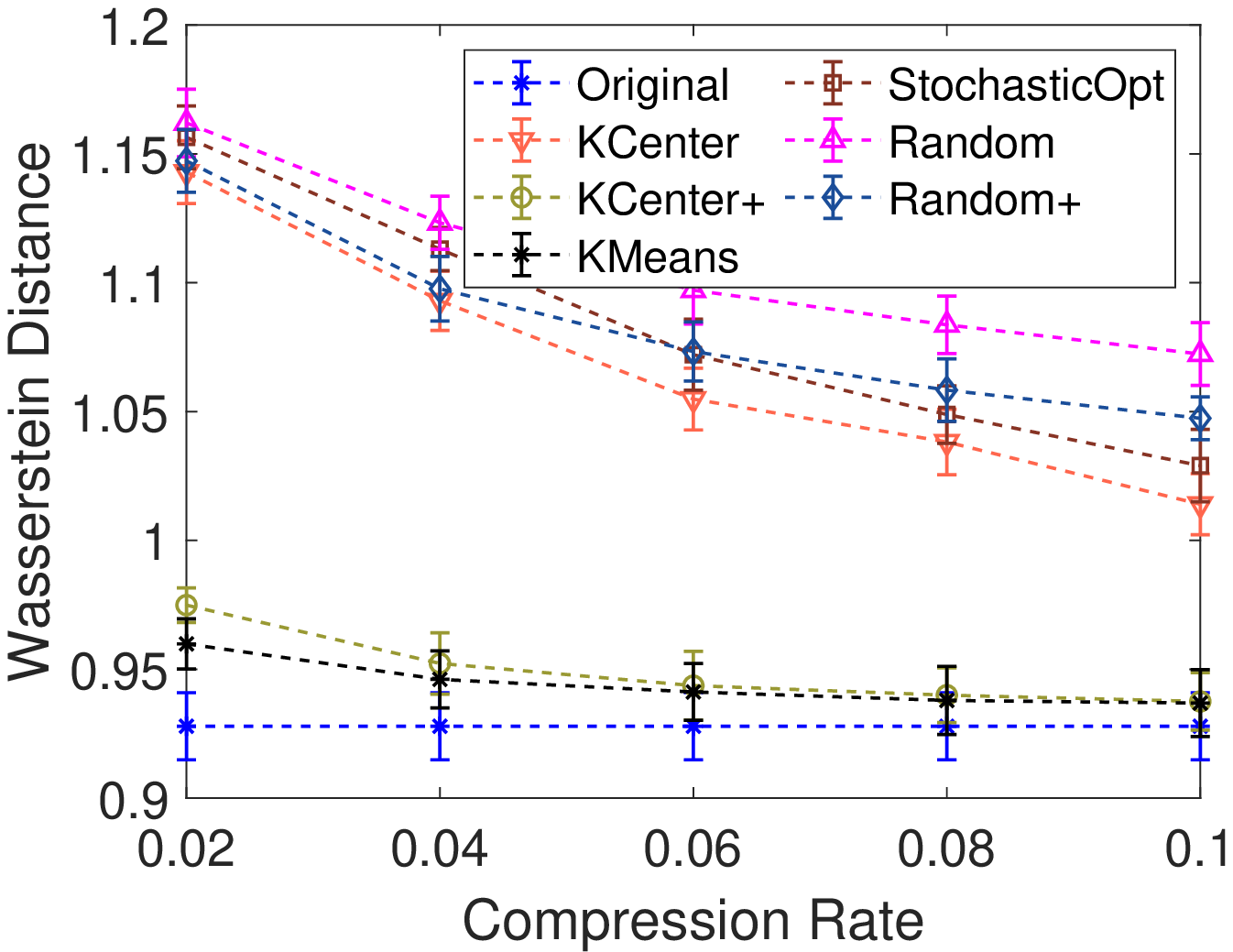} 
			\hspace{0.1in}
			\includegraphics[width=0.44\columnwidth]{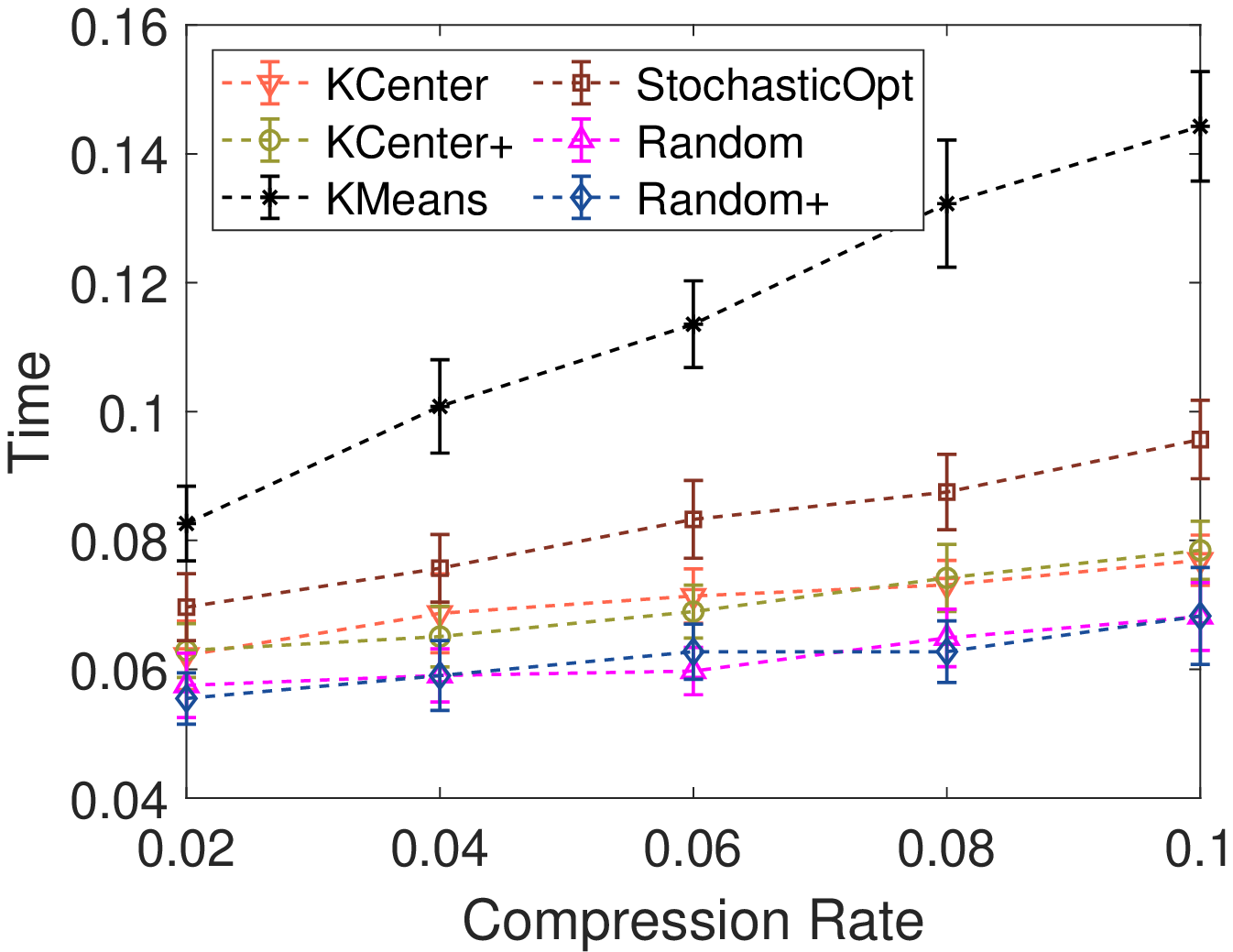}}
	\end{center}
	\vspace{-0.25in}
	\caption{The Wasserstein  distance and normalized running time on \textsc{es-en} for bilingual lexicon induction. } 
	\label{fig:esen1}  
\end{figure}

\begin{figure}[H]
	\begin{center}
		\centerline
		{\includegraphics[width=0.42\columnwidth]{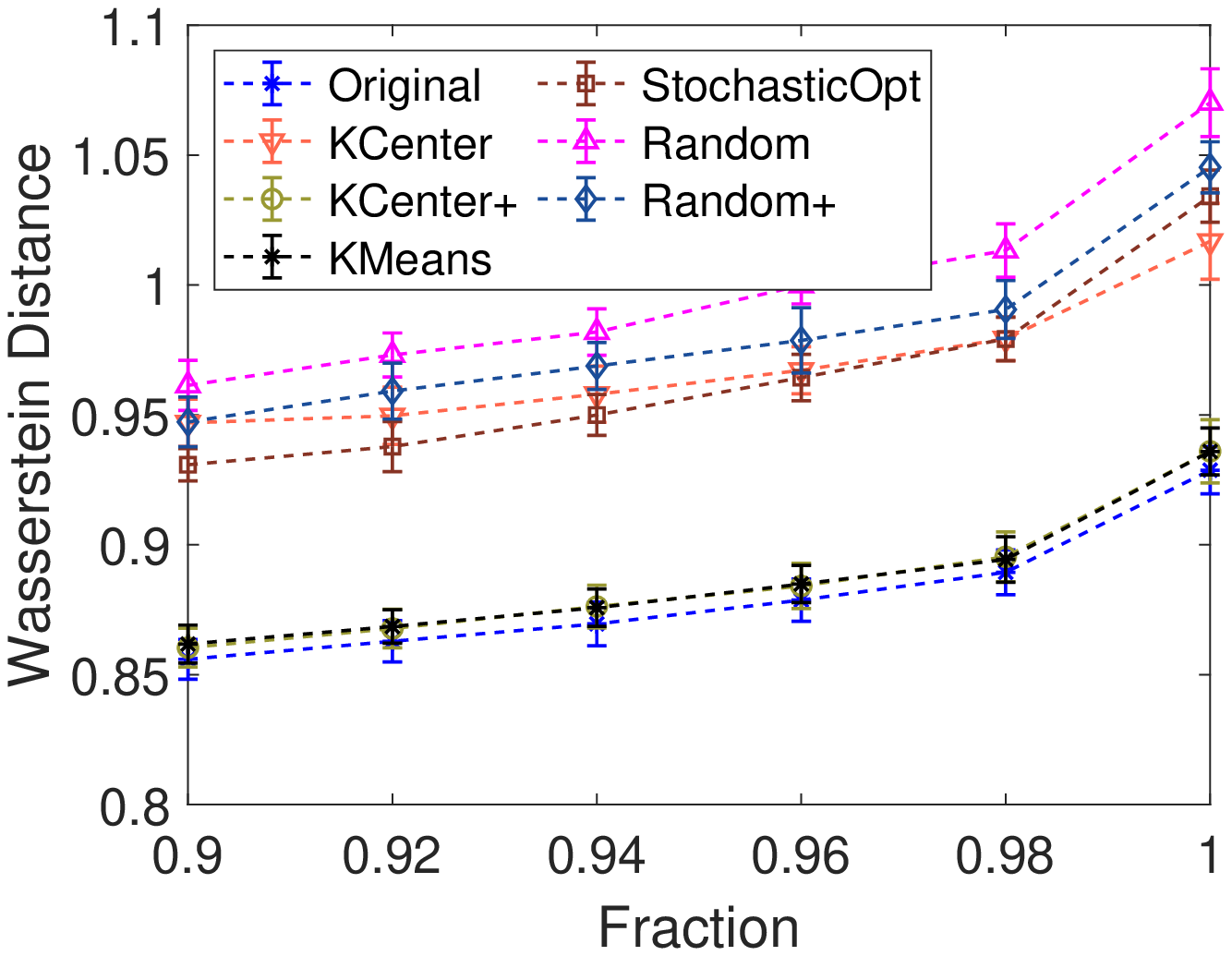} 
			\hspace{0.1in}
			\includegraphics[width=0.42\columnwidth]{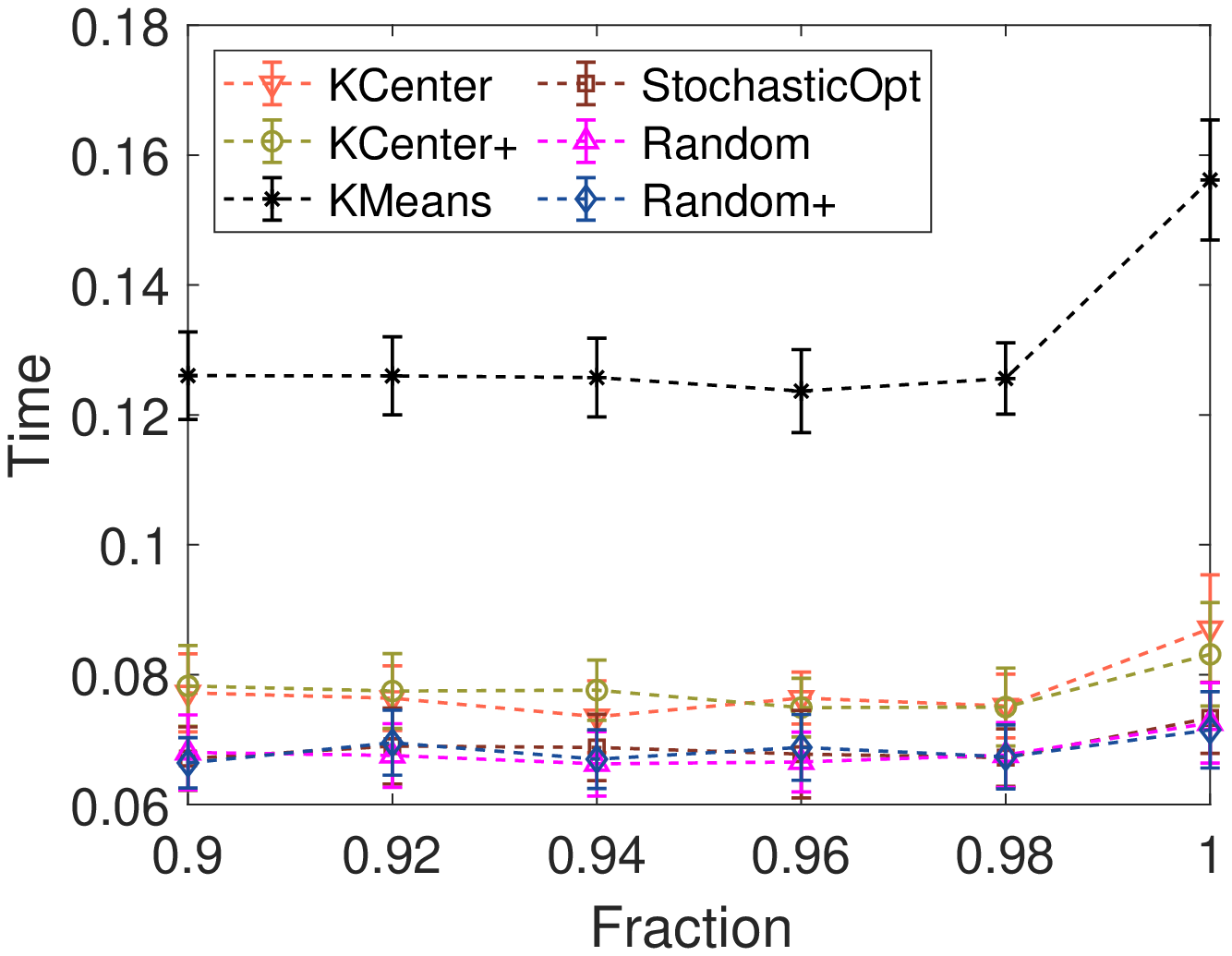}}
	\end{center}
	\vspace{-0.25in}
	\caption{The Wasserstein  distance and normalized running time on \textsc{es-en} for bilingual lexicon induction with different fraction $\lambda$. } 
	\label{fig:esen2}  
\end{figure}

\begin{figure}[H]
	\begin{center}
		\centerline
		{\includegraphics[width=0.44\columnwidth]{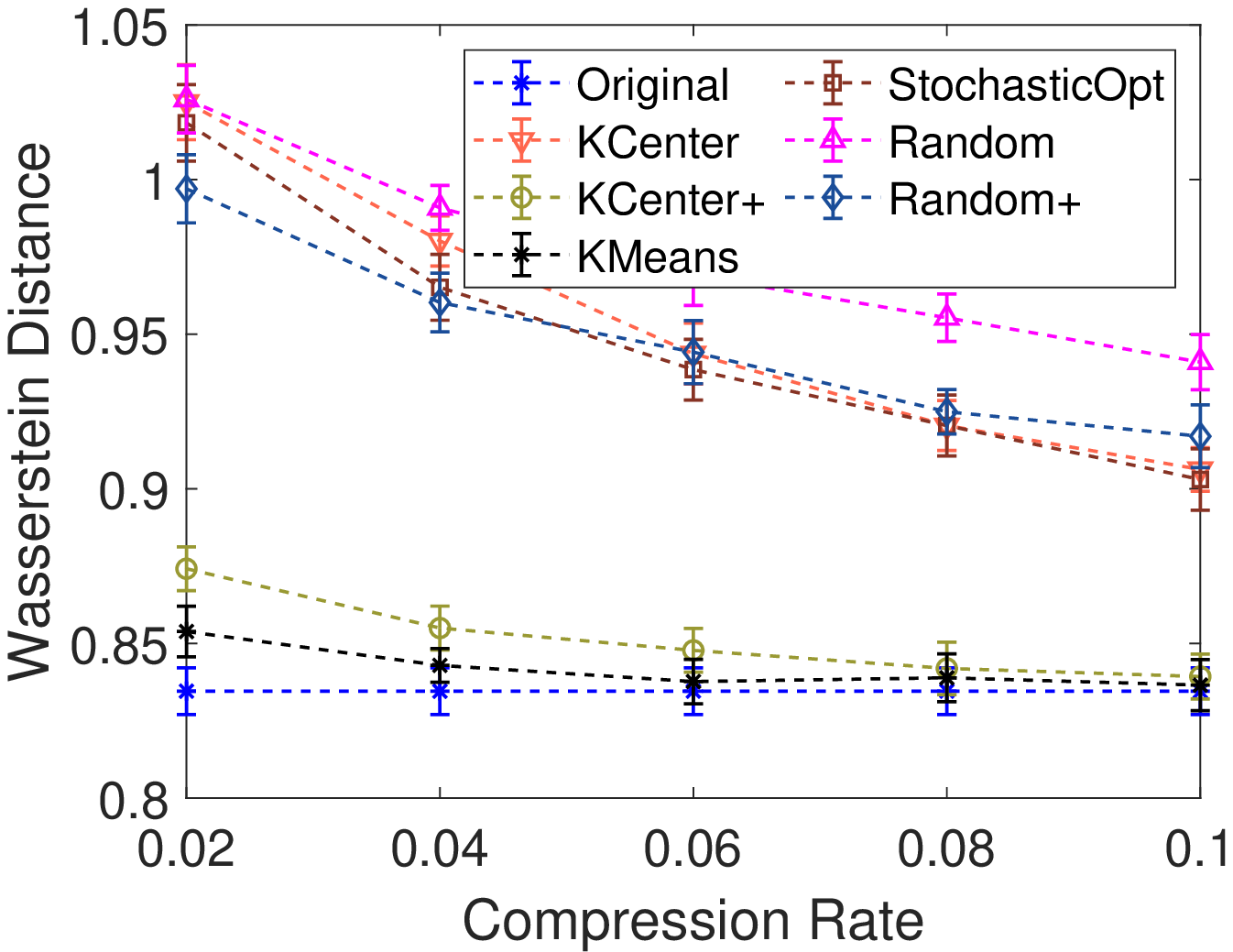} 
			\hspace{0.1in}
			\includegraphics[width=0.44\columnwidth]{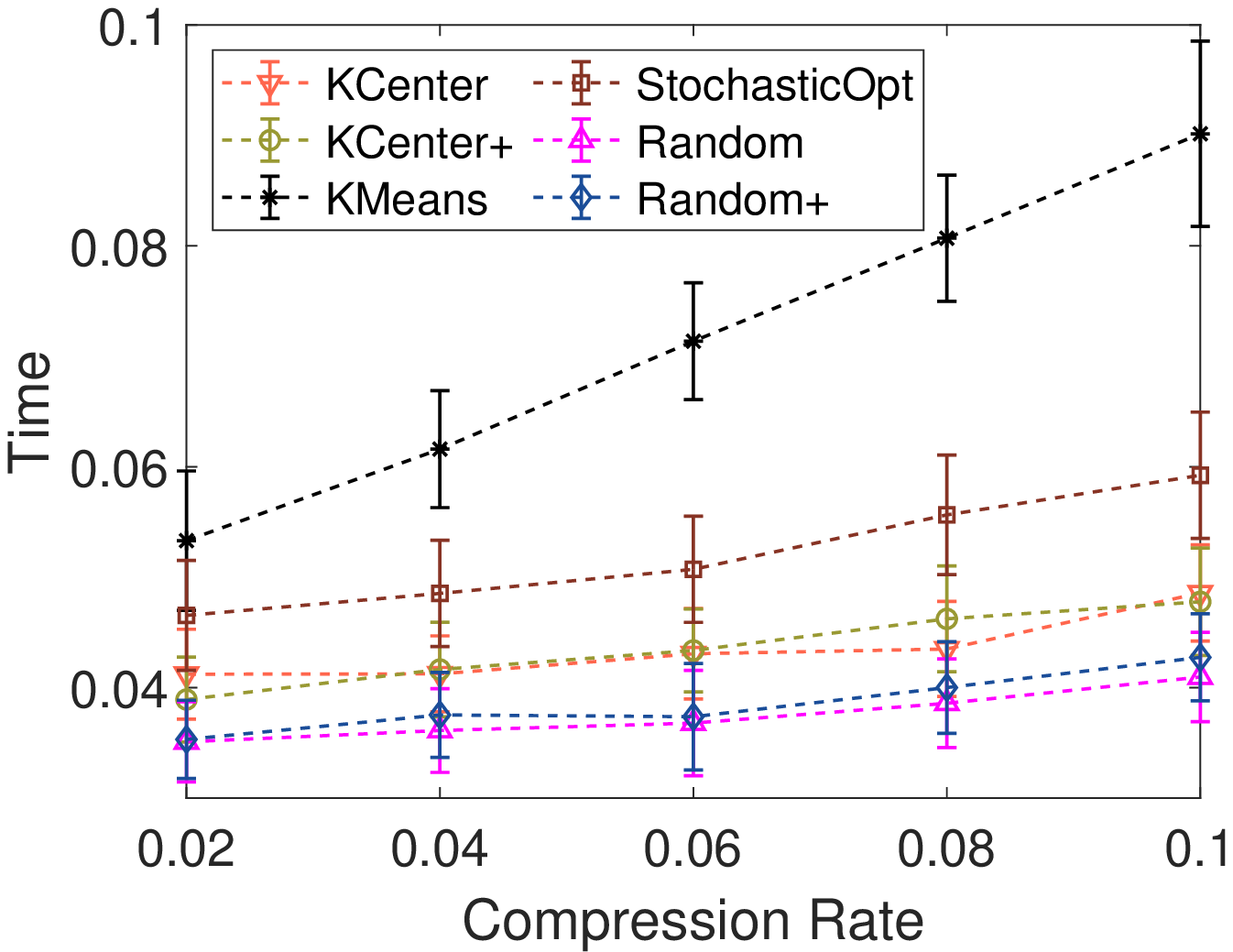}}
	\end{center}
	\vspace{-0.25in}
	\caption{The Wasserstein  distance and normalized running time on \textsc{it-en} for bilingual lexicon induction. } 
	\label{fig:iten1}  
\end{figure}

\begin{figure}[H]
	\begin{center}
		\centerline
		{\includegraphics[width=0.44\columnwidth]{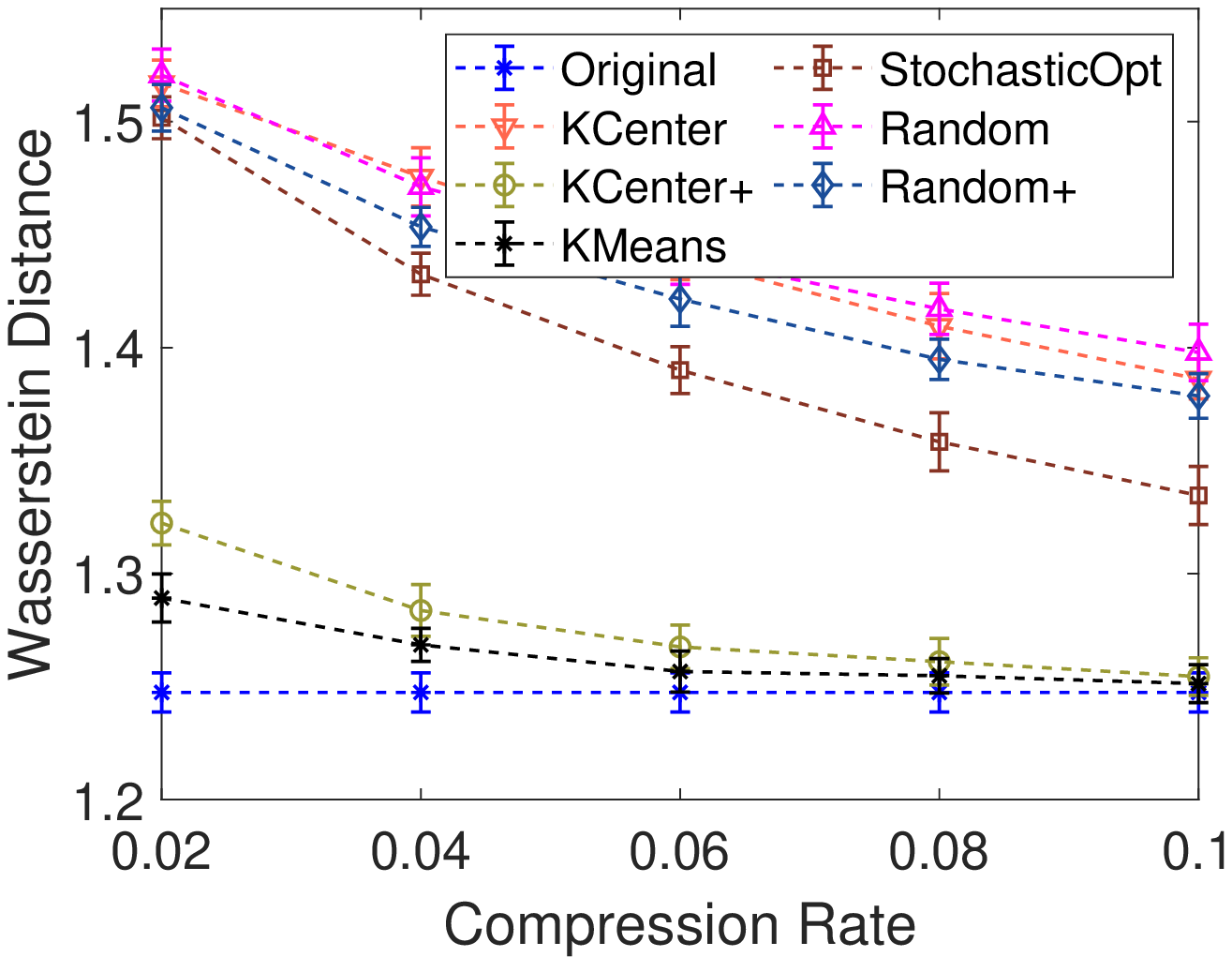} 
			\hspace{0.1in}
			\includegraphics[width=0.44\columnwidth]{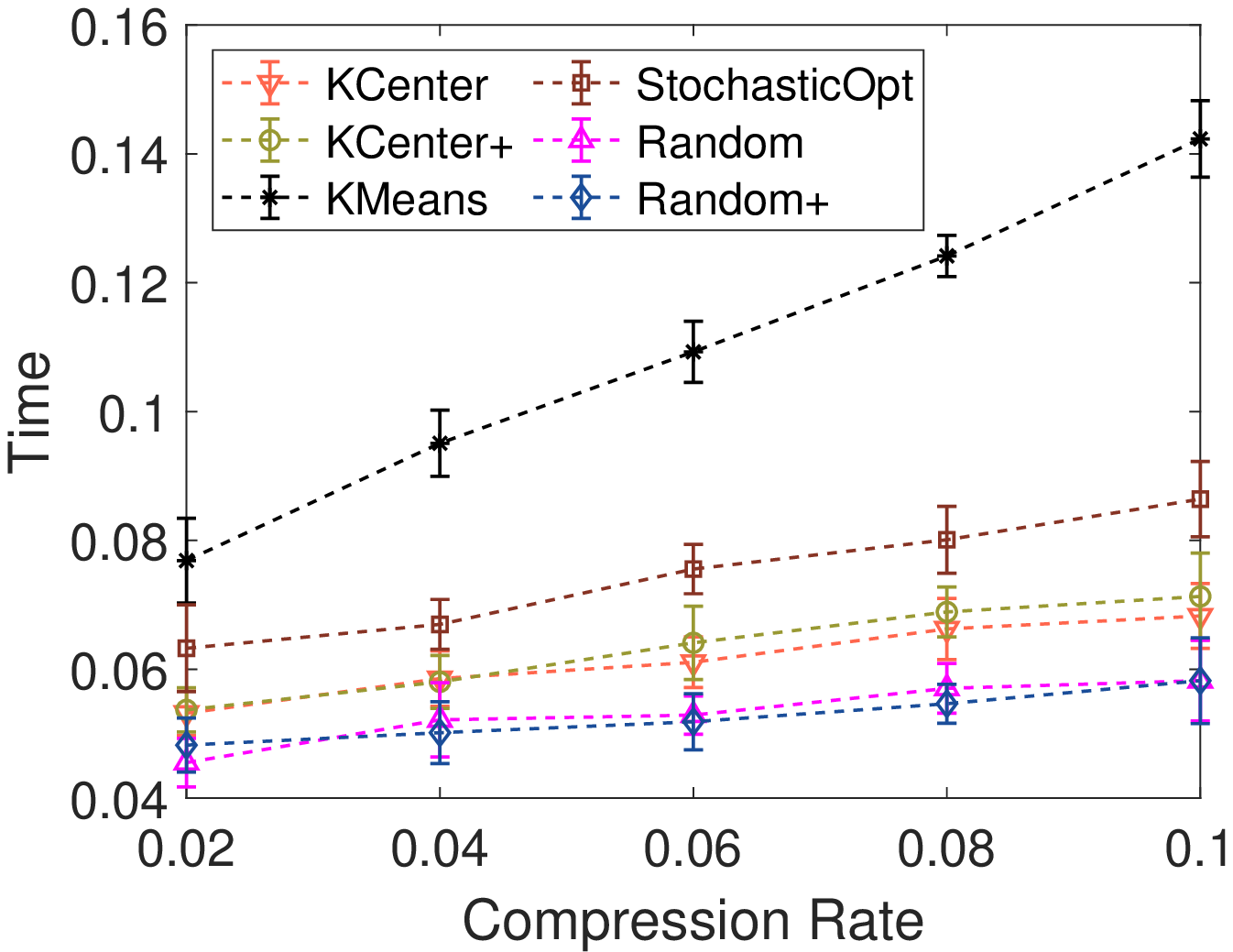}}
	\end{center}
	\vspace{-0.25in}
	\caption{The Wasserstein  distance and normalized running time on \textsc{ja-zh} for bilingual lexicon induction. } 
	\label{fig:jazh1}  
	
\end{figure}
\begin{figure}[H]
	\begin{center}
		\centerline
		{\includegraphics[width=0.44\columnwidth]{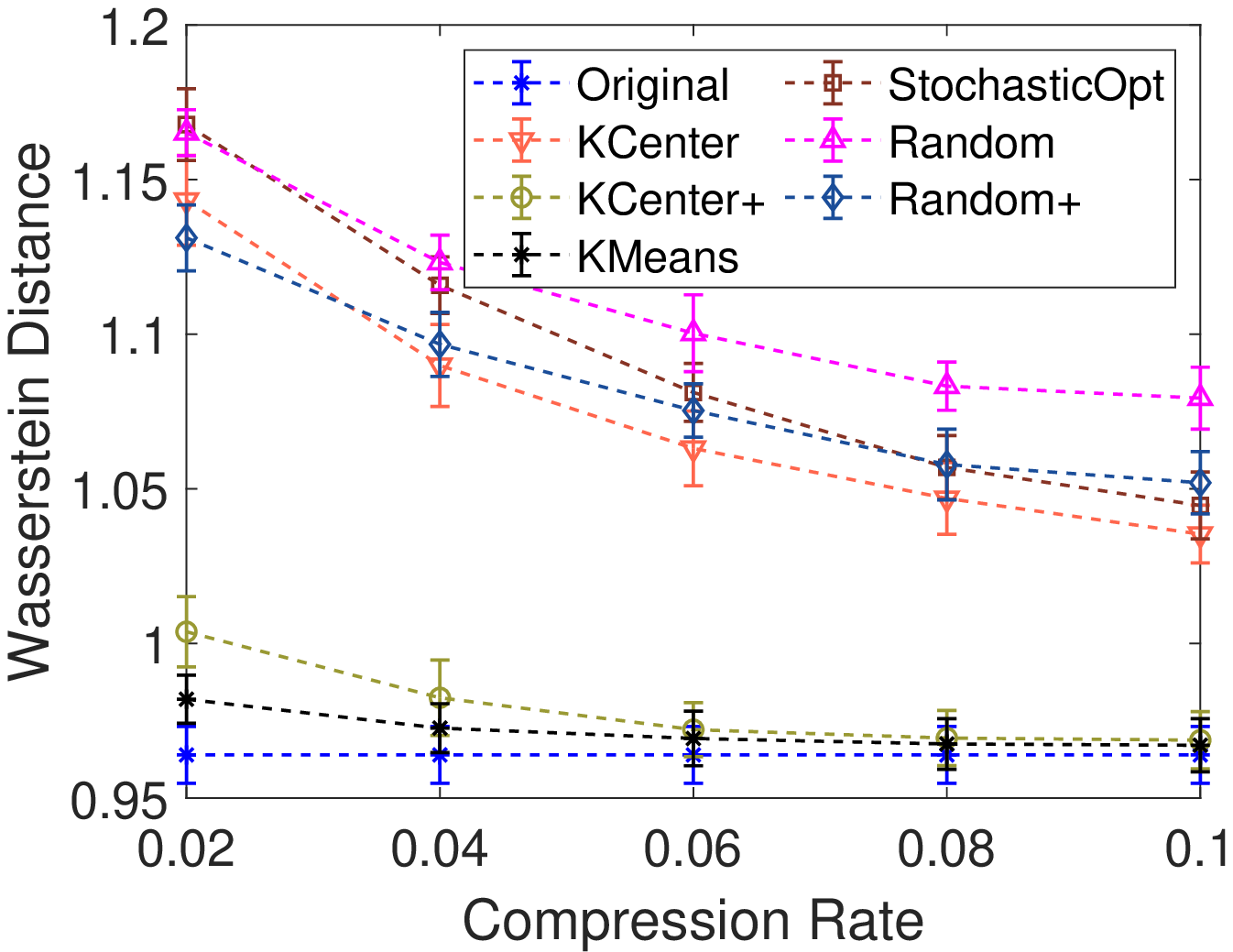} 
			\hspace{0.1in}
			\includegraphics[width=0.44\columnwidth]{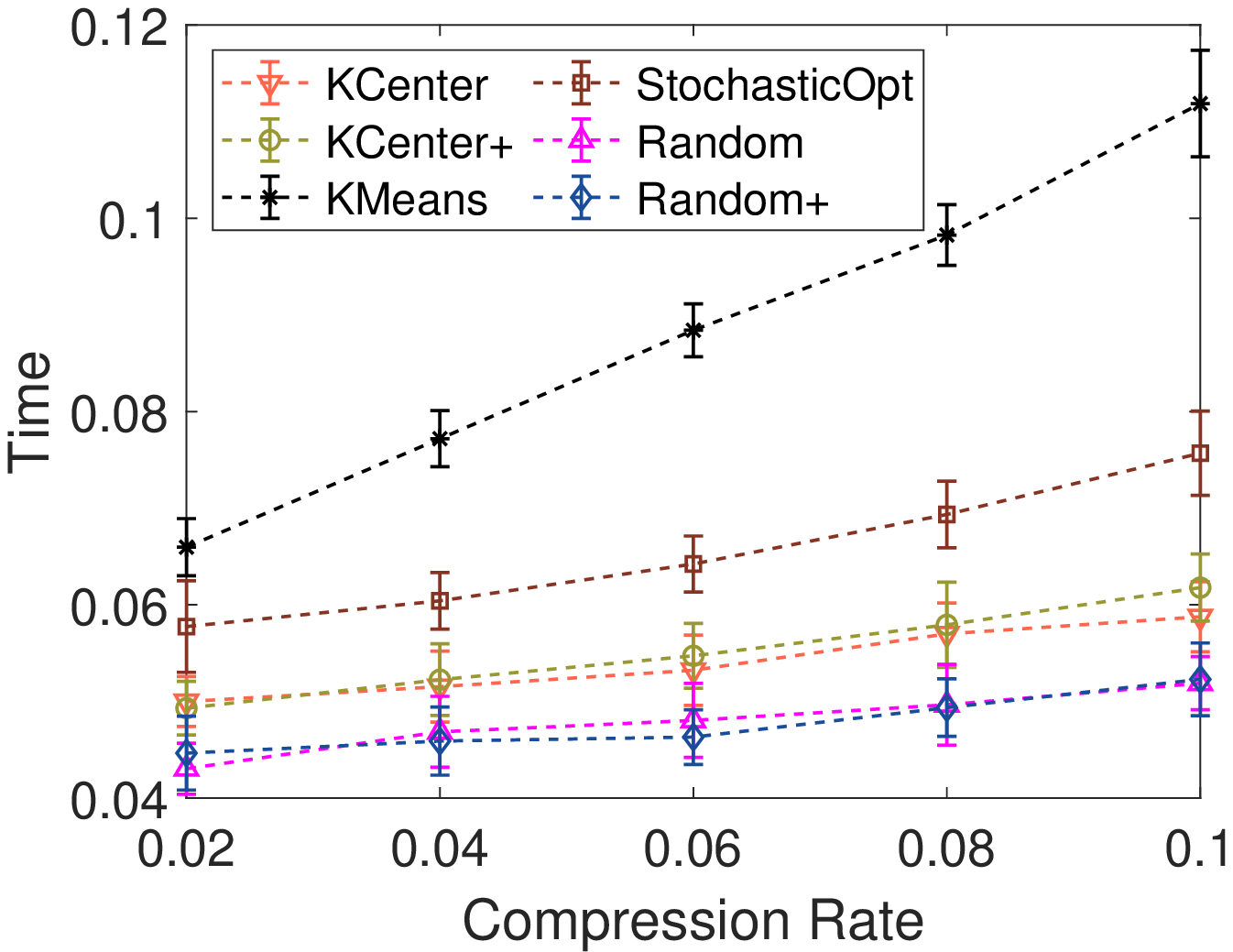}}
	\end{center}
	\vspace{-0.25in}
	\caption{The Wasserstein  distance and normalized running time on \textsc{tr-en} for bilingual lexicon induction. } 
	\label{fig:tren1}  
\end{figure}

\begin{figure}[H]
	\begin{center}
		\centerline
		{\includegraphics[width=0.44\columnwidth]{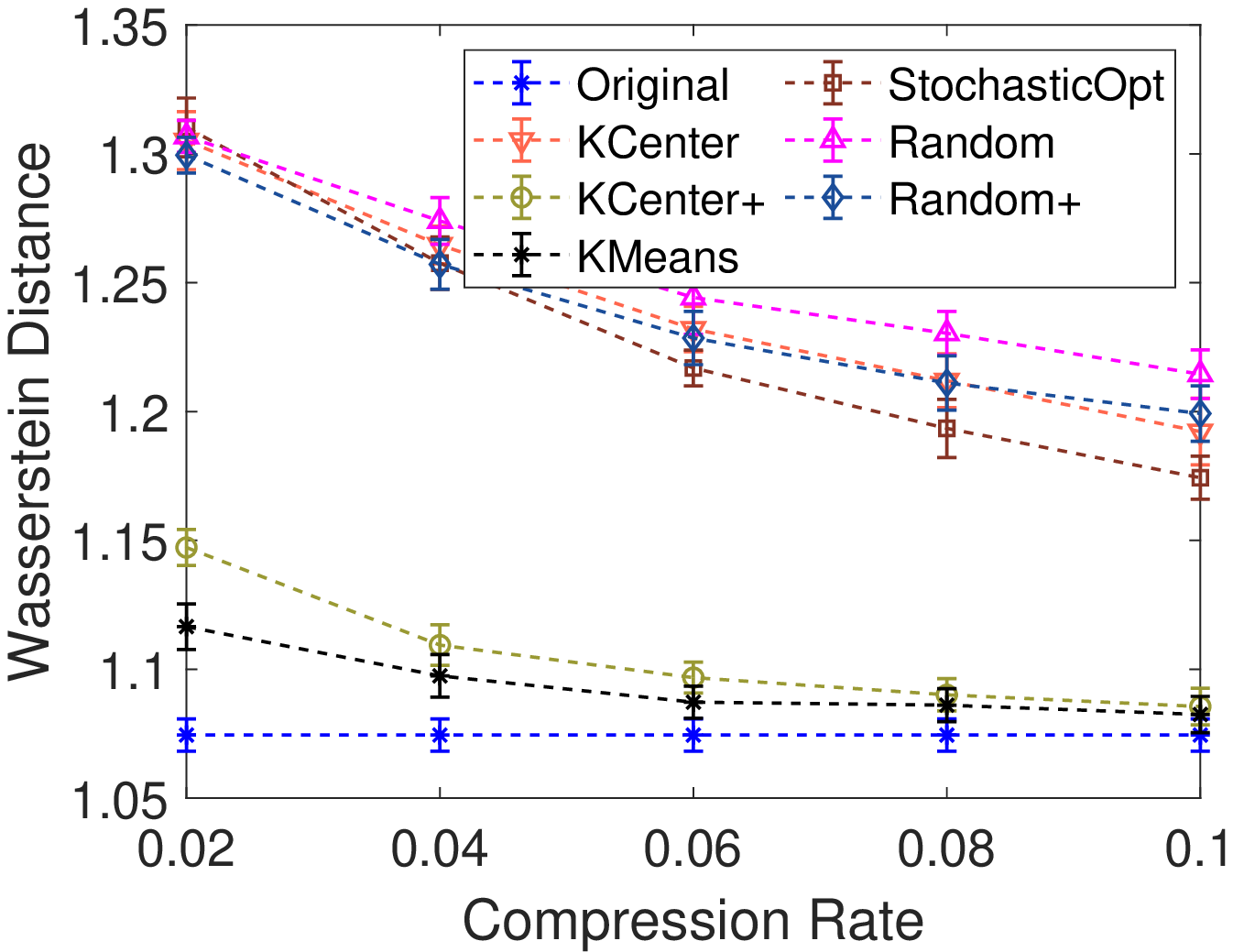} 
			\hspace{0.1in}
			\includegraphics[width=0.44\columnwidth]{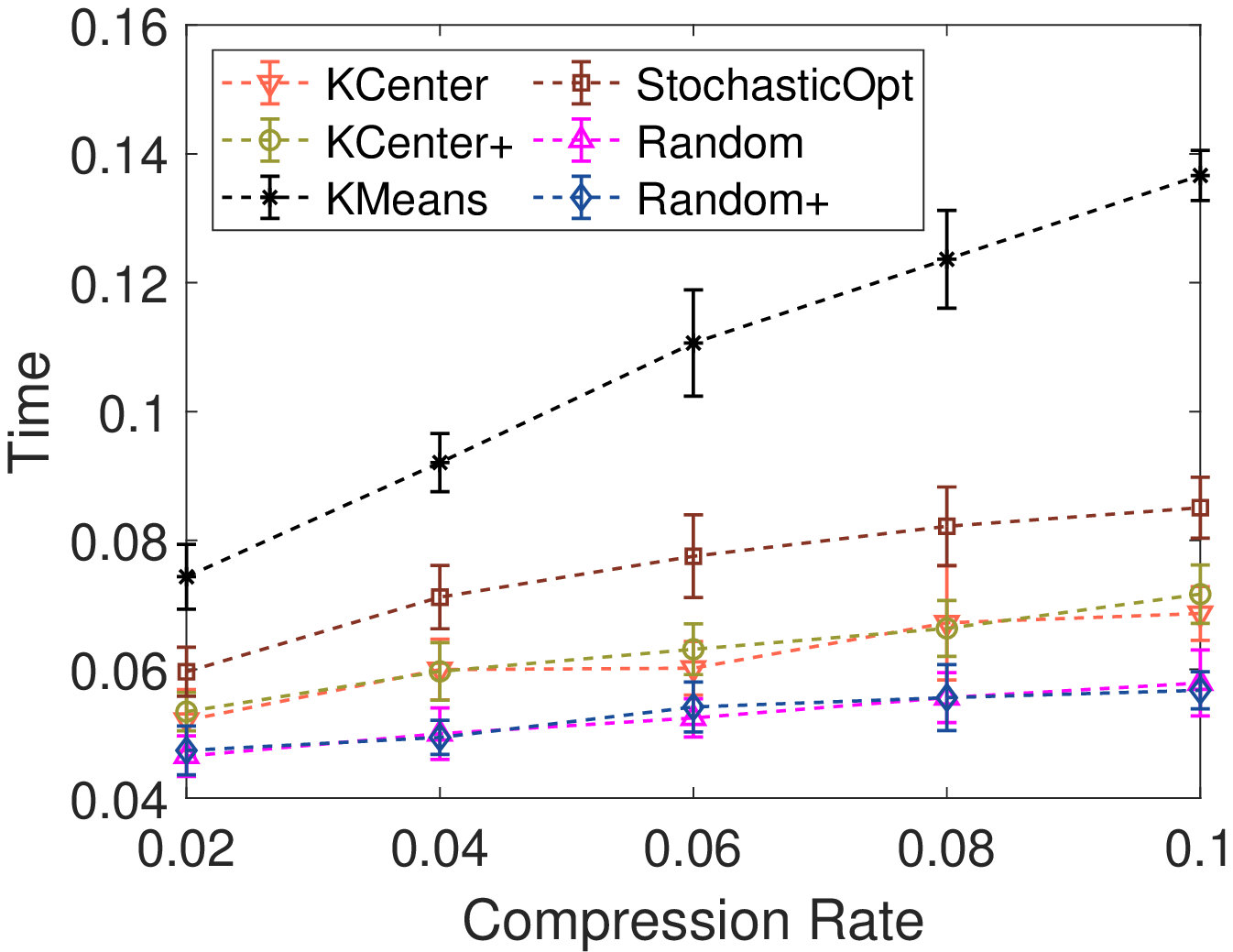}}
	\end{center}
	\vspace{-0.25in}
	\caption{The Wasserstein  distance and normalized running time on \textsc{zh-en} for bilingual lexicon induction. } 
	\label{fig:zhen1}  
\end{figure}

\begin{figure}[H]
	\begin{center}
		\centerline
		{\includegraphics[width=0.42\columnwidth]{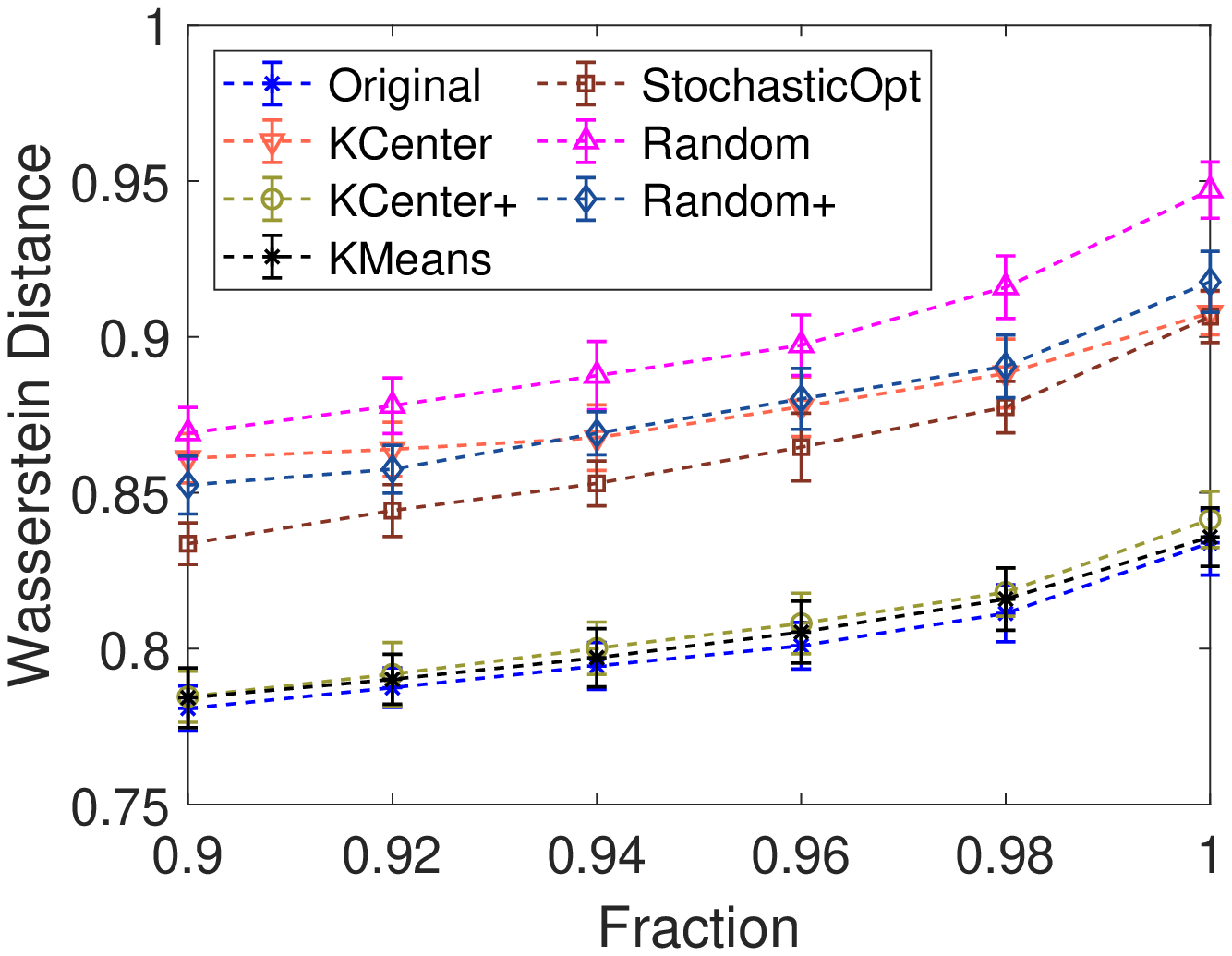} 
			\hspace{0.1in}
			\includegraphics[width=0.42\columnwidth]{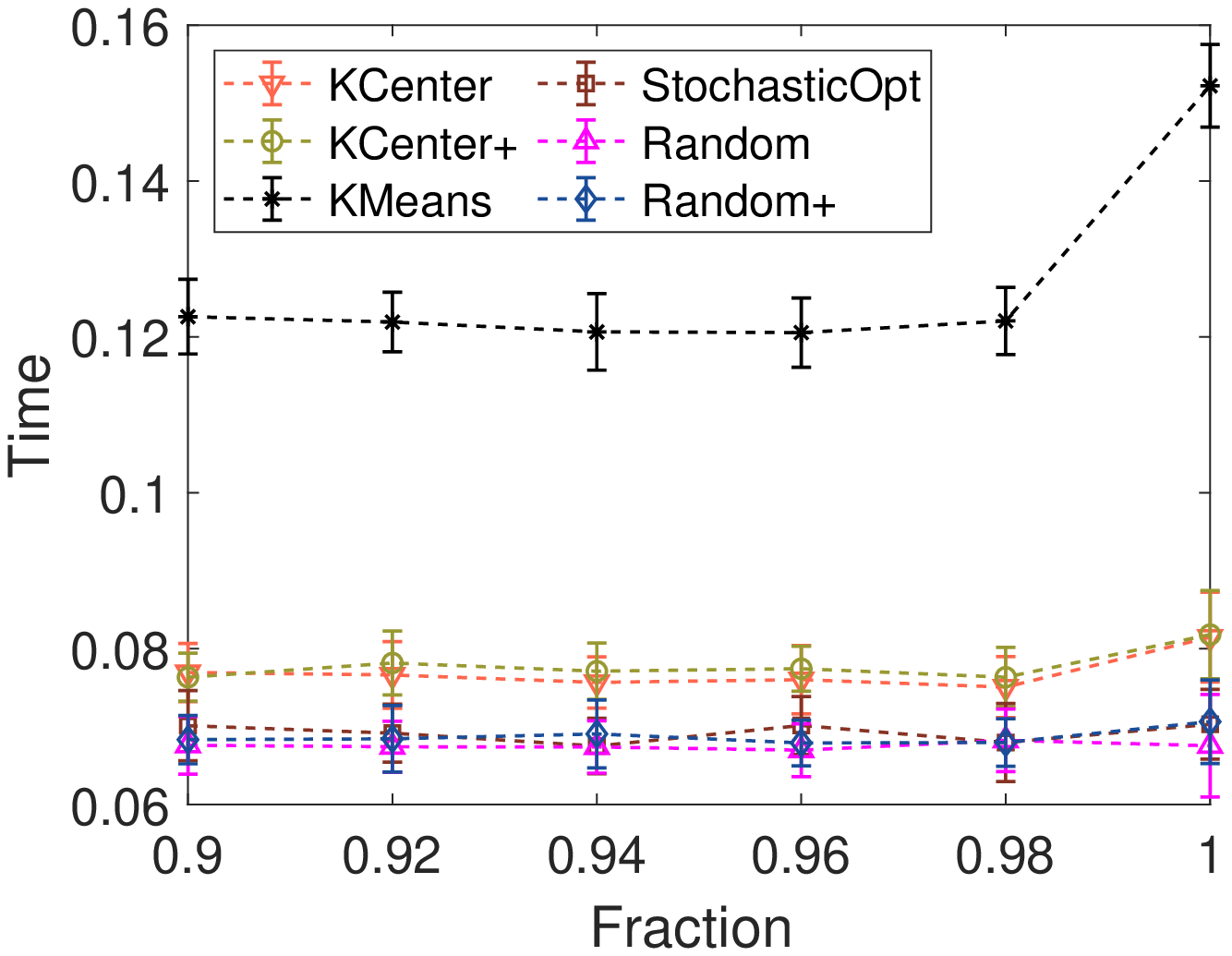}}
	\end{center}
	\vspace{-0.25in}
	\caption{The Wasserstein  distance and normalized running time on \textsc{it-en} for bilingual lexicon induction with different fraction $\lambda$. } 
	\label{fig:iten2}  
\end{figure}

\begin{figure}[H]
	\begin{center}
		\centerline
		{\includegraphics[width=0.42\columnwidth]{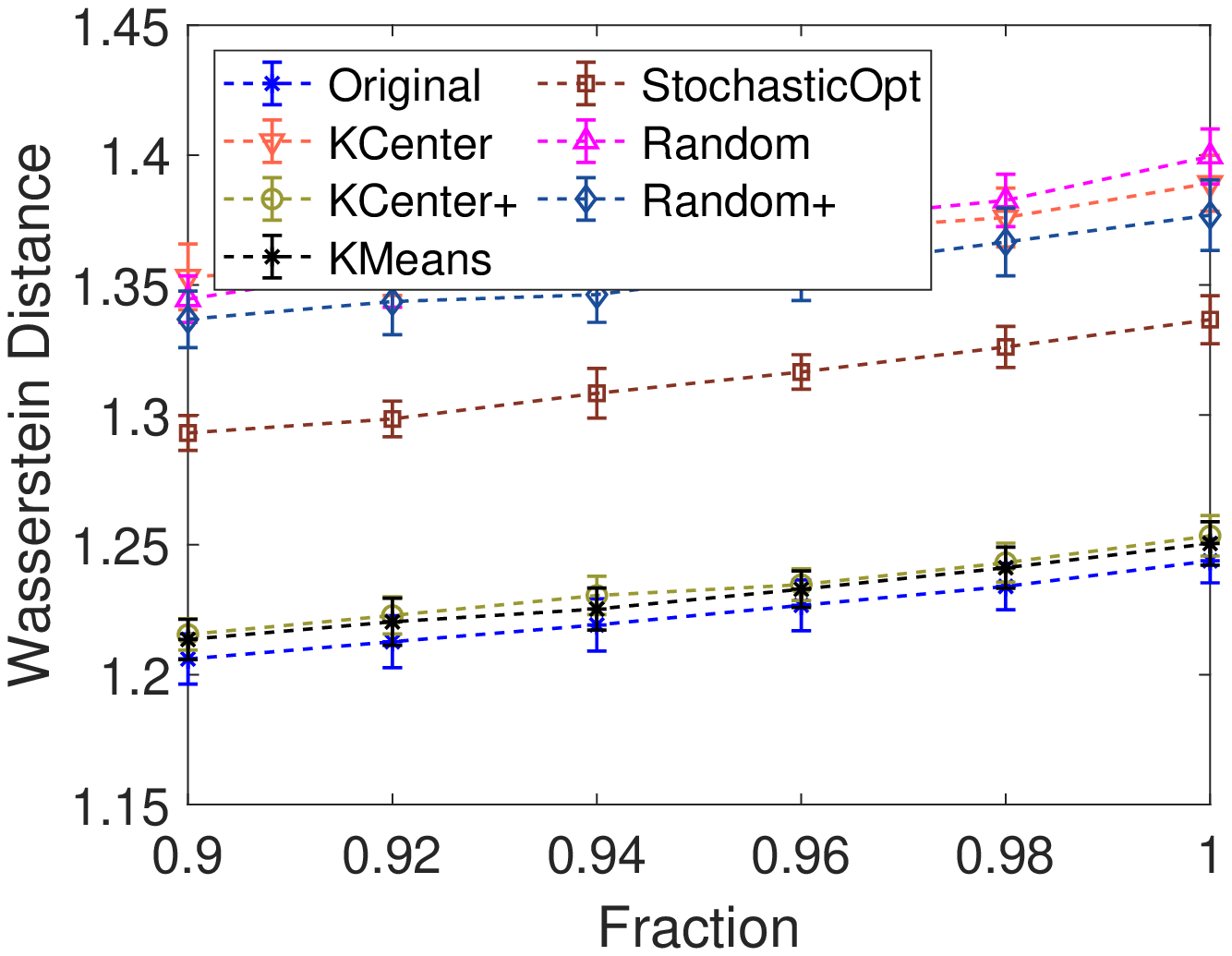} 
			\hspace{0.1in}
			\includegraphics[width=0.42\columnwidth]{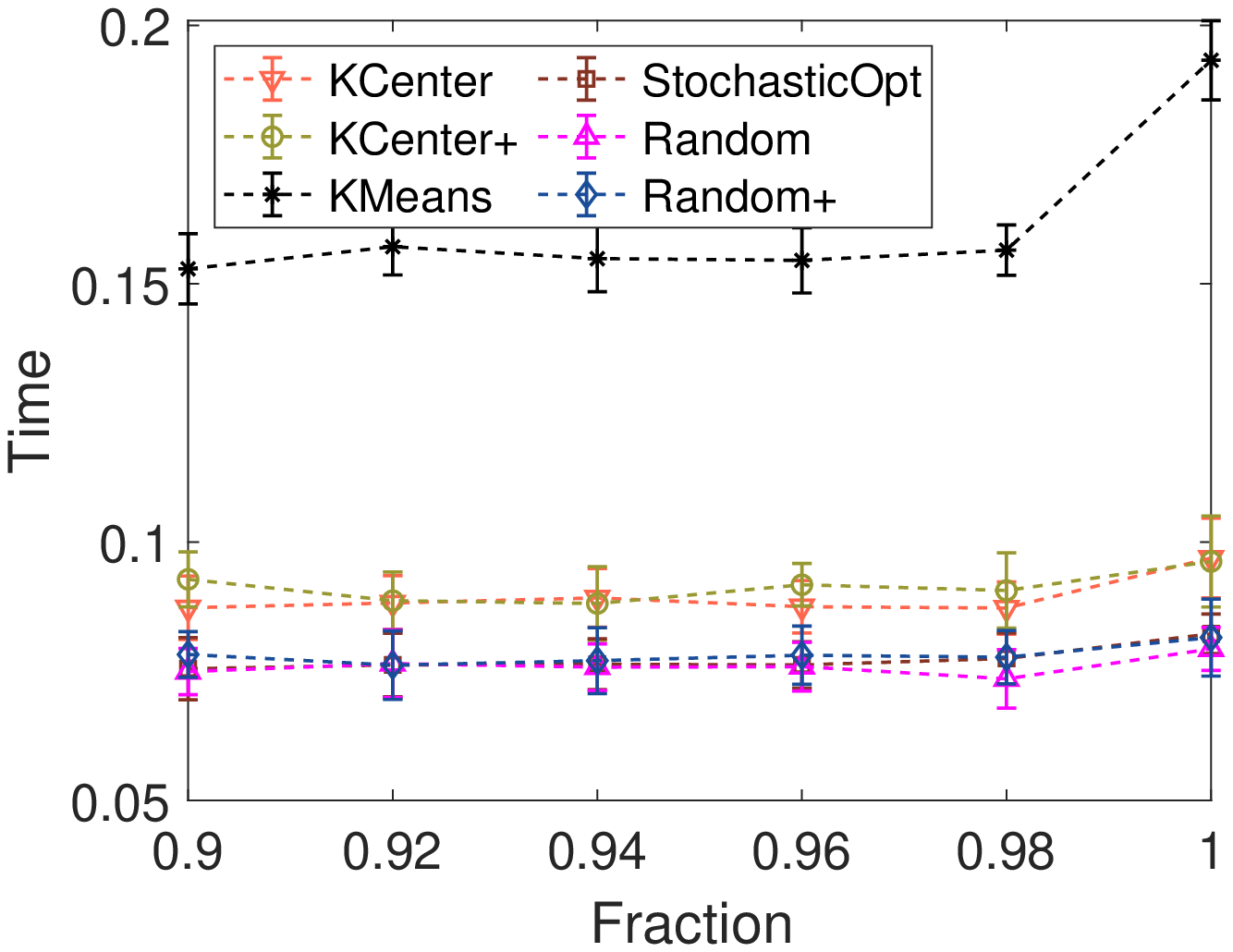}}
	\end{center}
	\vspace{-0.25in}
	\caption{The Wasserstein  distance and normalized running time on \textsc{ja-zh} for bilingual lexicon induction with different fraction $\lambda$. } 
	\label{fig:jazh2}  
	
\end{figure}
\begin{figure}[H]
	\begin{center}
		\centerline
		{\includegraphics[width=0.42\columnwidth]{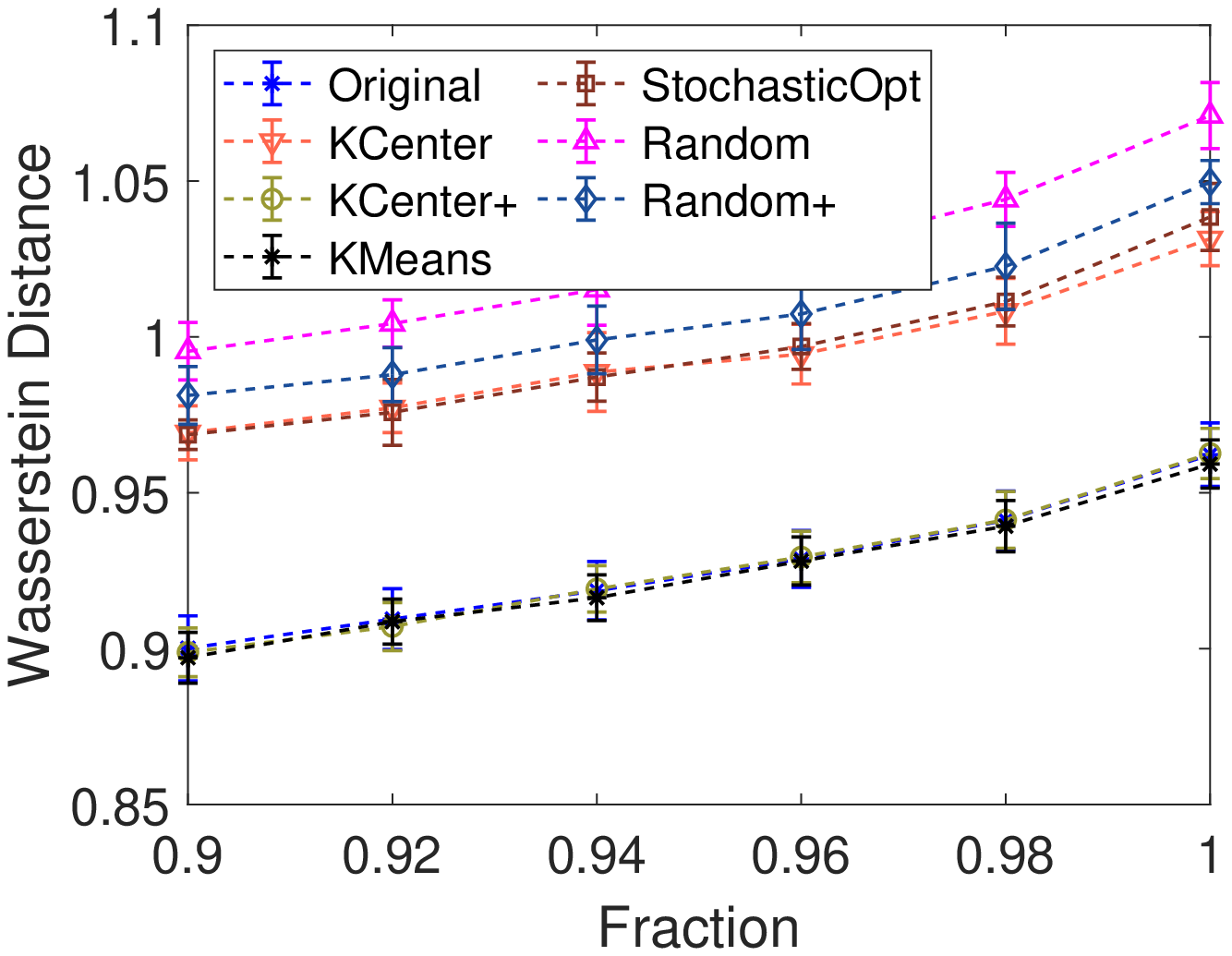} 
			\hspace{0.1in}
			\includegraphics[width=0.42\columnwidth]{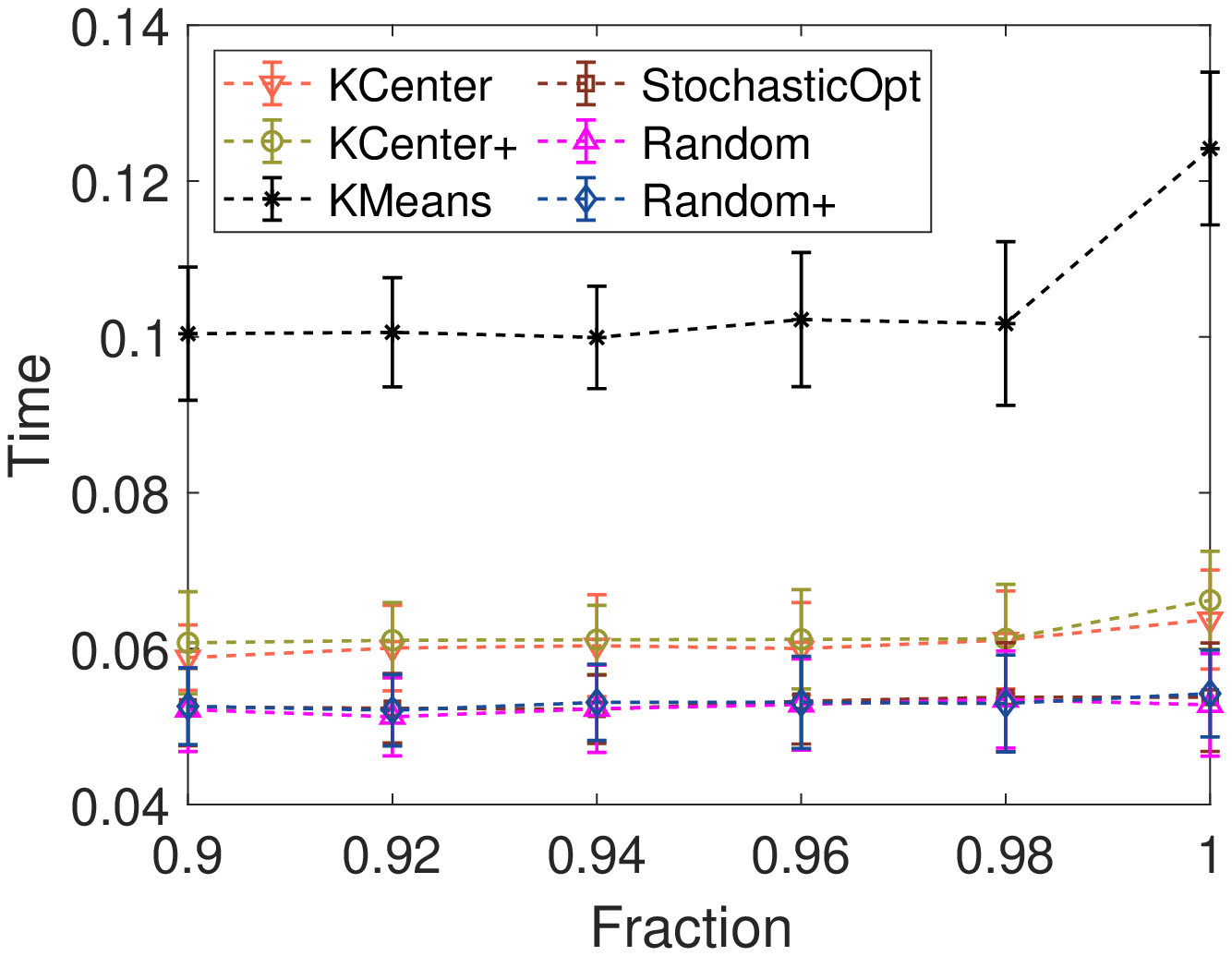}}
	\end{center}
	\vspace{-0.25in}
	\caption{The Wasserstein  distance and normalized running time on \textsc{tr-en} for bilingual lexicon induction with different fraction $\lambda$. } 
	\label{fig:tren2}  
\end{figure}

\begin{figure}[H]
	\begin{center}
		\centerline
		{\includegraphics[width=0.42\columnwidth]{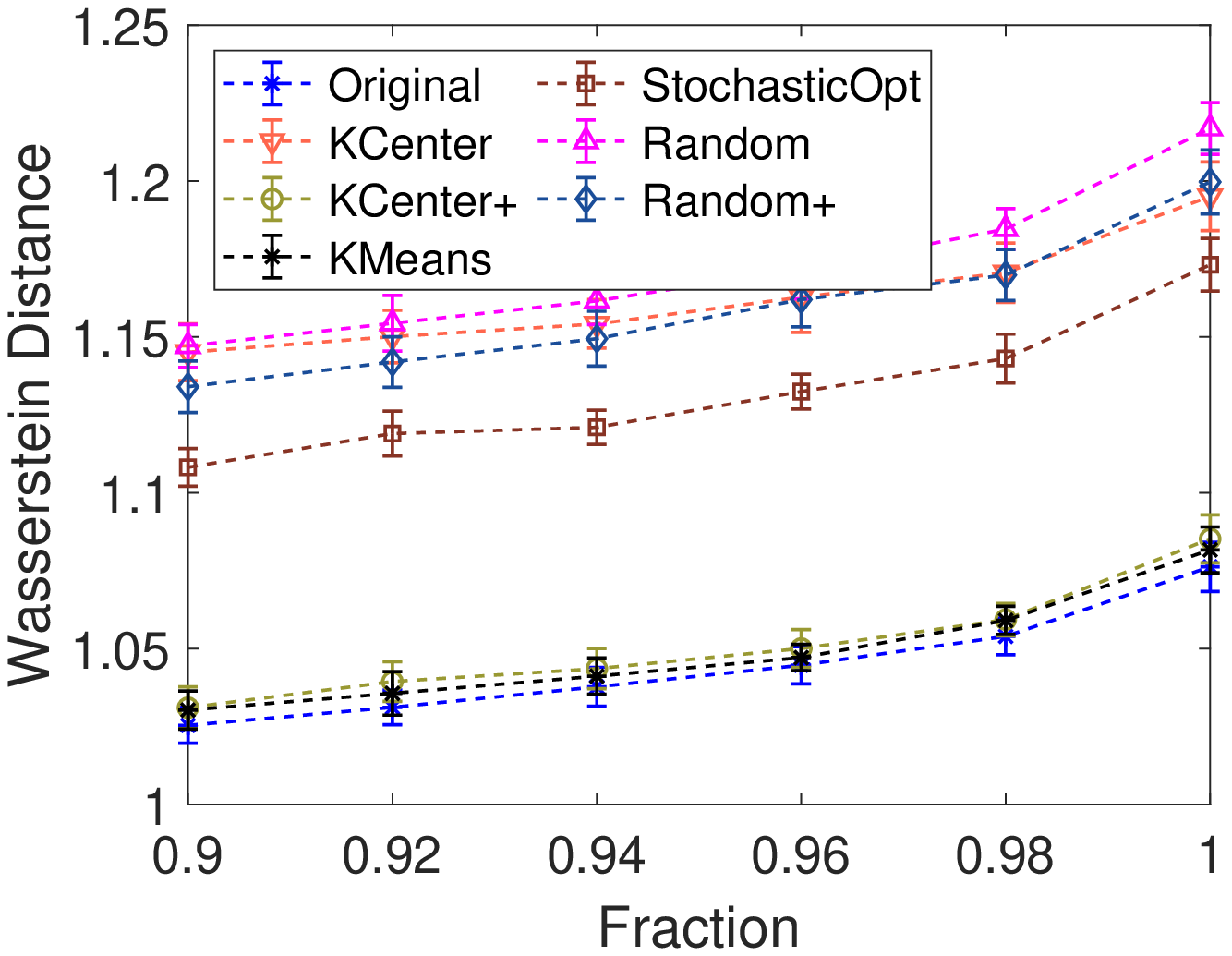} 
			\hspace{0.1in}
			\includegraphics[width=0.42\columnwidth]{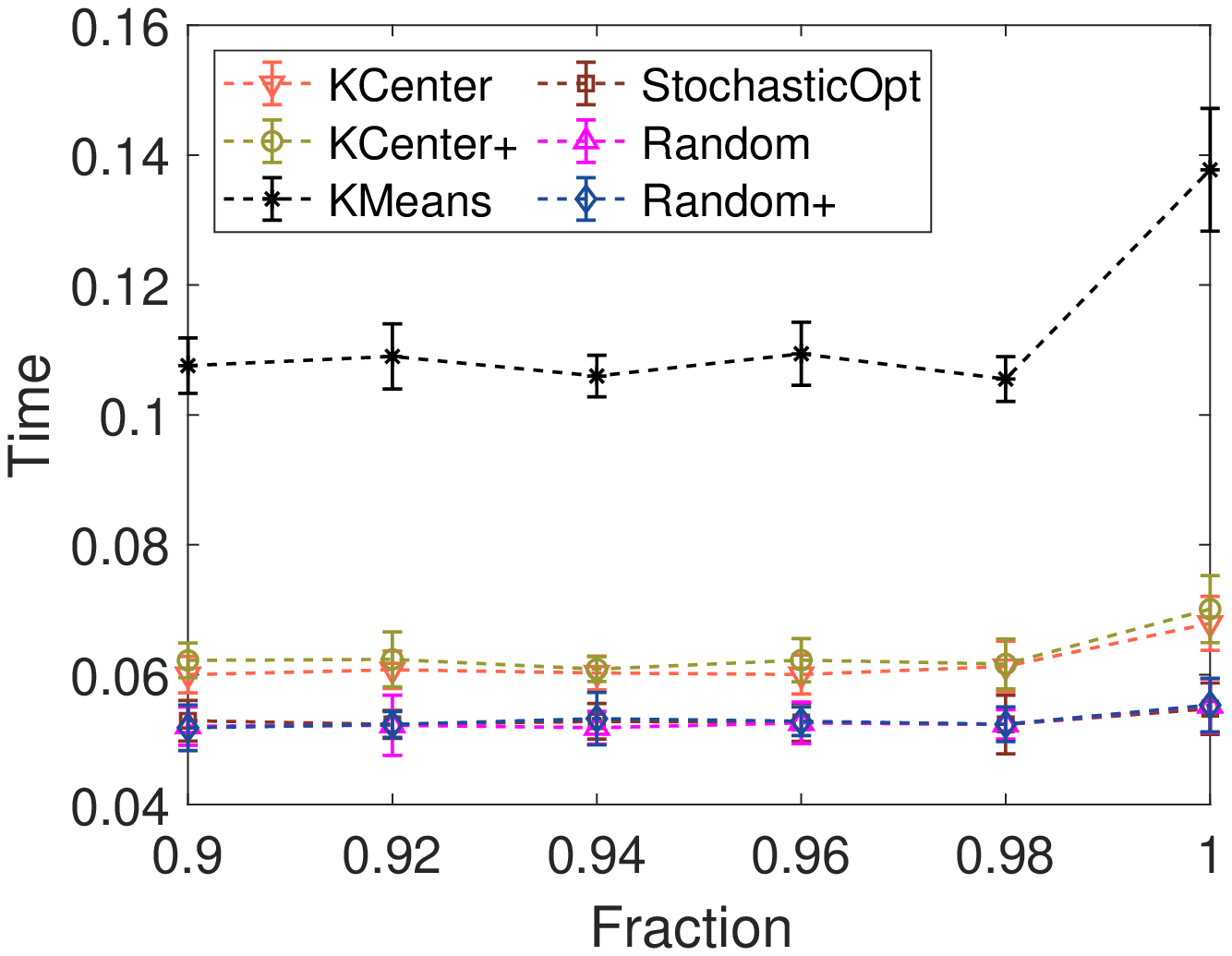}}
	\end{center}
	\vspace{-0.25in}
	\caption{The Wasserstein  distance and normalized running time on \textsc{zh-en} for bilingual lexicon induction with different fraction $\lambda$. } 
	\label{fig:zhen2}  
\end{figure}

\begin{table}[htbp]
	\small
	\centering
	\caption{Wasserstein distance for DA with different compression level}
	\begin{tabular}{c|c|ccccccc}
		\hline
		&    Compression   &\multicolumn{1}{c}{\textsc{Original}} & \multicolumn{1}{c}{\textsc{KCenter}} & \multicolumn{1}{c}{\textsc{KCenter+}} & \multicolumn{1}{c}{\textsc{KMeans}} &\multicolumn{1}{c}{\textsc{StochasticOpt}}& \multicolumn{1}{c}{\textsc{Random}} & \multicolumn{1}{c}{\textsc{Random+}} \\
		\hline
		
		\multirow{5}[0]{*}{D$\rightarrow$A} & 0.02 & 17.675  & 20.736  & \textbf{20.421} & 20.662  & 22.260  & 21.387 & 21.546 \\
		& 0.04 & 17.675  & 19.267  & \textbf{18.698} & 19.077  & 21.686  & 20.414 & 20.581 \\
		& 0.06 & 17.675  & 18.851  & \textbf{18.098} & 18.293  & 21.008  & 19.726 & 19.973 \\
		& 0.08 & 17.675  & 18.568  & \textbf{17.959} & 18.050  & 20.385  & 19.4  & 19.527 \\
		& 0.10 & 17.675  & 18.417  & \textbf{17.810} & 17.946  & 19.817  & 19.103 & 19.185 \\
		\hline
		\multirow{5}[0]{*}{D$\rightarrow$C} & 0.02 & 18.915  & 21.114  & \textbf{20.976} & 21.313   & 22.588  & 21.845 & 22.102 \\
		& 0.04 & 18.915  & 19.958  & \textbf{19.564} & 20.075  & 22.106  & 21.028 & 21.378 \\
		& 0.06 & 18.915  & 19.559  & \textbf{19.048} & 19.411   & 21.605  & 20.549 & 20.895 \\
		& 0.08 & 18.915  & 19.464  & \textbf{18.987} & 19.224  & 21.125  & 20.181 & 20.483 \\
		& 0.10 & 18.915  & 19.367  & \textbf{18.935} & 19.149  & 20.704  & 19.951 & 20.179 \\
		\hline
		\multirow{5}[0]{*}{D$\rightarrow$W} & 0.02 & 13.549  & \textbf{21.351} & 21.891  & 23.112  & 27.227  & 24.06 & 24.229 \\
		& 0.04 & 13.549  & 15.914  & \textbf{15.616} & 17.458  & 25.318  & 20.721 & 20.866 \\
		& 0.06 & 13.549  & 14.108  & \textbf{13.640} & 14.128  & 23.174  & 18.471 & 18.623 \\
		& 0.08 & 13.549  & 13.727  & \textbf{13.565} & 13.600  & 21.009  & 16.976 & 17.034 \\
		& 0.10 & 13.549  & 13.566  & \textbf{13.530} & 13.532  & 19.171  & 15.923 & 15.829 \\
		\hline
		\multirow{5}[0]{*}{W$\rightarrow$A} & 0.02 & 12.962  & 17.119  & \textbf{17.001} & 17.335   & 18.547  & 17.708 & 17.878 \\
		& 0.04 & 12.962  & 15.804  & \textbf{15.648} & 16.093  & 18.117  & 16.749 & 16.957 \\
		& 0.06 & 12.962  & 14.726  & \textbf{14.442} & 14.930   & 17.661  & 16.036 & 16.27 \\
		& 0.08 & 12.962  & 13.824  & \textbf{13.477} & 14.036  & 17.185  & 15.485 & 15.725 \\
		& 0.10 & 12.962  & 13.506  & \textbf{13.060} & 13.417  & 16.655  & 15.049 & 15.275 \\
		\hline
		\multirow{5}[0]{*}{W$\rightarrow$C} & 0.02 & 14.080  & \textbf{17.483} & 17.552  & 17.867   & 18.830  & 18.126 & 18.337 \\
		& 0.04 & 14.080  & 16.328  & \textbf{16.265} & 16.823   & 18.473  & 17.303 & 17.651 \\
		& 0.06 & 14.080  & 15.457  & \textbf{15.271} & 15.907 & 18.105  & 16.694 & 17.054 \\
		& 0.08 & 14.080  & 14.782  & \textbf{14.459} & 15.071   & 17.729  & 16.237 & 16.560 \\
		& 0.10 & 14.080  & 14.511  & \textbf{14.129} & 14.547 & 17.290  & 15.844 & 16.178 \\
		\hline
		\multirow{5}[0]{*}{W$\rightarrow$D} & 0.02 & 13.548  & \textbf{21.342} & 21.756  & 23.110   & 27.263  & 24.059 & 24.18 \\
		& 0.04 & 13.548  & 15.939  & \textbf{15.633} & 17.378  & 25.315  & 20.74 & 20.869 \\
		& 0.06 & 13.548  & 14.091  & \textbf{13.654} & 14.152  & 23.092  & 18.266 & 18.65 \\
		& 0.08 & 13.548  & 13.740  & \textbf{13.563} & 13.603 & 21.091  & 16.936 & 17.128 \\
		& 0.10 & 13.548  & 13.564  & \textbf{13.529} & 13.547  & 19.301  & 16.018 & 15.962 \\
		\hline
	\end{tabular}%
	
	\label{tab:DA_nonoiseemd}%
\end{table}%

\begin{table*}[htbp]
	\small
	\centering
	\caption{The Normalized Time for DA with different compression level}
	\begin{tabular}{c|c|cccccc}
		\hline
		&  Compression &  \multicolumn{1}{c}{\textsc{KCenter}} & \multicolumn{1}{c}{\textsc{KCenter+}} & \multicolumn{1}{c}{\textsc{KMeans}} & \multicolumn{1}{c}{\textsc{StochasticOpt}} & \multicolumn{1}{c}{\textsc{Random}} & \multicolumn{1}{c}{\textsc{Random+}} \\
		\hline
		\multirow{5}[0]{*}{D$\rightarrow$A} & 0.02 & 0.142  & 0.145  & 0.241  & 0.134  & 0.13  & 0.133 \\
		& 0.04 & 0.151  & 0.151  & 0.261  & 0.138  & 0.142 & 0.14 \\
		& 0.06 & 0.159  & 0.156  & 0.294  & 0.141  & 0.138 & 0.145 \\
		& 0.08 & 0.162  & 0.169  & 0.293  & 0.150  & 0.147 & 0.15 \\
		& 0.10 & 0.171  & 0.173  & 0.309  & 0.156  & 0.15  & 0.155 \\
		\hline
		\multirow{5}[0]{*}{D$\rightarrow$C} & 0.02 & 0.140  & 0.140  & 0.242  & 0.126  & 0.127 & 0.126 \\
		& 0.04 & 0.146  & 0.147  & 0.283  & 0.133  & 0.137 & 0.128 \\
		& 0.06 & 0.154  & 0.153  & 0.328  & 0.140  & 0.133 & 0.14 \\
		& 0.08 & 0.160  & 0.161  & 0.320  & 0.144  & 0.138 & 0.145 \\
		& 0.10 & 0.167  & 0.169  & 0.337  & 0.149  & 0.144 & 0.148 \\
		\hline
		\multirow{5}[0]{*}{D$\rightarrow$W} & 0.02 & 0.184  & 0.177  & 0.235  & 0.189  & 0.183 & 0.186 \\
		& 0.04 & 0.206  & 0.184  & 0.255  & 0.191  & 0.193 & 0.196 \\
		& 0.06 & 0.218  & 0.214  & 0.269  & 0.199  & 0.201 & 0.197 \\
		& 0.08 & 0.221  & 0.224  & 0.309  & 0.212  & 0.203 & 0.207 \\
		& 0.10 & 0.232  & 0.230  & 0.317  & 0.218  & 0.214 & 0.216 \\
		\hline
		\multirow{5}[0]{*}{W$\rightarrow$A} & 0.02 & 0.144  & 0.146  & 0.251  & 0.134  & 0.133 & 0.131 \\
		& 0.04 & 0.151  & 0.152  & 0.268  & 0.139  & 0.137 & 0.142 \\
		& 0.06 & 0.159  & 0.157  & 0.295  & 0.145  & 0.143 & 0.146 \\
		& 0.08 & 0.164  & 0.169  & 0.312  & 0.153  & 0.147 & 0.153 \\
		& 0.10 & 0.172  & 0.174  & 0.318  & 0.159  & 0.154 & 0.158 \\
		\hline
		\multirow{5}[0]{*}{W$\rightarrow$C} & 0.02 & 0.145  & 0.146  & 0.262  & 0.129  & 0.13  & 0.13 \\
		& 0.04 & 0.148  & 0.150  & 0.299  & 0.135  & 0.133 & 0.134 \\
		& 0.06 & 0.158  & 0.155  & 0.318  & 0.138  & 0.137 & 0.144 \\
		& 0.08 & 0.168  & 0.165  & 0.336  & 0.148  & 0.142 & 0.145 \\
		& 0.10 & 0.171  & 0.172  & 0.342  & 0.152  & 0.145 & 0.155 \\
		\hline
		\multirow{5}[0]{*}{W$\rightarrow$D} & 0.02 & 0.195  & 0.191  & 0.248  & 0.195  & 0.194 & 0.197 \\
		& 0.04 & 0.211  & 0.196  & 0.266  & 0.204  & 0.211 & 0.209 \\
		& 0.06 & 0.220  & 0.220  & 0.293  & 0.210  & 0.211 & 0.209 \\
		& 0.08 & 0.231  & 0.232  & 0.317  & 0.222  & 0.215 & 0.224 \\
		& 0.10 & 0.243  & 0.242  & 0.344  & 0.233  & 0.226 & 0.23 \\
		
		\hline
	\end{tabular}%
	
	\label{tab:DA_nonoisetime}%
\end{table*}%

\begin{table*}[htbp]
	\small
	\centering
	\caption{Accuracy for DA with different compression level}
	\begin{tabular}{c|c|ccccccc}
		\hline
		&    Compression   &\multicolumn{1}{c}{\textsc{Original}} & \multicolumn{1}{c}{\textsc{KCenter}} & \multicolumn{1}{c}{\textsc{KCenter+}} & \multicolumn{1}{c}{\textsc{KMeans}} &\multicolumn{1}{c}{\textsc{StochasticOpt}}& \multicolumn{1}{c}{\textsc{Random}} & \multicolumn{1}{c}{\textsc{Random+}} \\
		\hline
		\multirow{5}[0]{*}{D$\rightarrow$A} & 0.02 & 0.779  & \textbf{0.803} & 0.792  & 0.785  & 0.785  & 0.79  & 0.783 \\
		& 0.04 & 0.779  & \textbf{0.788} & 0.776  & 0.779  & 0.780  & 0.778 & 0.775 \\
		& 0.06 & 0.779  & 0.782  & \textbf{0.800} & 0.775  & 0.770  & 0.789 & 0.766 \\
		& 0.08 & 0.779  & 0.771  & \textbf{0.787} & 0.774  & 0.767  & 0.778 & 0.765 \\
		& 0.10 & 0.779  & 0.771  & \textbf{0.795} & 0.775  & 0.765  & 0.783 & 0.763 \\
		\hline
		\multirow{5}[0]{*}{D$\rightarrow$C} & 0.02 & 0.748  & 0.728  & 0.736  & \textbf{0.776} & 0.753  & 0.754 & 0.728 \\
		& 0.04 & 0.748  & 0.699  & 0.736  & \textbf{0.771} & 0.769  & 0.757 & 0.722 \\
		& 0.06 & 0.748  & 0.723  & 0.752  & 0.760  & \textbf{0.762} & 0.747 & 0.734 \\
		& 0.08 & 0.748  & 0.711  & 0.756  & \textbf{0.763} & 0.756  & 0.742 & 0.732 \\
		& 0.10 & 0.748  & 0.727  & 0.741  & \textbf{0.756} & 0.750  & 0.746 & 0.731 \\
		\hline
		\multirow{5}[0]{*}{D$\rightarrow$W} & 0.02 & 0.927  & 0.768  & \textbf{0.788} & 0.786  & 0.735  & 0.8   & 0.79 \\
		& 0.04 & 0.927  & \textbf{0.859} & 0.857  & 0.807  & 0.750  & 0.82  & 0.816 \\
		& 0.06 & 0.927  & 0.890  & \textbf{0.892} & 0.864  & 0.779  & 0.853 & 0.844 \\
		& 0.08 & 0.927  & \textbf{0.925} & 0.922  & 0.896  & 0.775  & 0.87  & 0.869 \\
		& 0.10 & 0.927  & 0.922  & \textbf{0.928} & 0.914  & 0.785  & 0.862 & 0.868 \\
		\hline
		\multirow{5}[0]{*}{W$\rightarrow$A} & 0.02 & 0.678  & 0.612  & \textbf{0.659} & 0.649  & 0.609  & 0.634 & 0.589 \\
		& 0.04 & 0.678  & 0.622  & \textbf{0.660} & 0.653  & 0.643  & 0.63  & 0.61 \\
		& 0.06 & 0.678  & 0.627  & \textbf{0.660} & 0.660  & 0.649  & 0.64  & 0.62 \\
		& 0.08 & 0.678  & 0.668  & \textbf{0.683} & 0.664  & 0.643  & 0.644 & 0.626 \\
		& 0.10 & 0.678  & 0.661  & \textbf{0.692} & 0.668  & 0.653  & 0.656 & 0.629 \\
		\hline
		\multirow{5}[0]{*}{W$\rightarrow$C} & 0.02 & 0.600  & 0.530  & 0.541  & \textbf{0.588} & 0.538  & 0.563 & 0.536 \\
		& 0.04 & 0.600  & 0.544  & 0.566  & \textbf{0.598} & 0.551  & 0.57  & 0.517 \\
		& 0.06 & 0.600  & 0.563  & 0.600  & \textbf{0.607} & 0.569  & 0.582 & 0.538 \\
		& 0.08 & 0.600  & 0.563  & 0.593  & \textbf{0.600} & 0.568  & 0.579 & 0.546 \\
		& 0.10 & 0.600  & 0.572  & 0.599  & \textbf{0.602} & 0.579  & 0.587 & 0.544 \\
		\hline
		\multirow{5}[0]{*}{W$\rightarrow$D} & 0.02 & 0.911  & 0.642  & \textbf{0.779} & 0.749  & 0.571  & 0.664 & 0.641 \\
		& 0.04 & 0.911  & 0.771  & \textbf{0.832} & 0.810  & 0.644  & 0.753 & 0.737 \\
		& 0.06 & 0.911  & 0.853  & \textbf{0.856} & 0.848  & 0.662  & 0.794 & 0.783 \\
		& 0.08 & 0.911  & \textbf{0.909} & 0.899  & 0.894  & 0.708  & 0.819 & 0.822 \\
		& 0.10 & 0.911  & 0.907  & \textbf{0.911} & 0.901  & 0.726  & 0.835 & 0.838 \\
		\hline
	\end{tabular}%
	\label{tab:DA_nonoiseacc}%
\end{table*}%

\begin{table*}[htbp]
	\small
	\centering
	\caption{Wasserstein distance for DA with different fraction $\lambda$}
	\begin{tabular}{c|c|ccccccc}
		\hline
		&       & \multicolumn{1}{c}{\textsc{Original}} & \multicolumn{1}{c}{\textsc{KCenter}} & \multicolumn{1}{c}{\textsc{KCenter+}} & \multicolumn{1}{c}{\textsc{KMeans}} 
		&\multicolumn{1}{c}{\textsc{StochasticOpt}}&\multicolumn{1}{c}{\textsc{Random}} & \multicolumn{1}{c}{\textsc{Random+}} \\
		\hline
		\multirow{6}[0]{*}{D$\rightarrow$A} & $\lambda=1.0$ &17.671   & 18.416  & \textbf{17.817} & 17.943  & 19.824 & 19.111 & 19.191 \\
		& $\lambda=0.98$ & 17.492   & 18.293  & \textbf{17.690} & 17.798  & 19.652 & 18.936 & 19.045 \\
		& $\lambda=0.96$ & 17.394 & 18.189  & \textbf{17.546} & 17.633  & 19.534 & 18.749 & 18.845 \\
		& $\lambda=0.94$ & 17.293  & 18.046  & \textbf{17.416} & 17.496  & 19.349 & 18.613 & 18.744 \\
		& $\lambda=0.92$ & 17.139  & 17.921  & \textbf{17.308} & 17.340  & 19.15 & 18.435 & 18.588 \\
		& $\lambda=0.90$ & 16.972  & 17.761  & \textbf{17.142} & 17.176  & 19.081 & 18.258 & 18.477 \\
		\hline
		\multirow{6}[0]{*}{D$\rightarrow$C} & $\lambda=1.0$ & 18.911   & 19.353  & \textbf{18.943} & 19.146  & 20.711 & 19.954 & 20.187 \\
		& $\lambda=0.98$ & 18.829   & 19.272  & \textbf{18.883} & 18.995  & 20.534 & 19.775 & 20.077 \\
		& $\lambda=0.96$ &18.726  & 19.182  & \textbf{18.746} & 18.860  & 20.35 & 19.599 & 19.913 \\
		& $\lambda=0.94$ &18.553 & 19.052  & \textbf{18.611} & 18.713  & 20.22 & 19.422 & 19.806 \\
		& $\lambda=0.92$ & 18.426   & 18.917  & \textbf{18.471} & 18.555  & 20.027 & 19.24 & 19.647 \\
		& $\lambda=0.90$ & 18.276 & 18.766  & \textbf{18.323} & 18.416  & 19.886 & 19.111 & 19.516 \\
		\hline
		\multirow{6}[0]{*}{D$\rightarrow$W} & $\lambda=1.0$ &13.555  & 13.554  & \textbf{13.531} & 13.536  & 19.239 & 15.999 & 15.919 \\
		& $\lambda=0.98$ &13.209  & 13.313  & \textbf{13.198} & 13.254  & 18.831 & 15.731 & 15.799 \\
		& $\lambda=0.96$ & 12.999  & 13.061  & 12.984  & \textbf{12.965} & 18.726 & 15.39 & 15.325 \\
		& $\lambda=0.94$ & 12.670 & 12.807  & 12.678  & \textbf{12.650} & 18.304 & 14.955 & 14.907 \\
		& $\lambda=0.92$ & 12.380  & 12.505  & 12.364  & \textbf{12.351} & 18.127 & 14.525 & 14.795 \\
		& $\lambda=0.90$ &12.037 & 12.175  & 12.045  & \textbf{12.016} & 17.91 & 14.243 & 14.463 \\
		\hline
		\multirow{6}[0]{*}{W$\rightarrow$A} & $\lambda=1.0$ & 12.960  & 13.523  & \textbf{13.063} & 13.413  & 16.657 & 15    & 15.225 \\
		& $\lambda=0.98$ & 12.815  & 13.379  & \textbf{12.903} & 13.252  & 16.51 & 14.846 & 15.074 \\
		& $\lambda=0.96$ & 12.675  & 13.225  & \textbf{12.746} & 13.093  & 16.329 & 14.669 & 14.939 \\
		& $\lambda=0.94$ & 12.502  & 13.059  & \textbf{12.600} & 12.922  & 16.267 & 14.518 & 14.759 \\
		& $\lambda=0.92$ & 12.364  & 12.923  & \textbf{12.453} & 12.818  & 16.141 & 14.309 & 14.638 \\
		& $\lambda=0.90$ &12.227 & 12.776  & \textbf{12.310} & 12.663  & 15.962 & 14.22 & 14.51 \\
		\hline
		\multirow{6}[0]{*}{W$\rightarrow$C} & $\lambda=1.0$ & 14.079  & 14.508  & \textbf{14.131} & 14.548  & 17.329 & 15.849 & 16.187 \\
		& $\lambda=0.98$ &13.923  & 14.372  & \textbf{13.985} & 14.400  & 17.123 & 15.695 & 16.028 \\
		& $\lambda=0.96$ &13.821  & 14.235  & \textbf{13.850} & 14.229  & 16.989 & 15.472 & 15.899 \\
		& $\lambda=0.94$ & 13.682  & 14.087  & \textbf{13.720} & 14.116  & 16.889 & 15.312 & 15.756 \\
		& $\lambda=0.92$ & 13.525  & 13.960  & \textbf{13.583} & 13.985  & 16.715 & 15.154 & 15.655 \\
		& $\lambda=0.90$ & 13.375  & 13.822  & \textbf{13.447} & 13.861  & 16.593 & 15.023 & 15.504 \\
		\hline
		\multirow{6}[0]{*}{W$\rightarrow$D} & $\lambda=1.0$ &13.554  & 13.540  & \textbf{13.525} & 13.527  & 19.164 & 15.986 & 15.864 \\
		& $\lambda=0.98$ & 13.217  & 13.291  & \textbf{13.197} & 13.239  & 18.971 & 15.807 & 15.609 \\
		& $\lambda=0.96$ & 12.998  & 13.052  & 12.984  & \textbf{12.962} & 18.595 & 15.349 & 15.286 \\
		& $\lambda=0.94$ & 12.670  & 12.807  & 12.681  & \textbf{12.672} & 18.216 & 15.011 & 15.032 \\
		& $\lambda=0.92$ &12.380  & 12.501  & 12.363  & \textbf{12.341} & 18.143 & 14.667 & 14.786 \\
		& $\lambda=0.90$ & 12.037 & 12.175  & 12.043  & \textbf{12.030} & 17.909 & 14.33 & 14.391 \\
		
		\hline
	\end{tabular}
	\label{tab:DA_noiseemd}%
\end{table*}%

\begin{table*}[htbp]
	\small
	\centering
	\caption{The Normalized Time for DA with different fraction $\lambda$}
	\begin{tabular}{c|c|cccccc}
		\hline
		&   &  \multicolumn{1}{c}{\textsc{KCenter}} & \multicolumn{1}{c}{\textsc{KCenter+}} & \multicolumn{1}{c}{\textsc{KMeans}} &\multicolumn{1}{c}{\textsc{StochasticOpt}}& \multicolumn{1}{c}{\textsc{Random}} & \multicolumn{1}{c}{\textsc{Random+}} \\
		\hline
		\multirow{6}[0]{*}{D$\rightarrow$A} & $\lambda=1.0$ & 0.171  & 0.175  & 0.299  & 0.156  & 0.154 & 0.156 \\
		& $\lambda=0.98$ & 0.137  & 0.140  & 0.248  & 0.128  & 0.124 & 0.126 \\
		& $\lambda=0.96$ & 0.138  & 0.140  & 0.238  & 0.128  & 0.124 & 0.126 \\
		& $\lambda=0.94$ & 0.135  & 0.137  & 0.242  & 0.128  & 0.122 & 0.125 \\
		& $\lambda=0.92$ & 0.143  & 0.145  & 0.248  & 0.131  & 0.129 & 0.13 \\
		& $\lambda=0.90$ & 0.137  & 0.138  & 0.236  & 0.129  & 0.121 & 0.127 \\
		\hline
		\multirow{6}[0]{*}{D$\rightarrow$C} & $\lambda=1.0$ & 0.166  & 0.164  & 0.329  & 0.156  & 0.144 & 0.149 \\
		& $\lambda=0.98$ & 0.137  & 0.137  & 0.255  & 0.131  & 0.12  & 0.123 \\
		& $\lambda=0.96$ & 0.130  & 0.131  & 0.242  & 0.123  & 0.115 & 0.117 \\
		& $\lambda=0.94$ & 0.130  & 0.130  & 0.239  & 0.123  & 0.113 & 0.116 \\
		& $\lambda=0.92$ & 0.133  & 0.139  & 0.256  & 0.130  & 0.119 & 0.122 \\
		& $\lambda=0.90$ & 0.140  & 0.139  & 0.262  & 0.133  & 0.124 & 0.128 \\
		\hline
		\multirow{6}[0]{*}{D$\rightarrow$W} & $\lambda=1.0$ & 0.234  & 0.245  & 0.337  & 0.225  & 0.219 & 0.225 \\
		& $\lambda=0.98$ & 0.275  & 0.274  & 0.374  & 0.256  & 0.256 & 0.259 \\
		& $\lambda=0.96$ & 0.192  & 0.197  & 0.263  & 0.181  & 0.178 & 0.179 \\
		& $\lambda=0.94$ & 0.184  & 0.183  & 0.248  & 0.173  & 0.168 & 0.172 \\
		& $\lambda=0.92$ & 0.184  & 0.181  & 0.255  & 0.173  & 0.167 & 0.174 \\
		& $\lambda=0.90$ & 0.181  & 0.180  & 0.254  & 0.177  & 0.173 & 0.175 \\
		\hline
		\multirow{6}[0]{*}{W$\rightarrow$A} & $\lambda=1.0$ & 0.188  & 0.191  & 0.346  & 0.172  & 0.169 & 0.172 \\
		& $\lambda=0.98$ & 0.144  & 0.148  & 0.265  & 0.136  & 0.132 & 0.135 \\
		& $\lambda=0.96$ & 0.145  & 0.146  & 0.253  & 0.132  & 0.128 & 0.133 \\
		& $\lambda=0.94$ & 0.151  & 0.156  & 0.260  & 0.140  & 0.137 & 0.141 \\
		& $\lambda=0.92$ & 0.150  & 0.153  & 0.263  & 0.139  & 0.136 & 0.14 \\
		& $\lambda=0.90$ & 0.145  & 0.148  & 0.263  & 0.136  & 0.131 & 0.136 \\
		\hline
		\multirow{6}[0]{*}{W$\rightarrow$C} & $\lambda=1.0$ & 0.168  & 0.169  & 0.334  & 0.147  & 0.144 & 0.148 \\
		& $\lambda=0.98$ & 0.125  & 0.127  & 0.238  & 0.112  & 0.109 & 0.114 \\
		& $\lambda=0.96$ & 0.125  & 0.127  & 0.238  & 0.112  & 0.11  & 0.112 \\
		& $\lambda=0.94$ & 0.134  & 0.137  & 0.251  & 0.117  & 0.115 & 0.119 \\
		& $\lambda=0.92$ & 0.128  & 0.131  & 0.237  & 0.114  & 0.11  & 0.115 \\
		& $\lambda=0.90$ & 0.127  & 0.129  & 0.242  & 0.112  & 0.111 & 0.116 \\
		\hline
		\multirow{6}[0]{*}{W$\rightarrow$D} & $\lambda=1.0$ & 0.221  & 0.220  & 0.329  & 0.208  & 0.197 & 0.206 \\
		& $\lambda=0.98$ & 0.165  & 0.162  & 0.236  & 0.154  & 0.147 & 0.153 \\
		& $\lambda=0.96$ & 0.168  & 0.169  & 0.237  & 0.154  & 0.152 & 0.153 \\
		& $\lambda=0.94$ & 0.166  & 0.165  & 0.232  & 0.153  & 0.152 & 0.153 \\
		& $\lambda=0.92$ & 0.161  & 0.167  & 0.232  & 0.151  & 0.147 & 0.15 \\
		& $\lambda=0.90$ & 0.169  & 0.167  & 0.244  & 0.160  & 0.154 & 0.159 \\
		
		\hline
	\end{tabular}%
	
	\label{tab:DA_noisetime}%
\end{table*}%

\begin{table*}[htbp]
	\small
	\centering
	\caption{Accuracy for DA with different fraction $\lambda$}
	\begin{tabular}{c|c|ccccccc}
		\hline
		&       & \multicolumn{1}{c}{\textsc{Original}}  & \multicolumn{1}{c}{\textsc{KCenter}} & \multicolumn{1}{c}{\textsc{KCenter+}} & \multicolumn{1}{c}{\textsc{KMeans}} 
		&\multicolumn{1}{c}{\textsc{StochasticOpt}}&\multicolumn{1}{c}{\textsc{Random}} & \multicolumn{1}{c}{\textsc{Random+}} \\
		\hline
		
		\multirow{6}[0]{*}{D$\rightarrow$A} & $\lambda=1.0$ & 0.779  & 0.773  & \textbf{0.794} & 0.778  & 0.764 & 0.773 & 0.772 \\
		& $\lambda=0.98$ & 0.779  & 0.774  & \textbf{0.791} & 0.774  & 0.766 & 0.785 & 0.765 \\
		& $\lambda=0.96$ & 0.777   & 0.774  & \textbf{0.787} & 0.774  & 0.767 & 0.778 & 0.769 \\
		& $\lambda=0.94$ & 0.775  & 0.772  & \textbf{0.795} & 0.774  & 0.765 & 0.78  & 0.761 \\
		& $\lambda=0.92$ & 0.780  & 0.776  & \textbf{0.800} & 0.774  & 0.758 & 0.784 & 0.766 \\
		& $\lambda=0.90$ & 0.767 & 0.776  & \textbf{0.801} & 0.781  & 0.757 & 0.78  & 0.766 \\
		\hline
		\multirow{6}[0]{*}{D$\rightarrow$C} & $\lambda=1.0$ & 0.748  & 0.730  & 0.736  & 0.754  & \textbf{0.755} & 0.747 & 0.734 \\
		& $\lambda=0.98$ & 0.756   & 0.730  & 0.746  & \textbf{0.755} & 0.752 & 0.746 & 0.73 \\
		& $\lambda=0.96$ & 0.760   & 0.733  & 0.747  & \textbf{0.761} & 0.75  & 0.751 & 0.736 \\
		& $\lambda=0.94$ & 0.758   & 0.733  & 0.757  & \textbf{0.758} & 0.753 & 0.745 & 0.73 \\
		& $\lambda=0.92$ & 0.754   & 0.733  & \textbf{0.755} & \textbf{0.755} & \textbf{0.755} & 0.747 & 0.736 \\
		& $\lambda=0.90$ & 0.749   & 0.734  & 0.755  & \textbf{0.758} & 0.753 & 0.756 & 0.74 \\
		\hline
		\multirow{6}[0]{*}{D$\rightarrow$W} & $\lambda=1.0$ & 0.927  & 0.921  & \textbf{0.928} & 0.917  & 0.786 & 0.867 & 0.885 \\
		& $\lambda=0.98$ & 0.925  & 0.924  & \textbf{0.927} & 0.911  & 0.792 & 0.867 & 0.88 \\
		& $\lambda=0.96$ & 0.919  & \textbf{0.926} & 0.917  & 0.915  & 0.768 & 0.871 & 0.874 \\
		& $\lambda=0.94$ & 0.921   & \textbf{0.925} & 0.922  & 0.914  & 0.784 & 0.875 & 0.88 \\
		& $\lambda=0.92$ & 0.925   & 0.923  & \textbf{0.924} & 0.917  & 0.779 & 0.876 & 0.872 \\
		& $\lambda=0.90$ & 0.927   & \textbf{0.926} & 0.925  & 0.917  & 0.78  & 0.879 & 0.87 \\
		\hline
		\multirow{6}[0]{*}{W$\rightarrow$A} & $\lambda=1.0$ & 0.678  & 0.664  & \textbf{0.697} & 0.670  & 0.66  & 0.657 & 0.631 \\
		& $\lambda=0.98$ & 0.675 & 0.661  & \textbf{0.691} & 0.666  & 0.65  & 0.642 & 0.625 \\
		& $\lambda=0.96$ &0.672  & 0.662  & \textbf{0.693} & 0.674  & 0.647 & 0.656 & 0.629 \\
		& $\lambda=0.94$ & 0.664   & 0.664  & \textbf{0.695} & 0.669  & 0.656 & 0.649 & 0.626 \\
		& $\lambda=0.92$ & 0.660   & 0.668  & \textbf{0.696} & 0.665  & 0.64  & 0.66  & 0.615 \\
		& $\lambda=0.90$ &0.660  & 0.666  & \textbf{0.690} & 0.667  & 0.648 & 0.634 & 0.622 \\
		\hline
		\multirow{6}[0]{*}{W$\rightarrow$C} & $\lambda=1.0$ & 0.600   & 0.571  & 0.597  & \textbf{0.601} & 0.576 & 0.578 & 0.544 \\
		& $\lambda=0.98$ & 0.588  & 0.571  & 0.589  & \textbf{0.601} & 0.58  & 0.577 & 0.545 \\
		& $\lambda=0.96$ & 0.597  & 0.568  & 0.581  & \textbf{0.597} & 0.579 & 0.591 & 0.554 \\
		& $\lambda=0.94$ & 0.598  & 0.572  & 0.584  & \textbf{0.599} & 0.575 & 0.582 & 0.551 \\
		& $\lambda=0.92$ & 0.588  & 0.572  & 0.583  & \textbf{0.597} & 0.572 & 0.572 & 0.537 \\
		& $\lambda=0.90$ & 0.585  & 0.576  & 0.588  & \textbf{0.599} & 0.578 & 0.573 & 0.55 \\
		\hline
		\multirow{6}[0]{*}{W$\rightarrow$D} & $\lambda=1.0$ & 0.911 & 0.903  & \textbf{0.911} & 0.903  & 0.724 & 0.835 & 0.844 \\
		& $\lambda=0.98$ &0.911 & 0.910  & \textbf{0.911} & 0.906  & 0.737 & 0.828 & 0.84 \\
		& $\lambda=0.96$ & 0.917 & \textbf{0.917} & \textbf{0.917} & 0.909  & 0.732 & 0.825 & 0.847 \\
		& $\lambda=0.94$ & 0.917  & 0.915  & \textbf{0.917} & 0.912  & 0.715 & 0.833 & 0.83 \\
		& $\lambda=0.92$ & 0.924  & 0.917  & \textbf{0.924} & 0.913  & 0.734 & 0.828 & 0.848 \\
		& $\lambda=0.90$ & 0.924  & 0.918  & \textbf{0.924} & 0.914  & 0.72  & 0.831 & 0.845 \\
		
		\hline
	\end{tabular}%
	\label{tab:DA_noiseacc}%
\end{table*}%

\section{Acknowledgements} 
The research of this work was supported in part by National Key R\&D program of China through grant 2021YFA1000900, the NSFC throught Grant 62272432, and the Provincial NSF of Anhui through grant 2208085MF163. We also want to thank the anonymous reviewers for their helpful comments.

\bibliographystyle{abbrv}

\bibliography{acm.bib}

\appendix
\section{Proof of Theorem \ref{the-quality-2}}
Let $\mathcal{T}_{ \mathtt{opt}}$ be the optimal rigid transformation with respect to $\min_\mathcal{T}\mathcal{W}^2_2\big(A, \mathcal{T}(B), \lambda\big)$. Since $\tilde{\mathcal{T}}$ yields $c$-approximation for minimizing $\mathcal{W}^2_2\big(S_A, \mathcal{T}(S_B), \lambda \big)$, we have 
\begin{eqnarray}
	\mathcal{W}^2_2\big(S_A, \tilde{\mathcal{T}}(S_B),\lambda\big)&\leq& c\cdot \min_\mathcal{T}\mathcal{W}^2_2\big(S_A, \mathcal{T}(S_B), \lambda \big) \nonumber\\
	&\leq& c\cdot\mathcal{W}^2_2\big(S_A, \mathcal{T}_{ \mathtt{opt}}(S_B), \lambda \big).\label{for-quality8}
\end{eqnarray}
We denote 
the flow of $\mathcal{W}^2_2\big(S_A, \tilde{\mathcal{T}}(S_B), \lambda \big)$ as $\tilde{F}^{s}=\{\tilde{f}_{ij}^s\}$ and  the flow of  $\mathcal{W}^2_2\big(A, \tilde{\mathcal{T}}(B), \lambda \big)$ as $\tilde{F}=\{\tilde{f}_{ij}\}$. 
Then we have 
\begin{eqnarray}
	&&\mathcal{W}^2_2\big(A, \tilde{\mathcal{T}}(B), \lambda\big)\nonumber \\
	&=&\frac{1}{\lambda \min\{W_A, W_B\}}\min_{\tilde{F}}\sum^{n_1}_{i=1}\sum^{n_2}_{j=1}\tilde{f}_{ij}||a_{i}-\tilde{\mathcal{T}}(b_j)||^2\nonumber\\
	&\leq&\frac{1}{\lambda \min\{W_A, W_B\}}\sum^{n_1}_{i=1}\sum^{n_2}_{j=1}\tilde{f}_{ij}^s||a_{i}-\tilde{\mathcal{T}}(b_j)||^2\nonumber\\	
	&\underbrace{\leq}_{\text{by (\ref{for-quality4})}}&\frac{1}{\lambda \min\{W_A, W_B\}}\sum^{n_1}_{i=1}\sum^{n_2}_{j=1}\tilde{f}_{ij}^s \big((1+2\epsilon)||a'_{i}-\tilde{\mathcal{T}}(b'_j)||^2+(2\epsilon+4\epsilon^2)\Delta^2 \big)\nonumber\\	
	&=&\frac{(1+2\epsilon)}{\lambda \min\{W_A, W_B\}}\sum^{n_1}_{i=1}\sum^{n_2}_{j=1}\tilde{f}_{ij}^s (||a'_{i}-\tilde{\mathcal{T}}(b'_j)||^2  ) +(2\epsilon+4\epsilon^2)\Delta^2 \frac{\sum^{n_1}_{i=1}\sum^{n_2}_{j=1}\tilde{f}_{ij}^s }{\lambda \min\{W_A, W_B\}}\nonumber \\
	&=&(1+2\epsilon)\mathcal{W}^2_2(S_A, \tilde{\mathcal{T}}(S_B), \lambda )+(2\epsilon+4\epsilon^2)\Delta^2.
	\label{for-quality9}
\end{eqnarray}
By using the similar idea, we denote the flow of  $\mathcal{W}^2_2\big(S_A, \mathcal{T}_{\mathtt{opt}}(S_B), \lambda \big)$ as $F^{s}=\{f_{ij}^{s}\}$ and  the flow of  $\mathcal{W}^2_2\big(A, \mathcal{T}_{\mathtt{opt}}(B), \lambda \big)$ as $F=\{f_{ij}\}$. Then we have 
\begin{eqnarray}
	&&\mathcal{W}^2_2\big(S_A, \mathcal{T}_{ \mathtt{opt}}(S_B), \lambda \big) \nonumber\\
	&=&\frac{1}{\lambda \min\{W_A, W_B\}}\min_{F^{s}}\sum^{n_1}_{i=1}\sum^{n_2}_{j=1}f_{ij}^{s}||a'_{i}-\mathcal{T}_{\mathtt{opt}}(b'_j)||^2\nonumber\\
	&\leq&\frac{1}{\lambda \min\{W_A, W_B\}}\sum^{n_1}_{i=1}\sum^{n_2}_{j=1}f_{ij}||a'_{i}-\mathcal{T}_{\mathtt{opt}}(b'_j)||^2\nonumber\\
	&\underbrace{\leq}_{\text{by(\ref{for-quality5})}}&\frac{1+2\epsilon}{\lambda\min\{W_A, W_B\}}\sum^{n_1}_{i=1}\sum^{n_2}_{j=1}\big(f_{ij}||a_{i}-\mathcal{T}_{\mathtt{opt}}(b_j)||^2+(2\epsilon+4\epsilon^2)\Delta^2\big)\nonumber\\
	&=&(1+2\epsilon)\mathcal{W}^2_2\big(A, \mathcal{T}_{ \mathtt{opt}}(B), \lambda \big)+(2\epsilon+4\epsilon^2)\Delta^2.\label{for-quality10}
\end{eqnarray}
Combining (\ref{for-quality8}), (\ref{for-quality9}), and (\ref{for-quality10}), we have 
\begin{eqnarray}
	&&\mathcal{W}^2_2\big(A, \tilde{\mathcal{T}}(B), \lambda \big) \nonumber\\
	&\underbrace{\leq}_{\text{by(\ref{for-quality9})}}&(1+2\epsilon)\mathcal{W}^2_2\big(S_A, \tilde{\mathcal{T}}(S_B), \lambda \big)+(2\epsilon+4\epsilon^2)\Delta^2\nonumber\\
	&\underbrace{\leq}_{\text{by(\ref{for-quality8})}}&(1+2\epsilon)\cdot c\cdot\mathcal{W}^2_2\big(S_A, \mathcal{T}_{ \mathtt{opt}}(S_B), \lambda \big)+(2\epsilon+4\epsilon^2)\Delta^2 \nonumber\\
	&\underbrace{\leq}_{\text{by(\ref{for-quality10})}}&c(1+2\epsilon) \cdot( (1+2\epsilon)\mathcal{W}^2_2\big(A, \mathcal{T}_{ \mathtt{opt}}(B), \lambda \big)+(2\epsilon+4\epsilon^2)\Delta^2 ) +(2\epsilon+4\epsilon^2)\Delta^2 \nonumber\\
	&=& c(1+2\epsilon)^2\cdot\mathcal{W}^2_2\big(A, \mathcal{T}_{ \mathtt{opt}}(B), \lambda \big) +2\epsilon(c+1+2c\epsilon)(1+2\epsilon)\Delta^2,
\end{eqnarray}
and the proof is completed.

\end{document}